\documentclass[sigconf]{acmart}

\copyrightyear{2022} 
\acmYear{2022} 
\setcopyright{acmcopyright}
\acmConference[ASE '22]{37th IEEE/ACM International Conference on Automated Software Engineering}{October 10--14, 2022}{Rochester, MI, USA}
\acmBooktitle{37th IEEE/ACM International Conference on Automated Software Engineering (ASE '22), October 10--14, 2022, Rochester, MI, USA}
\acmPrice{15.00}
\acmDOI{10.1145/3551349.3556907}
\acmISBN{978-1-4503-9475-8/22/10}

\usepackage[ruled,lined,noend,linesnumbered]{algorithm2e}

\usepackage{hyperref}
\hypersetup{hidelinks,
colorlinks=true,
allcolors=blue,
pdfstartview=Fit,
breaklinks=true}
\usepackage[hyphenbreaks]{breakurl}

\usepackage{graphicx}
\usepackage{textcomp}
\usepackage{multicol}
\usepackage{subcaption}
\usepackage{xcolor}
\def\BibTeX{{\rm B\kern-.05em{\sc i\kern-.025em b}\kern-.08em
		T\kern-.1667em\lower.7ex\hbox{E}\kern-.125emX}}
\newtheorem{definition}{Definition}
\newtheorem{theorem}{Theorem}
\newtheorem{lemma}{Lemma}
\usepackage[mathlines]{lineno}
\newcounter{example}
\newenvironment{example}[1][]{\refstepcounter{example}\par\medskip
	\noindent \textbf{Example~\theexample. #1} \rmfamily}{\medskip}
\usepackage{makecell}
\usepackage{multirow}
\usepackage{tabularx}
\usepackage{hhline}
\usepackage{colortbl}
\usepackage{xcolor}
\definecolor{atomictangerine}{rgb}{1.0, 0.6, 0.4}
\definecolor{apricot}{rgb}{0.98, 0.81, 0.69}
\definecolor{antiquewhite}{rgb}{0.98, 0.92, 0.84}
\definecolor{Gray}{gray}{0.40}
\definecolor{LightCyan}{rgb}{0.88,1,1}
\definecolor{ashgrey}{rgb}{0.7, 0.75, 0.71}
\definecolor{darkgray}{gray}{0.80}
\newcolumntype{e}{>{\columncolor{Gray}}r}

\SetCommentSty{mycommfont}
\usepackage{threeparttable}
	\title[Provably Tightest Linear Approximation for Robustness Verification of Sigmoid-like Neural Networks]{Provably Tightest Linear Approximation for\\ Robustness Verification of Sigmoid-like Neural Networks }
	
 	\author{Zhaodi Zhang}
 	 	\email{zdzhang@stu.ecnu.edu.cn}
 	\affiliation{%
      \institution{East China Normal University}
      \city{Shanghai}
      \country{China}
    }

 	\author{Yiting Wu}
 	 	\email{51205902026@stu.ecnu.edu.cn}
 	\affiliation{%
      \institution{East China Normal University}
      \city{Shanghai}
      \country{China}
    }

 	\author{Si Liu}
 	 	\email{si.liu@inf.ethz.ch}
 	\affiliation{%
      \institution{ETH Z{\"u}rich}
      \city{Z{\"u}rich}
      \country{Switzerland}
    }

 	\author{Jing Liu}
 	 	\email{jliu@sei.ecnu.edu.cn}
 	\affiliation{%
      \institution{Shanghai Key Laboratory of Trustworthy Computing,\\East China Normal University}
      \city{Shanghai}
      \country{China}
    }

 	\author{Min Zhang}
 	\orcid{0000-0003-1938-2902}
 	\email{zhangmin@sei.ecnu.edu.cn}
 	\affiliation{%
      \institution{East China Normal University, Shanghai Institute of Intelligent Science and Technology}
      \city{Shanghai}
      \country{China}
    }
	
	\begin{abstract}
The robustness of deep neural networks is  crucial to modern AI-enabled systems and should be formally verified. Sigmoid-like neural networks have been adopted in a wide range of applications. Due to their non-linearity, Sigmoid-like activation functions are usually \emph{over-approximated} for efficient verification, which inevitably introduces imprecision. Considerable efforts have been devoted to finding the so-called \emph{tighter} approximations to obtain more precise verification results. However, existing tightness definitions are heuristic and lack theoretical foundations. We conduct a thorough empirical analysis of existing \emph{neuron-wise} characterizations of tightness and reveal that they are superior only on specific neural networks. We then introduce  the notion of \emph{network-wise tightness} as a unified tightness definition and show that computing network-wise tightness is  a complex non-convex  optimization problem. We bypass the complexity from different perspectives via two efficient, provably tightest approximations. The results demonstrate the promising  performance achievement of our approaches over state of the art: (i) achieving up to 251.28\% improvement to certified lower robustness bounds; and (ii) exhibiting notably more precise verification results on convolutional networks. 
\vspace{-1mm}
\end{abstract}


\begin{document}	
	
	\maketitle
	
	\section{Introduction}
The reliability concerns about deep neural networks (DNNs) are increasing more drastically than ever, especially as such networks  
are being embedded into software systems to make them intelligent. 
 Considerable efforts from both AI and software engineering communities have been devoted
to achieving \emph{robust} DNNs by leveraging testing and verification techniques  \cite{wang2021beta,tjeng2017evaluating,weng2018towards,DunnPKM21,PaulsenWWW20,BalutaCMS21,YanCZTWW21,singh2019abstract}. Among these attempts, formal methods have been demonstrated effective in offering certified robustness guarantees,  giving birth to an emerging research field called \textit{Trustworthy AI} \cite{wing2021trustworthy}. One distinguishing feature of formal methods is that they could provide rigorous proofs of correctness  automatically when the properties are satisfied or disprove them by counterexamples (i.e., witnesses to the violations) \cite{clarke1997model,baier2008principles}. Robustness is an important correctness property in DNN verification: Minor modifications to the neural network's inputs must \emph{not} alter its outputs~\cite{carlini2017towards}. Guaranteeing robustness is indispensable to prevent AI-enabled systems from environmental perturbations and adversarial attacks.

Formal robustness verification of DNNs has been well studied in recent years \cite{pulina2010abstraction,katz2017reluplex,weng2018towards,tjeng2017evaluating,gehr2018ai2,wang2018efficient,ehlers2017formal,PaulsenWWW20,WengZCSHDBD18,wang2021beta}. Most efforts are focused on the \emph{ReLU networks} that  only use the simple piece-wise ReLU activation function. Despite their wide adoptions in modern AI-enabled systems, another notable class of S-shaped (or Sigmoid-like) activation functions, such as Sigmoid, Tanh, and Arctan, have not attracted much attention yet. Due to their non-linearity, Sigmoid-like activation functions are far more complex to be verified. A \textit{de facto} solution is to over-approximate such functions by linear bounds and to transform the verification problem into efficiently solvable linear programming. 
Many state-of-the-art DNN verification techniques, e.g., abstract interpretation \cite{gehr2018ai2,singh2019abstract}, symbolic interval propagation~\cite{wang2018efficient}, model checking  \cite{pulina2010abstraction}, differential verification \cite{PaulsenWWW20}, reachability and output range analysis \cite{tran2020nnv,dutta2018output}, are based on  linear approximation.  

Over-approximation inevitably introduces  imprecision, rendering approximation-based verification incomplete: \emph{Unknown} results may be returned when the neural network's robustness cannot be verified. Considerable efforts have been devoted to finding the so-called \emph{tighter} approximations to achieve more precise verification results. For example,  a larger certified lower robust bound \cite{boopathy2019cnn,lyu2020fastened} (the perturbation distance under which a neural network is proved robust against any allowable perturbation) is preferable in approximation. Several characterizations of tightness and  approximation approaches have been proposed for Sigmoid-like activation functions~\cite{boopathy2019cnn,zhang2018efficient,lyu2020fastened,wu2021tightening,HenriksenL20,lin2019robustness}. However, they are all heuristic and lack theoretical foundations for the individual outperformance.
 
We conduct a thorough empirical analysis of existing approaches and reveal that they are  superior 
only on specific neural networks.
In particular, we have found that the claimed tighter approximation actually produces smaller certified lower bounds according to the tightness defined and observed frequent occurrences of such cases.

Motivated by these observations, we introduce the notion of \emph{network-wise tightness} as a \emph{unified} tightness definition to characterize linear approximations of Sigmoid-like activation functions. This new definition ensures that a tighter approximation can always compute larger certified lower bounds (i.e., larger safe radius). However, we show that it unfortunately implies that computing the tightest approximation is essentially a network-wise non-convex optimization problem~\cite{lyu2020fastened}, which is hard to solve in practice \cite{pathak2020non}. 
 
We bypass the complex optimization problem from two different perspectives, depending on the neural network architecture. For the networks \emph{with only one hidden layer}, we leverage a gradient-based searching algorithm for computing the tightest approximations. Regarding the networks  \emph{with multiple hidden layers}, based on our  empirically study of  the state-of-the-art tools, we  have gained an insight that \emph{a larger robust  bound can be computed when the intervals keep to be tighter during the layer-by-layer propagation}.  Based on this insight, we propose a \emph{neuron-wise} tightest approximation and prove that it guarantees the \emph{network-wise tightest approximation} when the networks are of non-negative weights. Such networks have been demonstrated suitable in a wide range of applications such as effective defense for adversarial attacks in malware and spam detection \cite{DBLP:journals/corr/abs-1806-06108,ceschin2019shallow,kargarnovin2021mal2gcn} and  balancing accuracy and robustness in autoencoding~
\cite{DBLP:conf/icassp/NeacsuPB20,ali2017automatic}.

We have implemented a prototype of our approach called \textsc{NeWise}\footnote{Our code is available at \url{https://github.com/FormalAIze/NeWise.git}.} and  extensively compared it to three state-of-the-art tools, namely \textsc{DeepCert} \cite{wu2021tightening}, \textsc{VeriNet} \cite{HenriksenL20}, and  \textsc{RobustVerifier} \cite{lin2019robustness}. Our experimental results show that \textsc{NeWise} (i) achieves up to 251.28\% improvement to certified lower robustness bounds in the provably tightest cases and (ii) exhibits up to 122.22\% improvement to  certified lower robustness bounds on convolutional networks.

To summarize, this paper makes three major contributions:

\begin{enumerate}

\item We have introduced a novel unified definition of \emph{network-wise tightness} to characterize the tightness of linear approximations for neural network robustness  verification. 

\item We have identified two cases where we can efficiently achieve  provably tightest approximations; the corresponding approaches have been proposed. 

\item We have implemented a verification tool and conducted comprehensive evaluation on its effectiveness and efficiency over three state-of-the-art verifiers.  

\end{enumerate}

The remainder of this paper proceeds as follows: 
Section \ref{sec:prelim} gives preliminaries on  robustness verification of neural networks. 
Section \ref{sec:linear} shows the tightness measurements of linear approximations and introduce our notion of network-wise tightness. Sections \ref{sec:1layer} and \ref{sec:appr} present our provably tightest approximations from two different perspectives, respectively. 
Section \ref{sec:exp} describes our evaluation results. We discuss related work in Section \ref{sec:work} and conclude in Section \ref{sec:con}.

	\section{Preliminaries}
\label{sec:prelim}

\begin{figure}[t]
	\centering
	\includegraphics[width=0.48\textwidth]{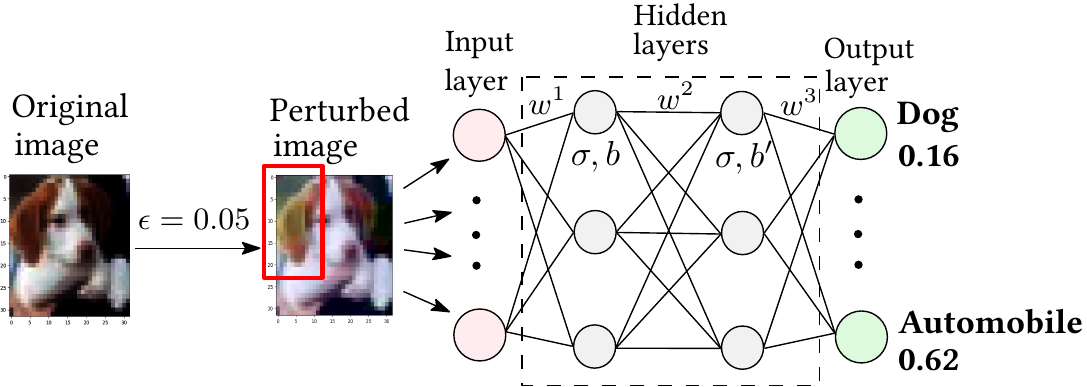}
	\caption{A perturbed image of a dog is misclassified to an automobile with 62\% probability in 0.05 perturbation radius. 
	}
	\label{robust_veri_pro}
	\vspace{-3mm}
\end{figure}

\subsection{Robustness Verification of Neural Networks}

\subsubsection{Deep Neural Network}

A deep neural network is a directed network, where the nodes are called neurons and arranged layer by layer. Each neuron is associated with an activation function $\sigma(x)$ and a bias $b$. 
Except for the first layer, the neurons on a layer are connected to those on the preceding layer, as shown in Figure \ref{robust_veri_pro}. 
Every edge is associated with a weight, which is computed by training. 
The first and last layers are called input and output layers, respectively. The others between them are called hidden layers.

The \emph{execution} of a neural network follows the style of layer-by-layer propagation. 
Each neuron on the input layer admits a number. The number is multiplied by the weights on the edges and then passed to the successor neurons on the next layer. All the incoming numbers are summed. The summation is fed to the activation function $\sigma$ and the output of $\sigma$ is added with the bias $b$. The result is then propagated to the next layer until reaching the output layer.

Formally, a $k$-layer neural network is a function $f:\mathbb{R}^n\rightarrow \mathbb{R}^m$ of the form $f^k\,\circ\,\sigma^{k-1}\,\circ\,\ldots\,\circ\,\sigma^{1}\,\circ\,f^{1}$, with $\sigma^t$ being a non-linear and differentiable activation function for $t$-th layer. The function $f^t$ is either an affine transformation \emph{or} a convolutional operation: 
\begin{linenomath*}   
	\begin{align}
		f(x)&=Wx+b, \tag{Affine Transformation}\\
		f(x)&=W*x+b, \tag{Convolutional Operation}
	\end{align}
\end{linenomath*}   
where $W$, $b$, and $*$ refer to the weight matrix,  the bias vector, and the convolutional product, respectively. In this work, we focus on the networks with the Sigmoid-like activation functions i.e., Sigmoid, Tanh, and Arctan, which are defined as follows, respectively. 
\begin{align*}
\sigma(x) = \frac{1}{1+e^{-x}}, \hspace{5mm}
\sigma(x) = \frac{e^x - e^{-x}}{e^x + e^{-x}},
\hspace{5mm} \sigma(x) = tan^{-1}(x) \notag 
\end{align*}

The output of a neural network  $f$ is a vector of $m$ floating numbers between 0 and 1, denoting the probabilities of classifying an input to the $m$ labels. Let $S$ be the set of $m$ classification labels for the network $f$. We use ${\mathcal L}(f(x))$ to represent the output label for the input $x$ with 
\begin{linenomath*}   
\begin{align}\notag 
{\mathcal L}(f(x))=\mathop{\arg\max}_{s\in S} f(x)[s]. 	
\end{align}
\end{linenomath*}   
Intuitively, ${\mathcal L}(f(x))$ returns a label $s$ in $S$ such that $f(x)[s]$ is maximal among the numbers in the output vector. 

\subsubsection{Robustness and Robustness Verification}

Neural networks are essentially ``programs'' composed by computers by fine-tuning the weights in the networks from training data. Unlike the handcrafted programs developed by programmers, neural networks lack formal requirements and are almost inexplicable, making it very challenging to formalize and verify their properties.

A neural network is called \emph{robust} if reasonable perturbations to its inputs do not alter the classification result. A perturbation is typically measured by the distance between the perturbed input $x'$ and the original one $x$ by using $\ell_p$-norm, denoted by $||x-x'||_{p}\triangleq \sqrt[p]{|x_1-x'_1|^p+\ldots+|x_n-x'_n|^p}$, where $p$ can be $1$, $2$ or $\infty$, and $n$ is the length of the vectors $x$. In this work, we consider the most general case when $p=\infty$.

\begin{example}
	We consider an example to explain how a perturbed image is misclassified. As shown in Figure \ref{robust_veri_pro}, a normal image of a dog  can be correctly classified by a neural network. We assume that the image can be perturbed within a 0.05 distance under $\ell_\infty$-norm. There exists a perturbed image such that when it is fed into the network, the outputs of the two neurons labeled by \emph{dog} and \emph{automobile} are 0.16 and 0.62, respectively. It indicates that the image is classified to a dog (resp. automobile) with the probability of 16\% (reps. 62\%). 
	Therefore, it is classified to be an automobile, although it still represents a dog to human eyes, apparently. 
\end{example}

The robustness of a neural network can be quantitively measured by a lower bound $\epsilon$, which refers to a safe perturbation distance such that any perturbations below $\epsilon$ have the same classification result as the original input to the neural network.

\begin{definition}[Local Robustness]\label{robustness property}
Given a neural network $f$, an input $x_{0}$, and a bound  $\epsilon$ under $\ell_p$-norm, $f$ is called robust  w.r.t. $x_{0}$ iff $\mathcal{L}(f(x))=\mathcal{L}(f(x_{0}))$ holds for each $x$ such that $||x-x_{0}||_{p}\leq\epsilon$. Such $\epsilon$ is called a certified lower bound.
\end{definition}

The twin problems of verifying $f$'s  robustness are: (i) to  prove that, for each $x$ satisfying   $||x-x_{0}||_{p}\leq\epsilon$,  
\begin{align}\label{eq:rob}
f_{s_0}(x)-f_{s}(x)>0
\end{align}
holds for each $s\in S-\{s_0\}$, where  $s_0={\mathcal L}(f(x_0))$ and $f_{s}(x)$ returns the probability, i.e., $P({\mathcal L}(f(x))=s)$, of classifying $x$ to the label $s$ by $f$; and (ii) to  compute a certified lower bound --- a larger certified lower bound implies a more precise robustness verification result. As directly computing $\epsilon$ is difficult due to the non-linearity of the constraint (\ref{eq:rob}), 
most of the state-of-the-art approaches \cite{boopathy2019cnn,zhang2018efficient,wu2021tightening} adopt the efficient binary search algorithm to first determine a candidate $\epsilon$ and then check whether (\ref{eq:rob}) is true or false on $\epsilon$.

\subsection{Approximation-based Robustness Verification}

A neural network $f$ is  highly non-linear due to the inclusion of activation functions.  Proving Formula (\ref{eq:rob}) is computationally expensive, e.g., NP-complete even for the simplest fully connected ReLU networks \cite{katz2017reluplex,salzer2021reachability}. 
Many approaches have been investigated to improve the verification efficiency while sacrificing completeness. Representative methods include interval analysis  ~\cite{wang2018formal}, abstract interpretation~\cite{singh2019abstract,gehr2018ai2}, and output range estimation~\cite{dutta2018output,xiang2018output}, etc. 
The technique underlying these approaches is to over-approximate the non-linear activation functions using linear constraints, which can be more efficiently solved than the original ones.

Instead of directly proving Formula (\ref{eq:rob}), the approximation-based approaches  over-approximate both $f_{s_0}(x)$ and $f_{s}(x)$ by two linear constraints and prove that the lower linear bound  $f_{L,s_0}(x)$ of $f_{s_0}(x)$ is greater than the upper linear bound   $f_{U,s}(x)$ of $f_{s}(x)$. Apparently, $f_{L,s_0}(x)-f_{U,s}(x)>0$ is a sufficient condition of Formula (\ref{eq:rob}), and it is significantly more efficient to prove or disprove. Therefore, it is widely adopted in neural network verification \cite{wu2021tightening,boopathy2019cnn,HenriksenL20,lin2019robustness}, although it may produce false positives when it is disproved.

\begin{definition}[Upper/Lower linear bounds]\label{def:linearbounds}
	Let  $\sigma(x)$ be a non-linear function with $x\in [l,u]$, $\alpha_{L},\alpha_{U},\beta_{L},\beta_{U} \in \mathbb{R}$, and
		\begin{align}
			h_{U}(x)=\alpha_{U}x+\beta_{U},\qquad h_{L}(x)=\alpha_{L}x+\beta_{L}. 
		\end{align} 
	$h_{U}(x)$ and $h_{L}(x)$ are called upper and lower linear bounds of
	$\sigma(x)$ if the following condition holds:
		\begin{align}\label{linearboundcondition}
			\forall x \in [l,u], \quad h_{L}(x)\leq  \sigma(x)\leq h_{U}(x).
		\end{align}
\end{definition}

Over-approximating the non-linear activation functions using linear lower and upper bounds is the key to the approximation of a neural network. For each activation function $\sigma$ on a domain $[l,u]$, we define an upper linear bound $h_U$ and a lower one $h_L$ to ensure that for all $x$ in $[l,u]$, $\sigma(x)$ is enclosed in $[h_L(x),h_U(x)]$. 

\begin{figure}[t]
	\includegraphics[width=0.46\textwidth]{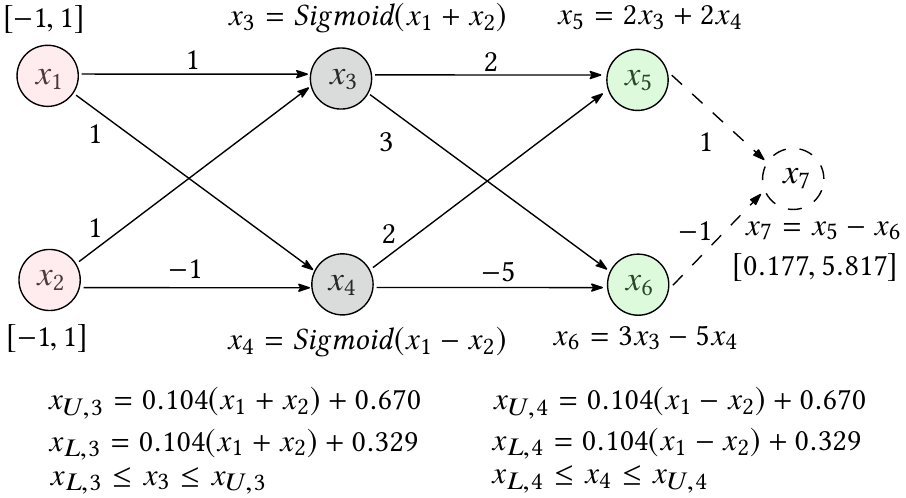}
	\caption{An example of approximation-based verification.}
	\label{fig:veriexample}
	\vspace{-3mm}
\end{figure}

Given an input range as in Definition \ref{eq:rob}, the output ranges of a network are computed by propagating the output interval of each  neuron as in Definition \ref{def:linearbounds} to the output layer.

\begin{example}\label{exp:verify}
	We consider an example of verifying a simple neural network based on  approximation, as shown in Figure \ref{fig:veriexample}. The original verification problem is to prove that for any input $(x_1,x_2)$ with $x_1\in[-1,1]$ and $x_2\in[-1,1]$, it is always classified to the label of neuron $x_5$. That is equivalent to proving that the output of the auxiliary neuron  $x_7=(x_5-x_6)$ is always greater than 0. We define the linear upper/lower bounds $x_{U,3},x_{L,3}$ and $x_{U,4},x_{L,4}$ to over-approximate $x_3$ and $x_4$, respectively. It suffices to prove that $x_{L,5}-x_{U,6}>0$ is always true. We can over-estimate the output interval of $x_{L,5}-x_{U,6}$ is $[0.177,5.817]$ using $x_{U,3},x_{L,3}$ and $x_{U,4},x_{L,4}$ and consequently prove the robustness of the network for all the inputs in $[-1,1]\times [-1,1]$. 
\end{example}

\begin{figure*}[t]
	\centering
	\begin{subfigure}{0.24\textwidth}
		\includegraphics[width=\textwidth]{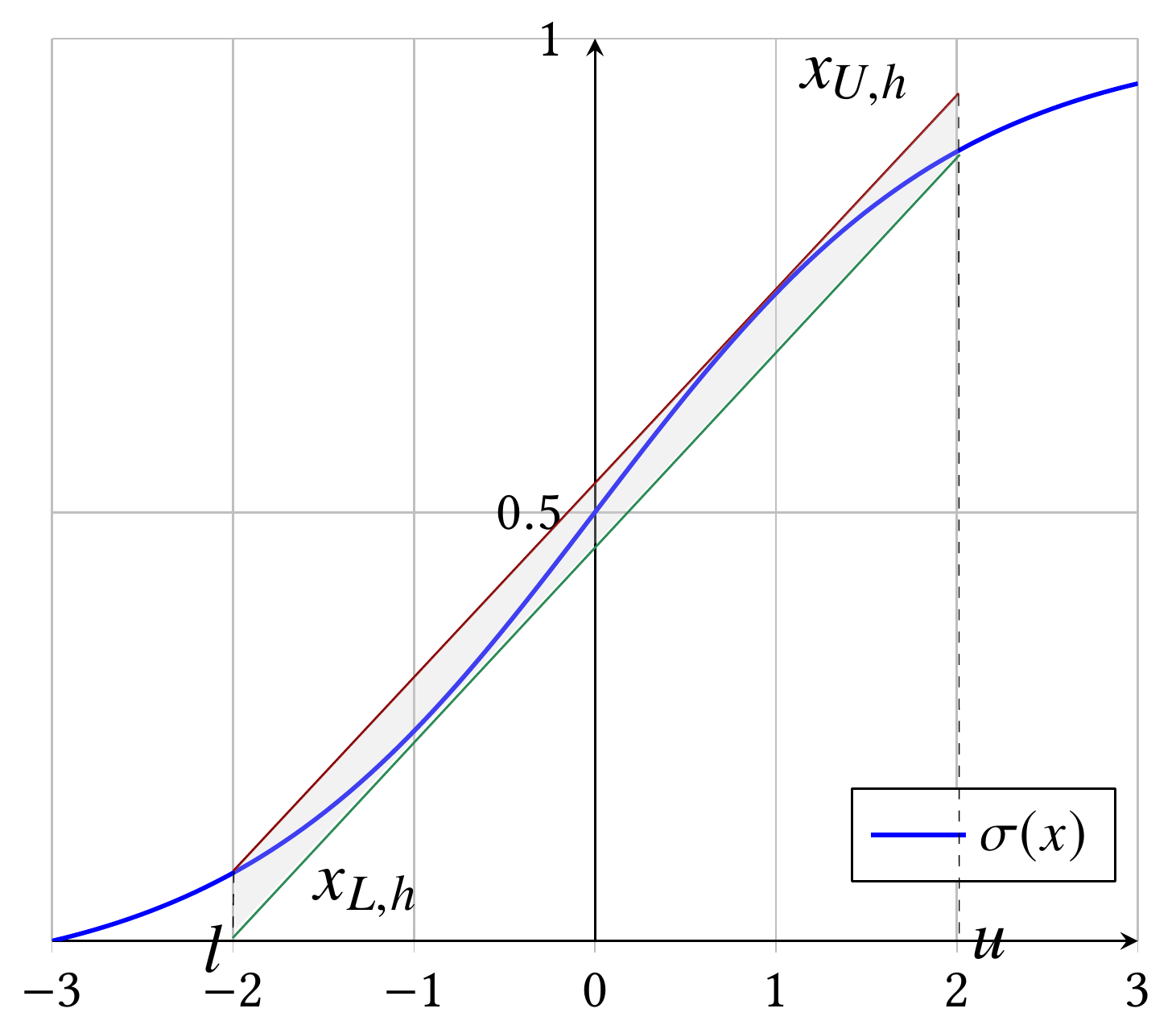}
		\caption{Minimal area and parallel lines.}
		\label{fig:minimal}
	\end{subfigure}
	~\begin{subfigure}{0.24\textwidth}
		\includegraphics[width=\textwidth]{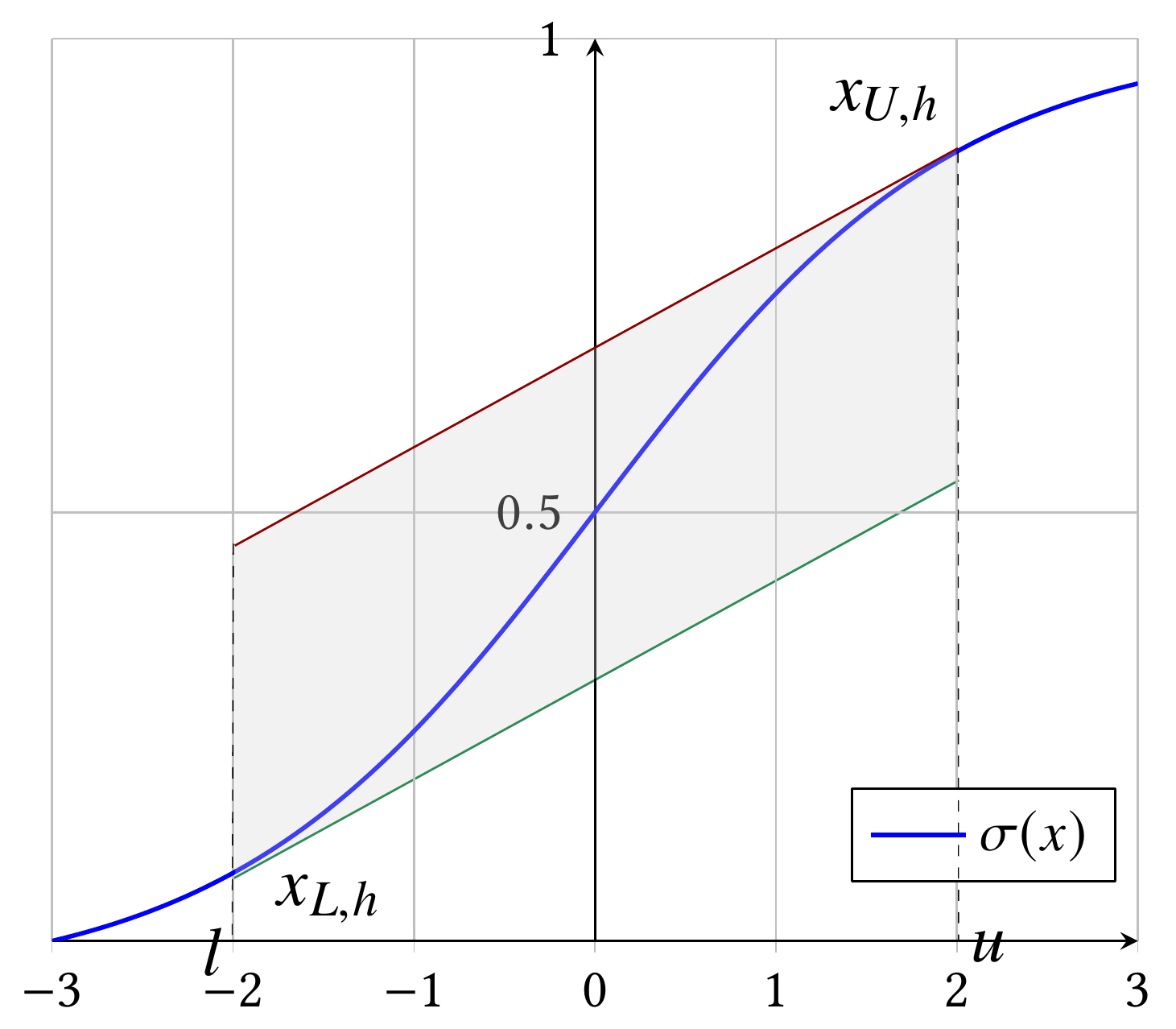}
		\caption{Endpoints (our work).}
		\label{fig:endpoints}	
	\end{subfigure}
	~\begin{subfigure}{0.24\textwidth}
		\includegraphics[width=\textwidth]{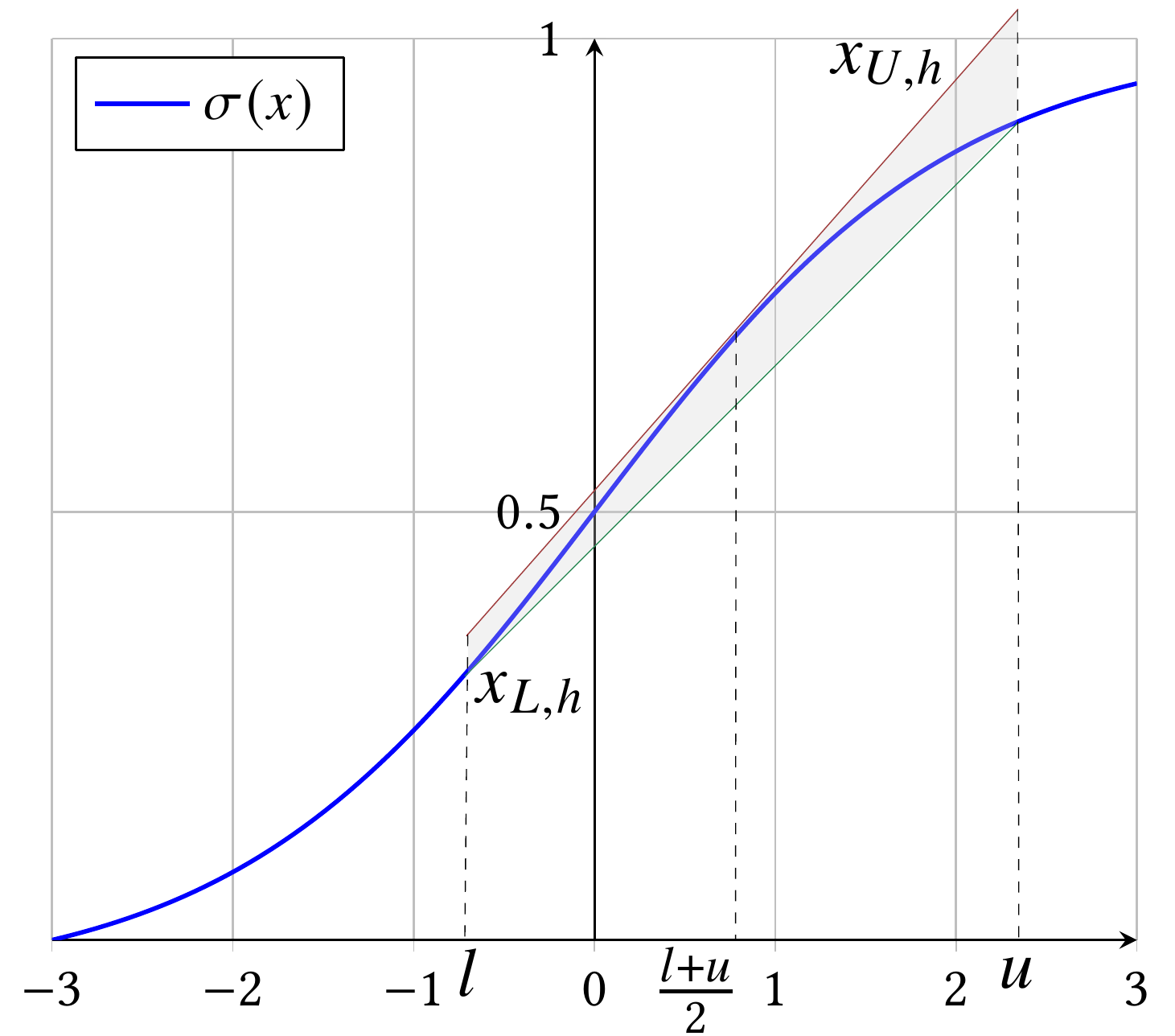}
		\caption{Minimal area \cite{HenriksenL20}.}
		\label{fig:middle}	
	\end{subfigure}
	~\begin{subfigure}{0.24\textwidth}	
		\includegraphics[width=\textwidth]{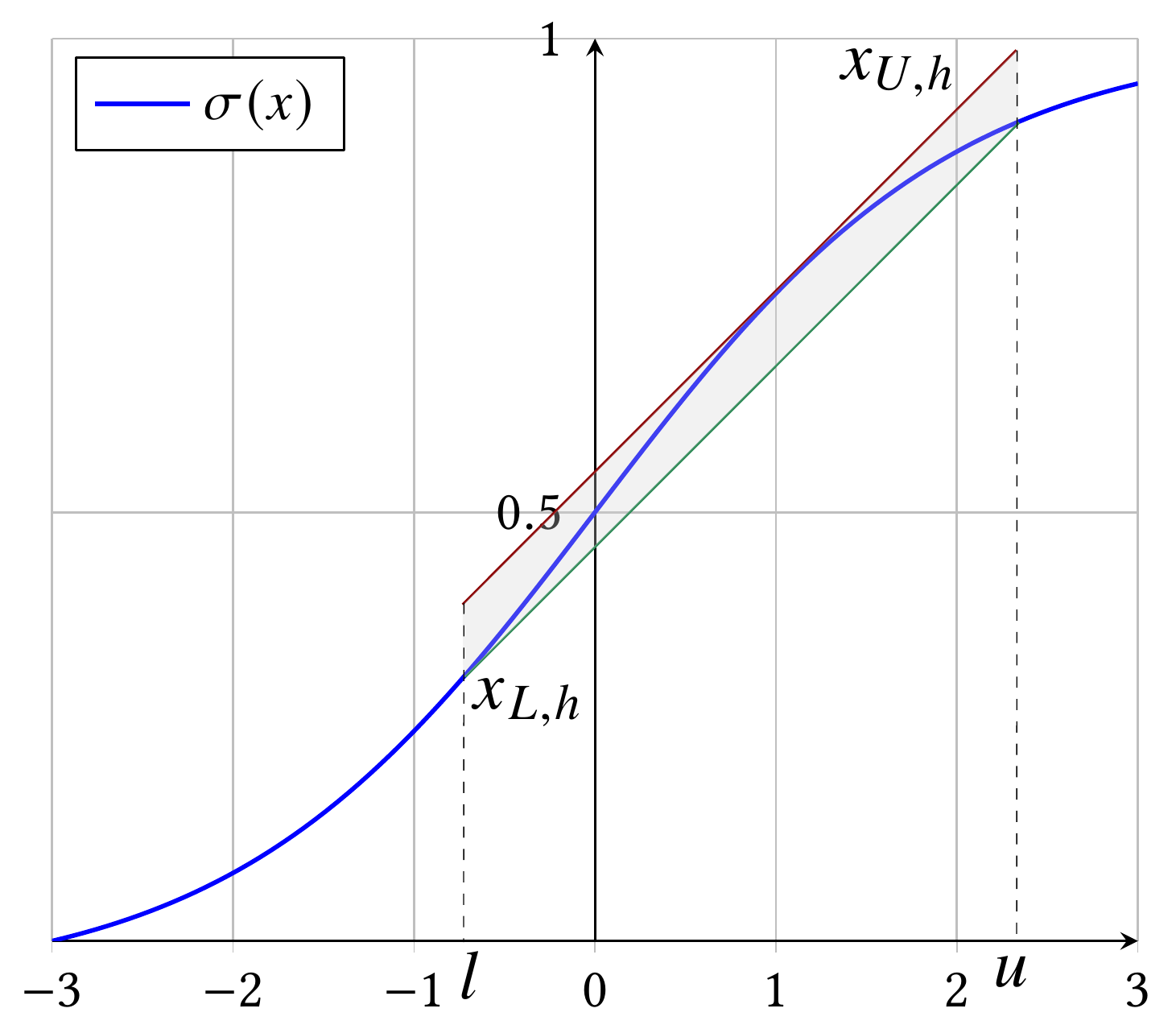}
		\caption{Parallel lines \cite{wu2021tightening}.}
		\label{fig:parallel}	
	\end{subfigure}
	\label{fig:approximations}
	\vspace{-2mm}
	\caption{Different approximations according to the domain of $\sigma(x)$ and their tightness definitions.}   
	\vspace{-2mm}
\end{figure*}

Note that it is not necessary to approximate an activation function using only one linear upper or lower bound. One may consider piece-wise linear bounds made up of a sequence of linear segments to approximate the function more tightly by being closer to it. However, the piece-wise way causes the number of constraints to blow up exponentially when propagated layer by layer \cite{singh2019abstract}. It would drastically reduce the verification scalability. Thus, over-approximating an activation function using one upper linear bound and one lower linear bound is the most efficient and widely-adopted choice for the approximation-based robustness verification approaches.

	\section{Linear Approximation Approaches}\label{sec:linear}
In this section, we analyze the tightness issue of existing approximation approaches and formally define a unified network-wise tightness to characterize the approximations. The network-wise tightness guarantees that output neurons can produce precise output ranges. 
\subsection{The Tightness Issue of Approximations}\label{sub:tightness}
As approximation inevitably introduces overestimation, defining the \emph{tightest} possible approximation is crucial to obtaining precise verification results. Several approximation approaches have been proposed under different strategies.

Henriksen et al. \cite{HenriksenL20} proposed to measure the tightness of approximations using the enclosed area between the bound and the approximated function. An approximation is tighter if the corresponding area is smaller than another. 
By this definition, the approximations to $x_3,x_4$ should be the following linear bounds:
\begin{align}
	x'_{U,3}&= 0.204(x_1+x_2)+0.527,\\
	x'_{L,3}&= 0.204(x_1+x_2)+0.472,\\
	x'_{U,4}&= 0.204(x_1-x_2)+0.527,\\
	x'_{L,4}&= 0.204(x_1-x_2)+0.472.
\end{align}
Figure \ref{fig:minimal} shows the bounds graphically. Apparently, they are closer to the activation function on the interval $[-2,2]$. Surprisingly, using the \textit{tighter} linear bounds the output range of $x_7$ is $[-0.079, 6.073]$, by which we cannot prove and disprove the robustness.  
Wu and Zhang adopted the same strategy for this case in their recent work~\cite{wu2021tightening}. In Example \ref{exp:verify}, we adopt a new strategy by taking the tangent lines at the two endpoints as its upper and lower bounds, as shown in Figure \ref{fig:endpoints}. 
We obtain a 
smaller output range by these bounds, although they are  far less tight than the one in Figure \ref{fig:minimal}.

In another case shown in Figure \ref{fig:middle}, Henriksen et al. \cite{HenriksenL20} proved that the tangent line at the middle point when $x=\frac{l+u}{2}$ is the tight upper bound because the enclosed area between it and the activation function is minimal. In Wu and Zhang's approach, they adopted the tangent line that is parallel to the lower bound as its upper bound, as shown in Figure \ref{fig:parallel}.  
Some other approaches such as \cite{boopathy2019cnn,zhang2018efficient,lin2019robustness} adopt similar approximation strategies, but they have been experimentally proved not as tight as the ones in \cite{HenriksenL20,wu2021tightening}. 

Lyu et al. \cite{lyu2020fastened} proposed a gradient-based searching  approach for computing a tighter approximation if the approximation can produce tighter input intervals for the following neurons. 
However, the experimental results in the work \cite{wu2021tightening} show that this approach neither guarantees it always produces larger certified lower robust bound than other approaches and its scalability is rather limited due to the complexity of the searching algorithm for each neuron.

\begin{table}[t]
	\centering
	\footnotesize 
	\setlength{\tabcolsep}{6.3pt}
	\caption{Tightness evaluation of state of the art: \textsc{DeepCert}\cite{wu2021tightening}, \textsc{VeriNet}\cite{lyu2020fastened}, and \textsc{RobustVerifier}\cite{lin2019robustness}.
		$\mbox{CNN}_{t-c}$ denotes a CNN with $t$ layers and $c$ filters of size 3$\times$3. The models are pre-trained \cite{singh2019abstract,wu2021tightening, tjx_models} by, e.g., \textsc{DeepPoly} \cite{singh2019abstract}.	
		\vspace{-2mm}
		}
	\begin{tabular}{|c|l|r|r|r|r|}
		\hline 
		\multirow{2}{*}{\textbf{Dataset}} & \multirow{2}{*}{\textbf{Model}} & \multirow{2}{*}{\textbf{\#Neurons}} &   \multicolumn{3}{c|}{\textbf{Certified Lower Bound  (Average)}}  \\
		\hhline{~~~---}
		& & & \textsc{DeepCert} & \textsc{VeriNet}  & \textsc{Rob.Ver.}  \\
		\hline
		\hline
		\multirow{11}{*}{MNIST}  
		& 3x50                 & 160    & 0.0076 & \cellcolor{apricot}0.0077 & 0.0065   \\
		& 3x100                & 310    & 0.0086 & \cellcolor{apricot}0.0087 & 0.0074   \\
		& 3x200                & 610    & 0.0091 & \cellcolor{apricot}0.0092 & 0.0079   \\ 
		& 5x100                & 510    & 0.0061 & \cellcolor{apricot}0.0062 & 0.0052 \\
		& 6x500                & 3,010  & \cellcolor{apricot}0.0778 & 0.0776 & 0.0665  \\
		& $\mbox{CNN}_{3-2}$  & 2,514  & 0.0579 & \cellcolor{apricot}0.0580 & 0.0569  \\
		& $\mbox{CNN}_{3-4}$  & 5,018  & \cellcolor{apricot}0.0473 & 0.0472 & 0.0464  \\
		& $\mbox{CNN}_{4-5}$  & 8,680  & 0.0539 & \cellcolor{apricot}0.0543 & 0.0522  \\
		& $\mbox{CNN}_{5-5}$  & 10,680 & 0.0548 & \cellcolor{apricot}0.0550 & 0.0513  \\
		& $\mbox{CNN}_{6-5}$  & 12,300 & \cellcolor{apricot}0.0590 & 0.0588 & 0.0541  \\
		& $\mbox{CNN}_{8-5}$  & 14,570 & 0.0878 & \cellcolor{apricot}0.0882 & 0.0685  \\
		\hline
		\multirow{5}{*}{\makecell{Fashion\\ MNIST}}  
		& 3x50                & 160   & 0.0101 & \cellcolor{apricot}0.0102 & 0.0086   \\   
		& 5x100               & 510   & 0.0078 & \cellcolor{apricot}0.0079 & 0.0066  \\
		& $\mbox{CNN}_{4-5}$ & 8,680 & \cellcolor{apricot}0.0721 & 0.0720 & 0.0666  \\
		& $\mbox{CNN}_{5-5}$ & 10,680& 0.0676 & \cellcolor{apricot}0.0677 & 0.0605  \\
		& $\mbox{CNN}_{6-5}$ & 12,300& \cellcolor{apricot}0.0695 & 0.0691 & 0.0627  \\
		\hline
		\multirow{4}{*}{CIFAR10} 
		& 3x50                & 160    & 0.0045 & \cellcolor{apricot}0.0046 & 0.0042  \\  
		& 5x100               & 510    & \cellcolor{apricot}0.0038 & 0.0037 & 0.0033  \\ 
		& $\mbox{CNN}_{3-2}$ & 3,378  & 0.0312 & \cellcolor{apricot}0.0313 & 0.0311  \\ 
		& $\mbox{CNN}_{6-5}$ & 17,110 & \cellcolor{apricot}0.0224 & 0.0223 & 0.0212  \\
		\hline
	\end{tabular}
	\label{certified lower bound results on general models}
	\vspace{-4mm}
\end{table}

\begin{figure*}	
	\setlength{\tabcolsep}{1pt}
	\begin{tabular}{rrrrrr}
		\begin{subfigure}{0.160\linewidth}
			\centering
			\includegraphics[width=\textwidth]{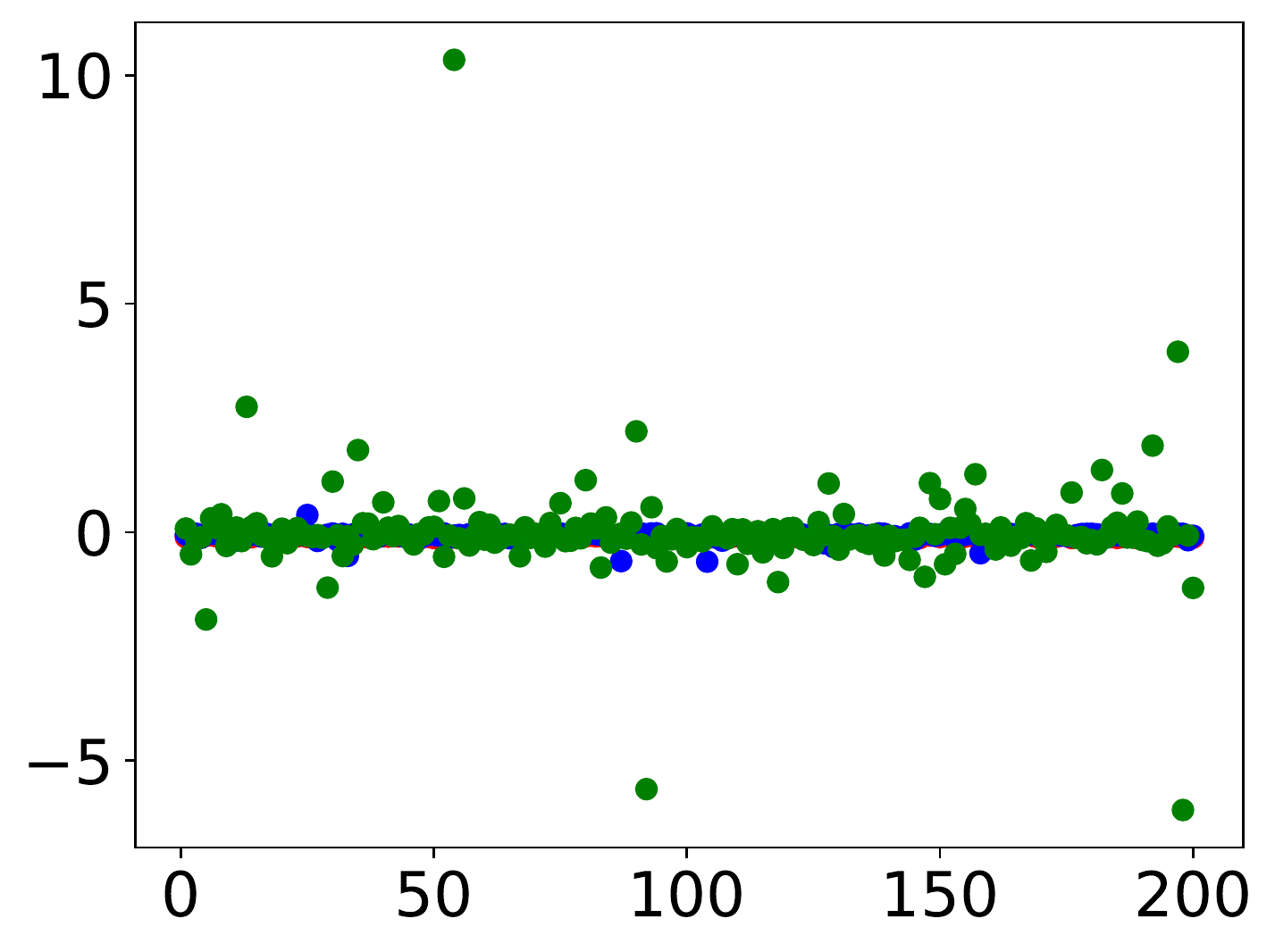}
			\caption{Layer 2 of $N_1$}
			\label{fig:layer2}
		\end{subfigure} &
		\begin{subfigure}{0.163\linewidth}
			\centering
			\includegraphics[width=\textwidth]{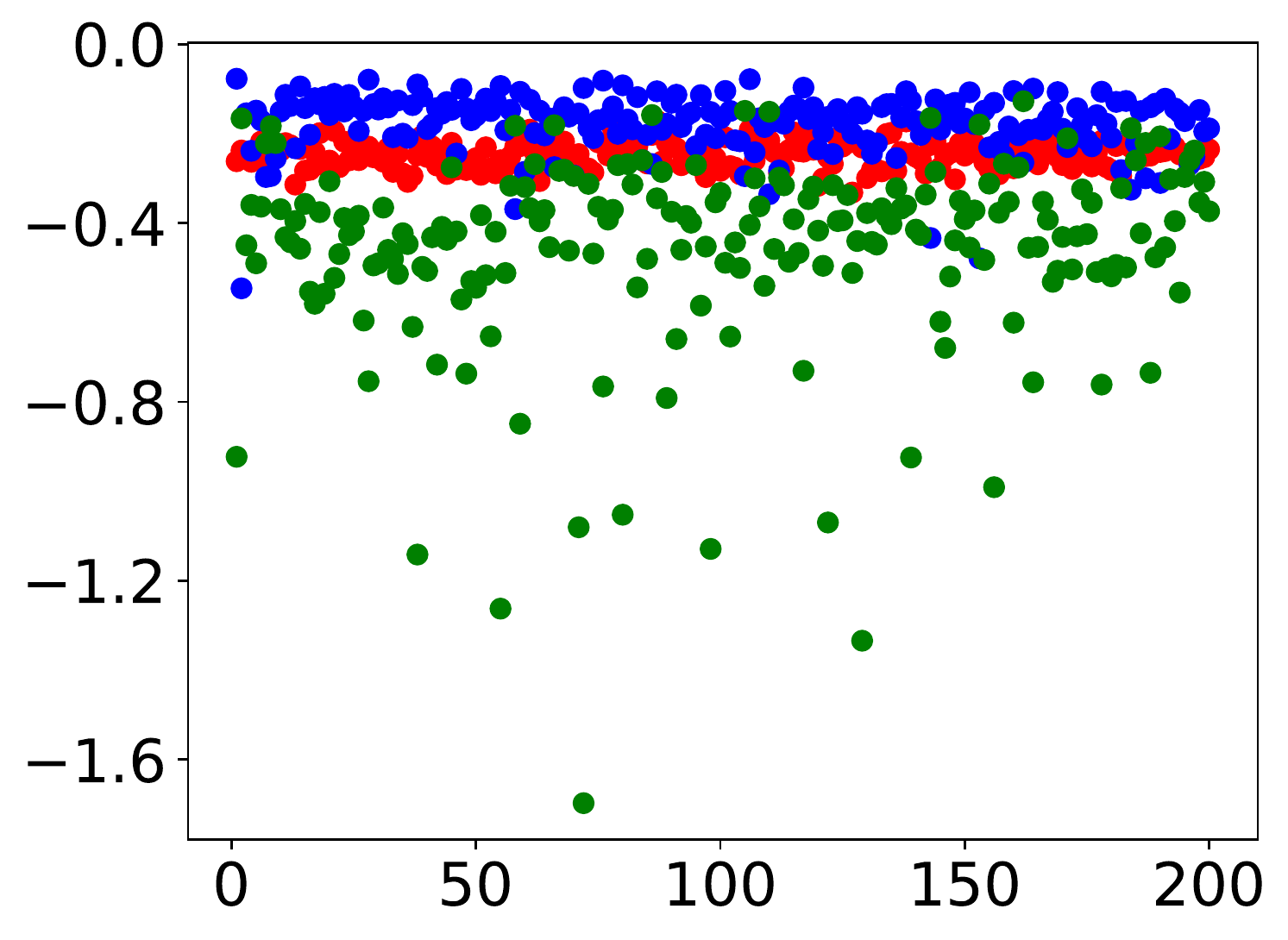}
			\caption{Layer 3 of $N_1$}
		\end{subfigure}&
		\begin{subfigure}{0.162\linewidth}
			\centering		\includegraphics[width=\textwidth]{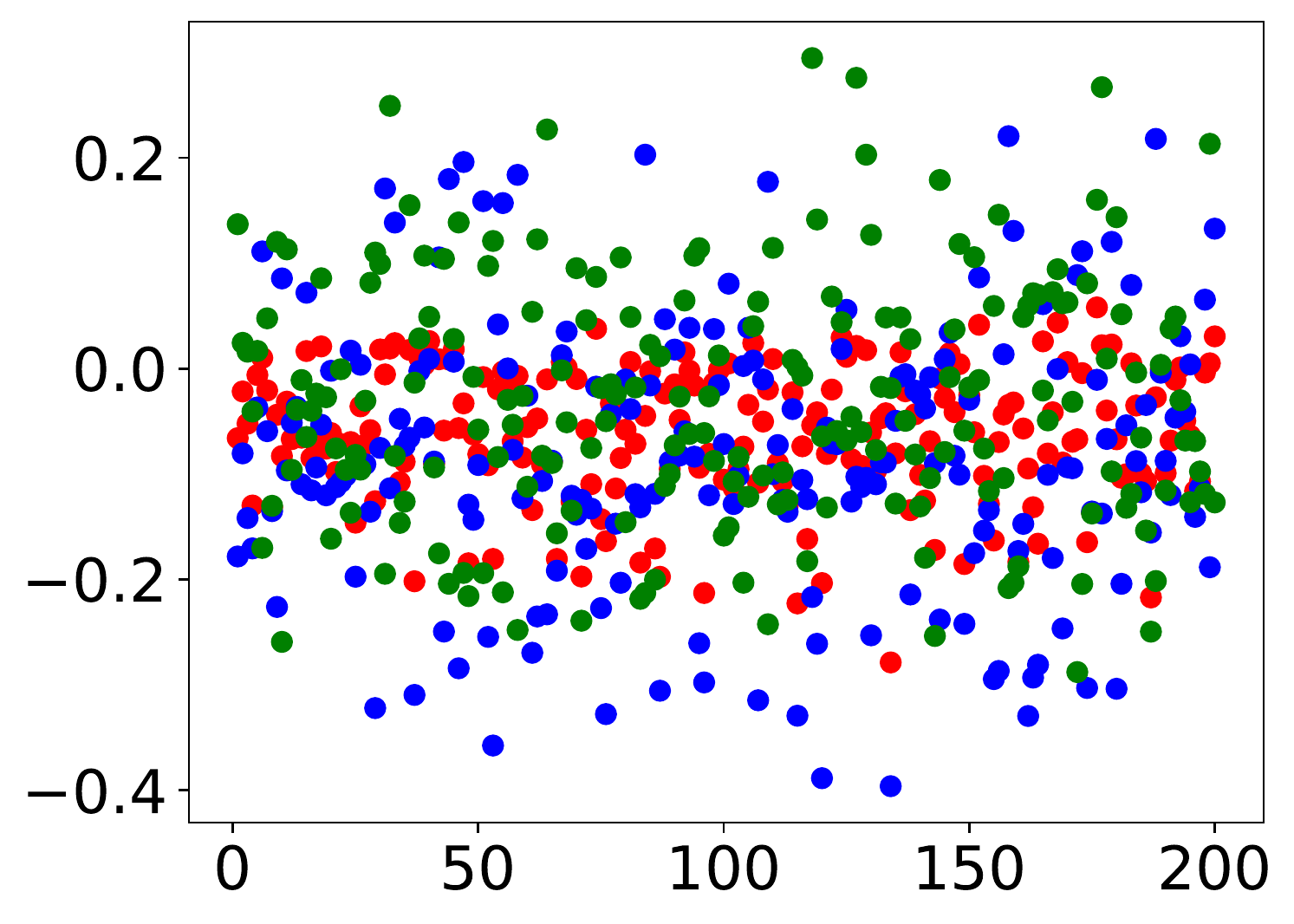}
			\caption{Layer 4 of $N_1$}
		\end{subfigure}&
		\begin{subfigure}{0.162\linewidth}
			\centering
			\includegraphics[width=\textwidth]{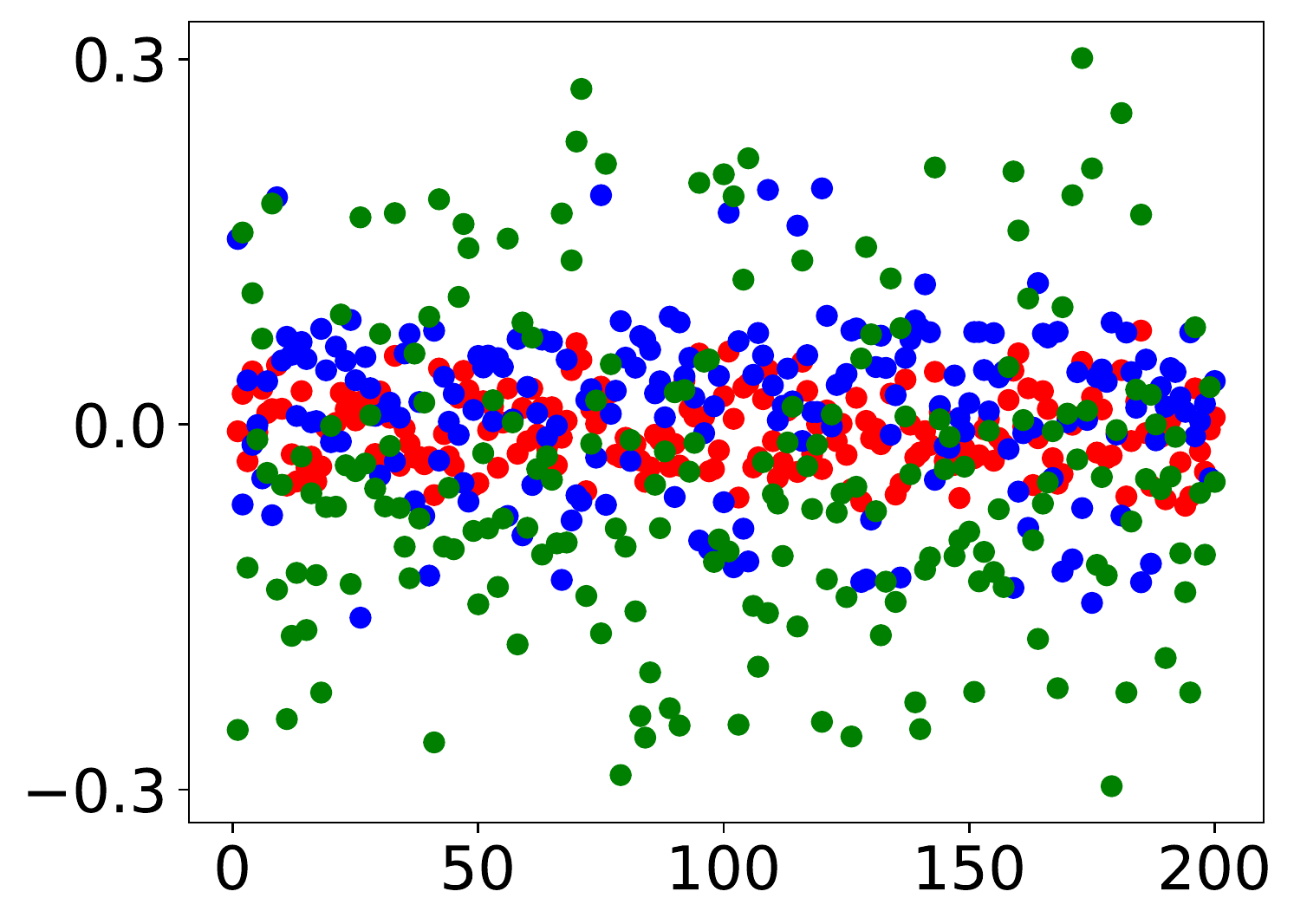}
			\caption{Layer 5 of $N_1$}
		\end{subfigure}&
		\begin{subfigure}{0.162\linewidth}
			\centering
			\includegraphics[width=\textwidth]{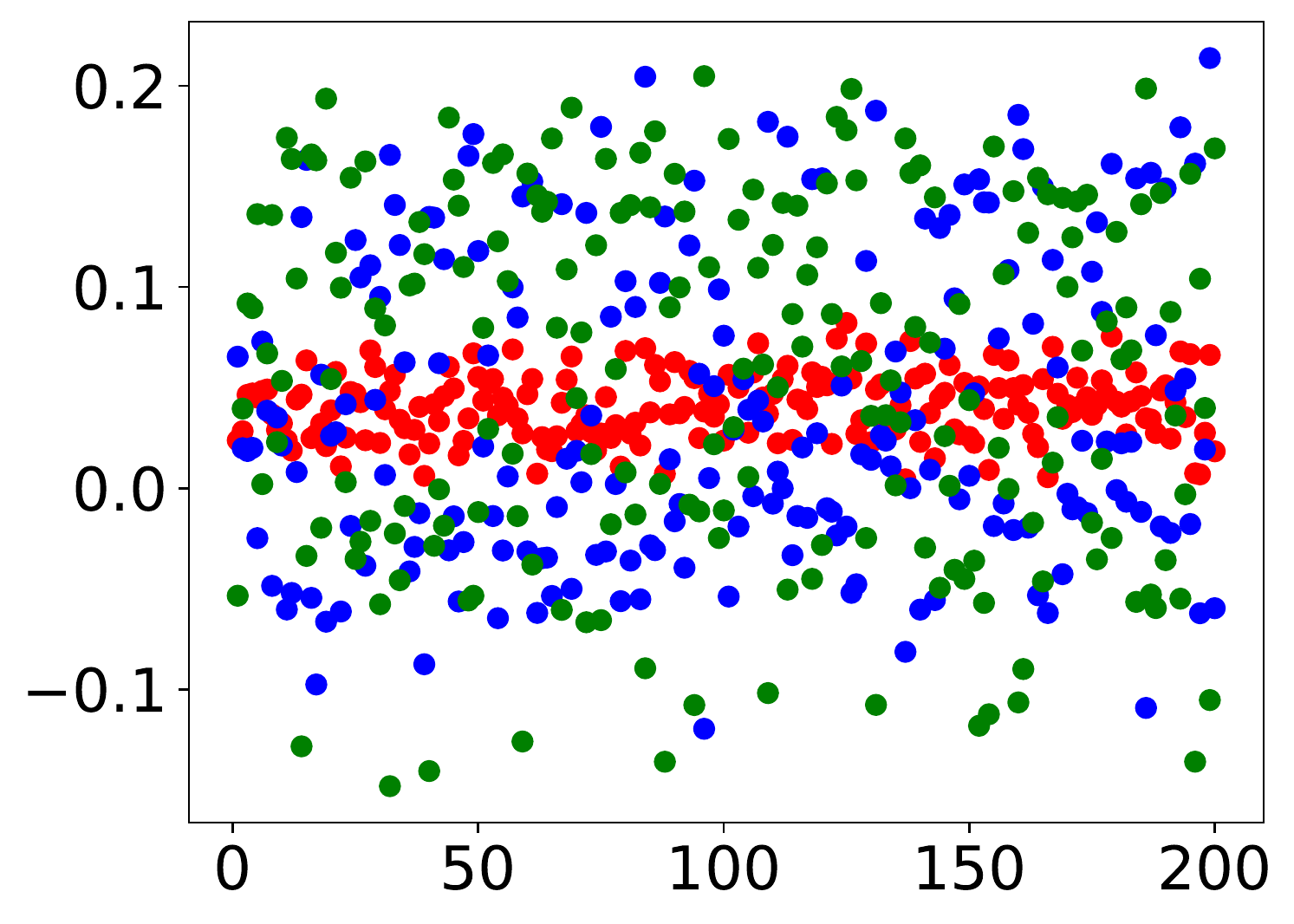}
			\caption{Layer 6 of $N_1$}
		\end{subfigure}&
		\begin{subfigure}{0.162\linewidth}
			\centering
			\includegraphics[width=\textwidth]{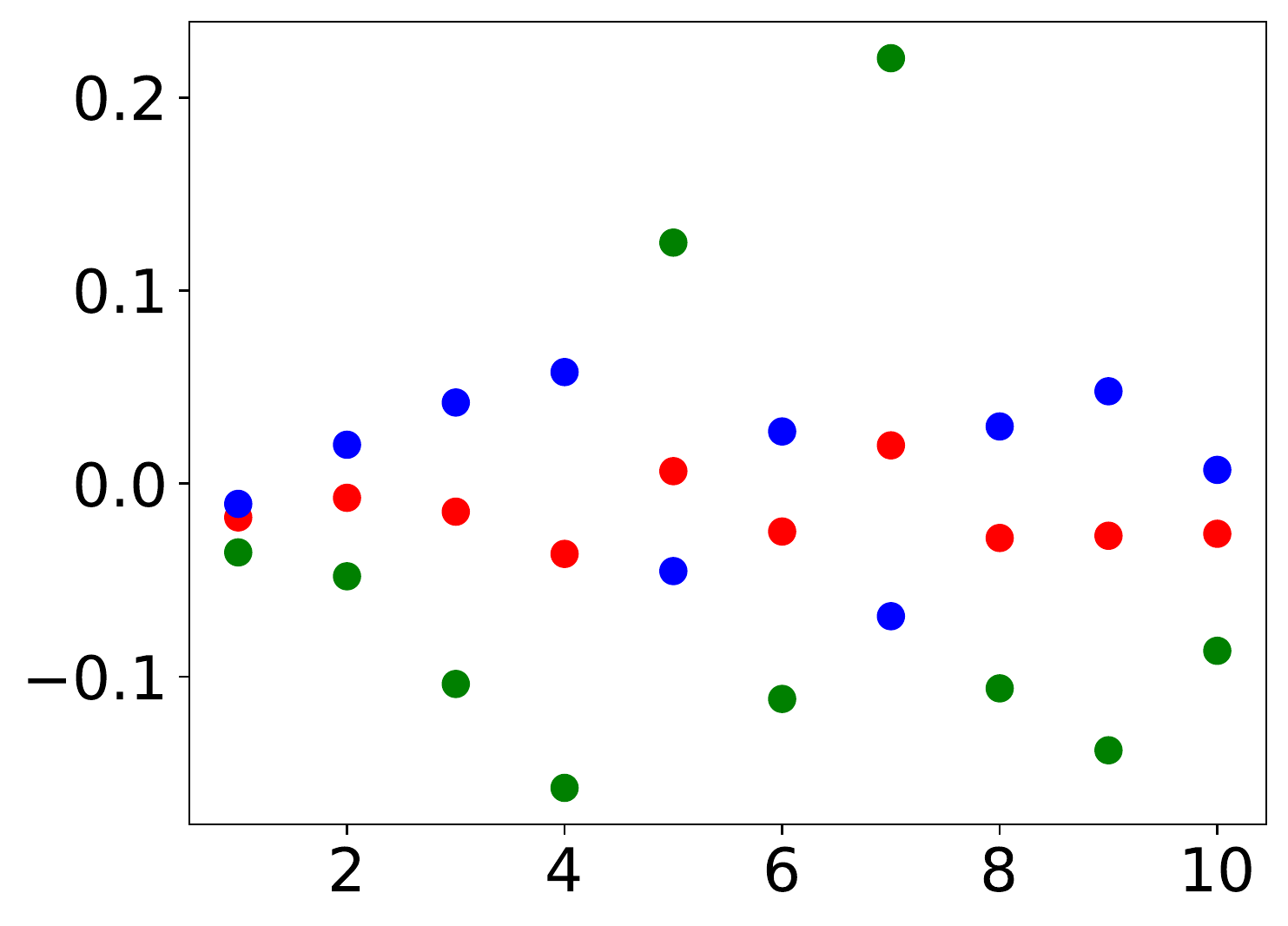}
			\caption{Output layer of $N_1$}		
			\label{fig:output}
		\end{subfigure}
		\\
		\begin{subfigure}{0.162\linewidth}
			\centering
			\includegraphics[width=\textwidth]{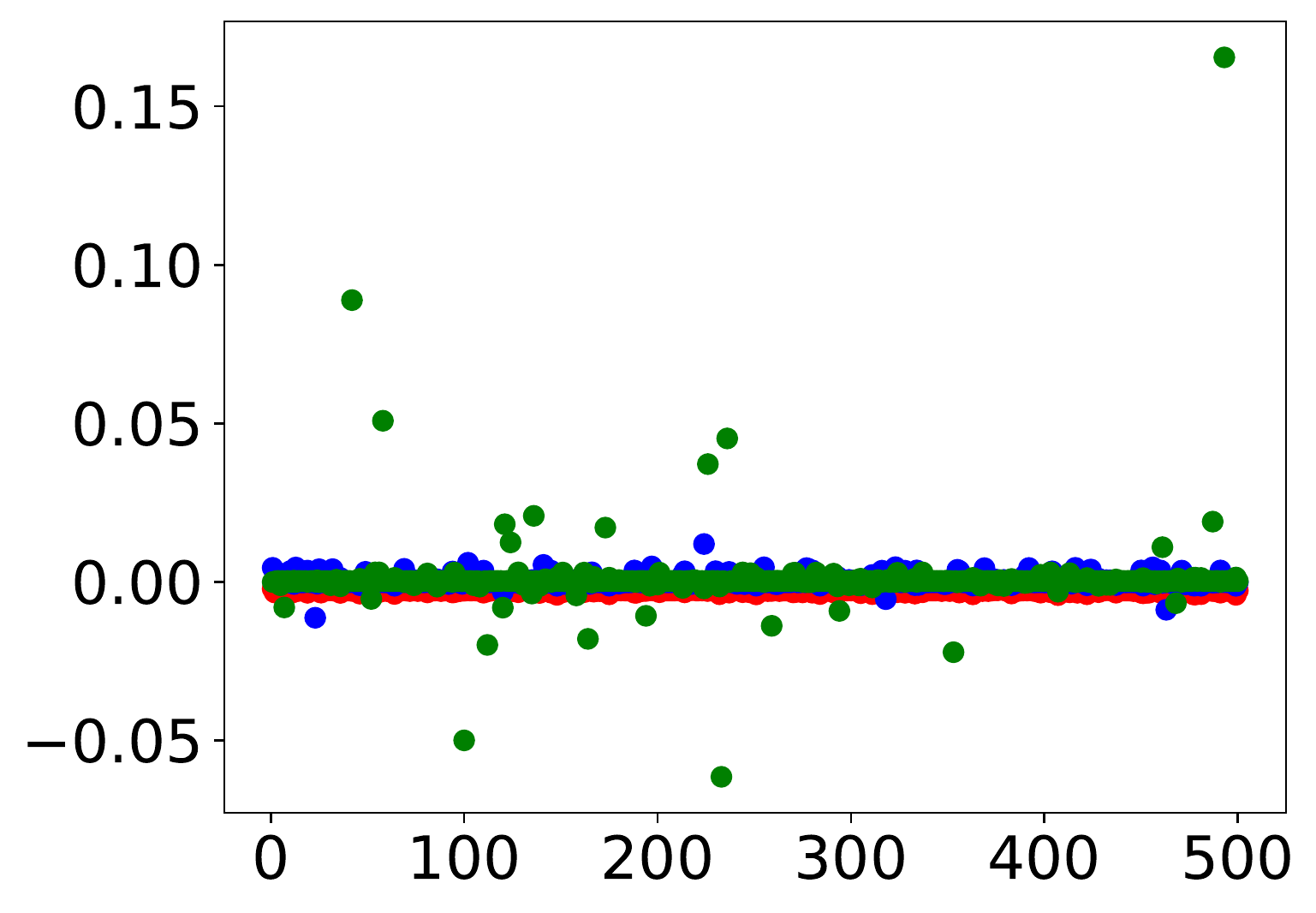}
			\caption{Layer 2 of $N_2$}
			\label{fig:layer2'}
		\end{subfigure}&
		\begin{subfigure}{0.15\linewidth}
			\centering
			\includegraphics[width=\textwidth]{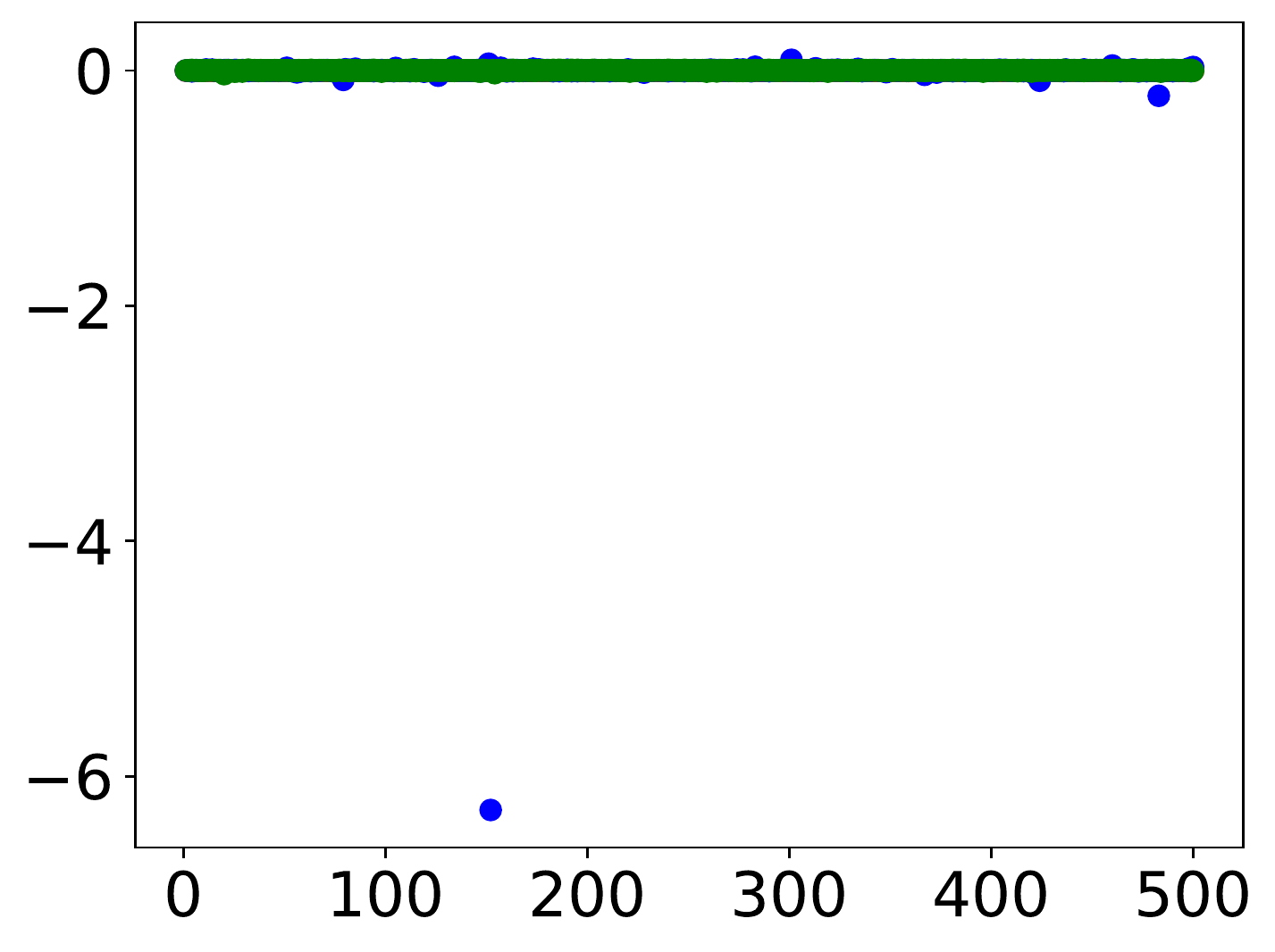}
			\caption{Layer 3 of $N_2$}
		\end{subfigure}&
		\begin{subfigure}{0.15\linewidth}
			\centering
			\includegraphics[width=\textwidth]{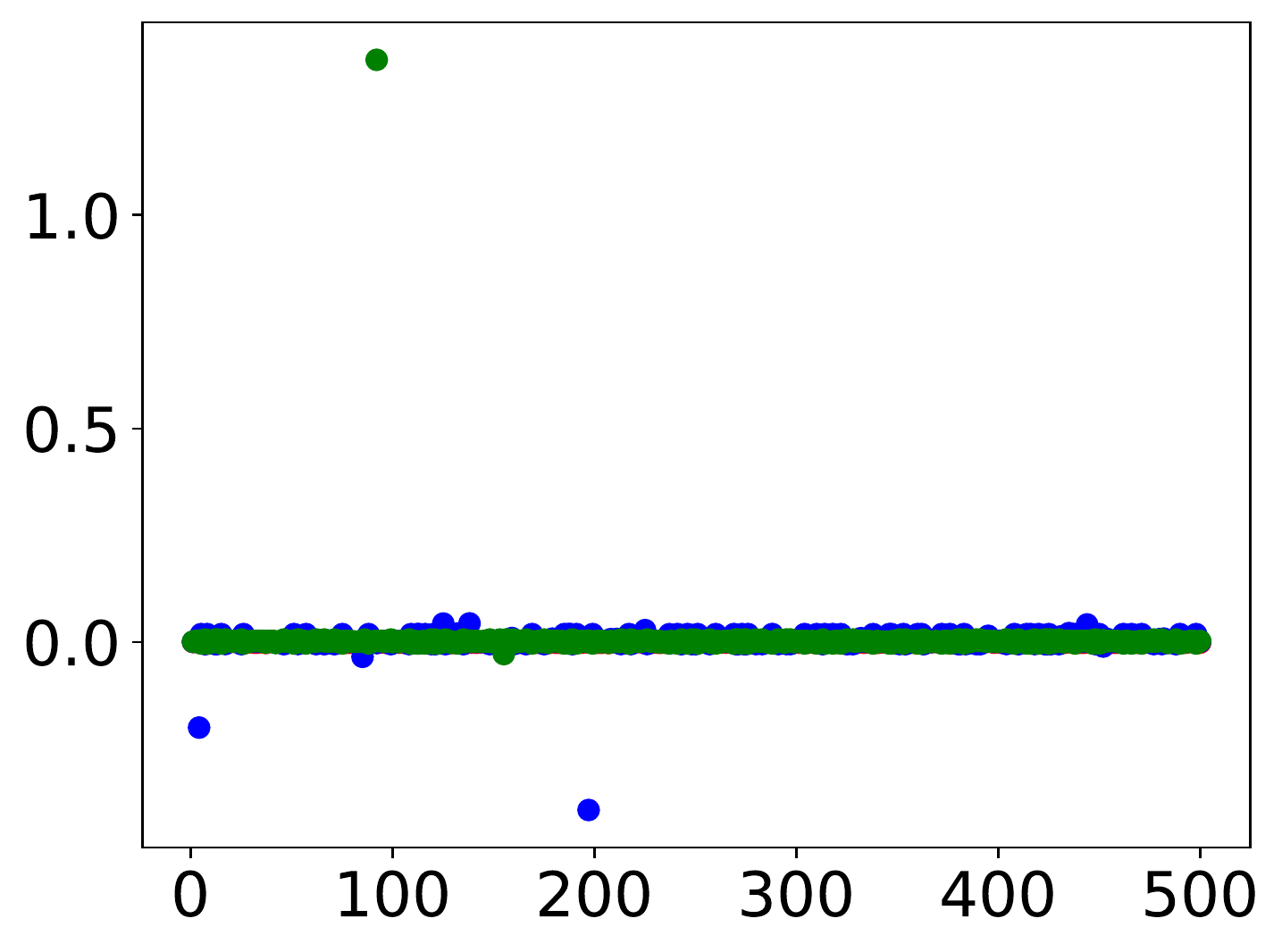}
			\caption{Layer 4 of $N_2$}
		\end{subfigure}&
		\begin{subfigure}{0.162\linewidth}
			\centering
			\includegraphics[width=\textwidth]{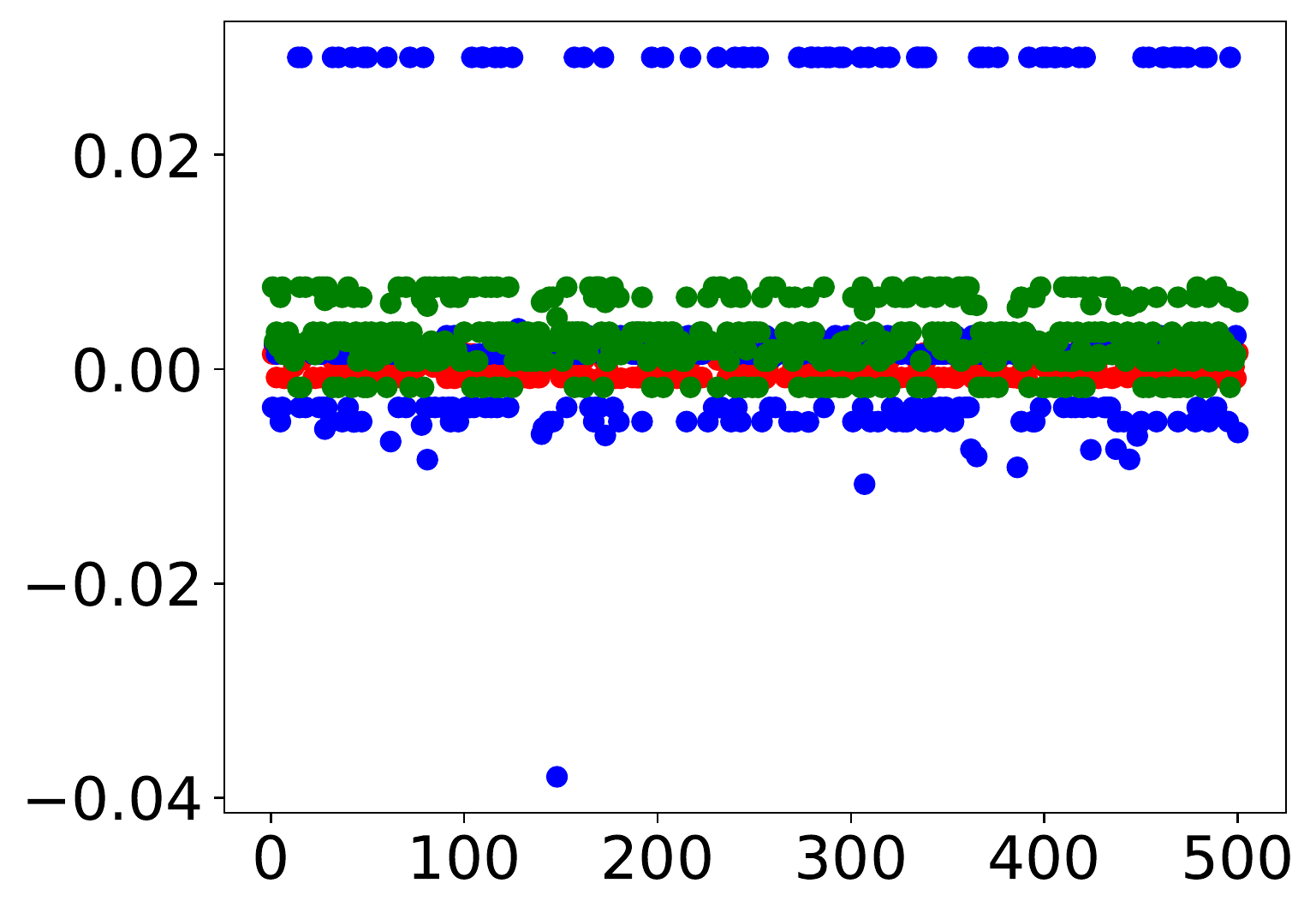}
			\caption{Layer 5 of $N_2$}
		\end{subfigure}&
		\begin{subfigure}{0.162\linewidth}
			\centering
			\includegraphics[width=\textwidth]{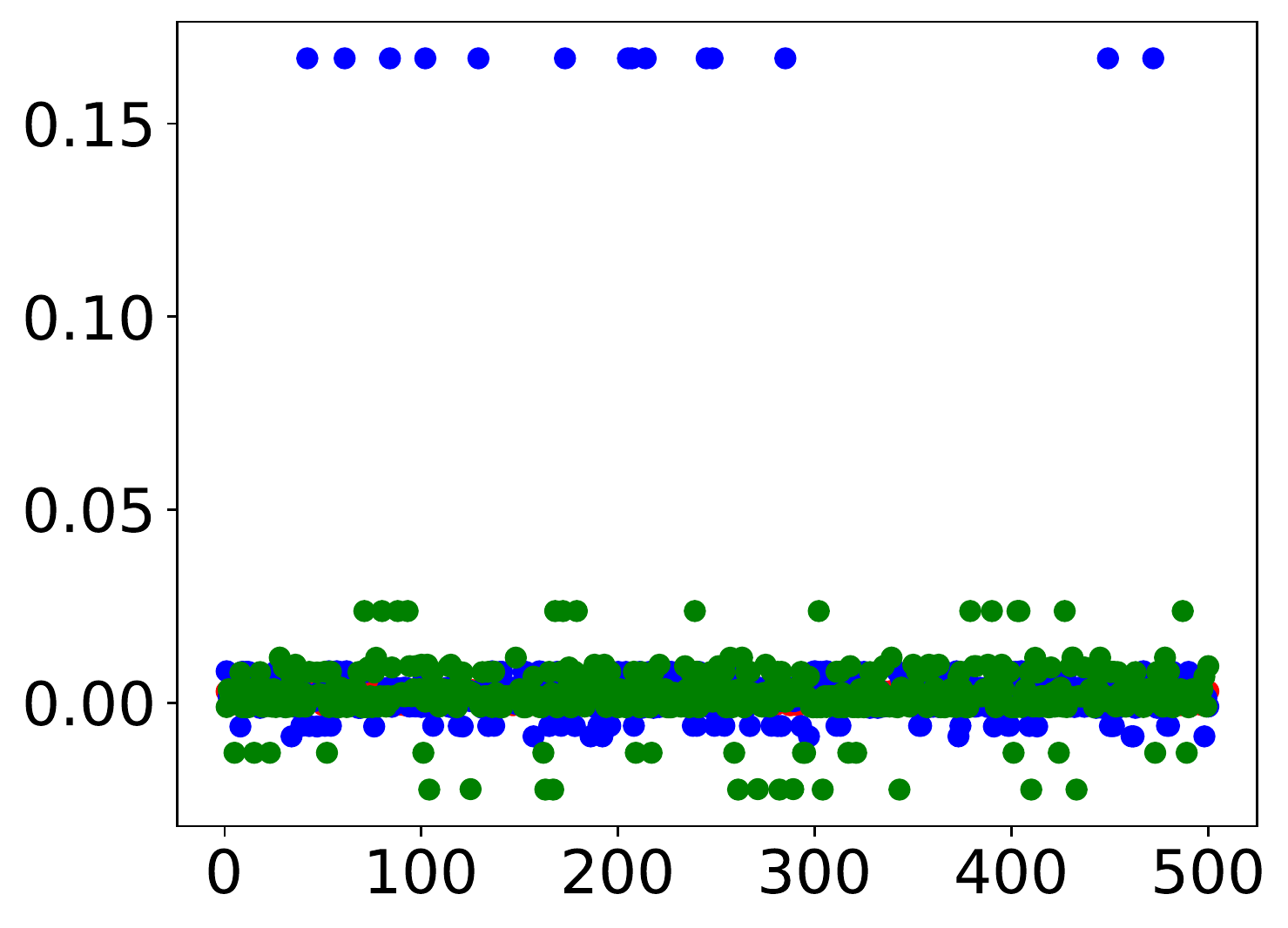}
			\caption{Layer 6 of $N_2$}
		\end{subfigure}&
		\begin{subfigure}{0.162\linewidth}
			\centering
			\includegraphics[width=\textwidth]{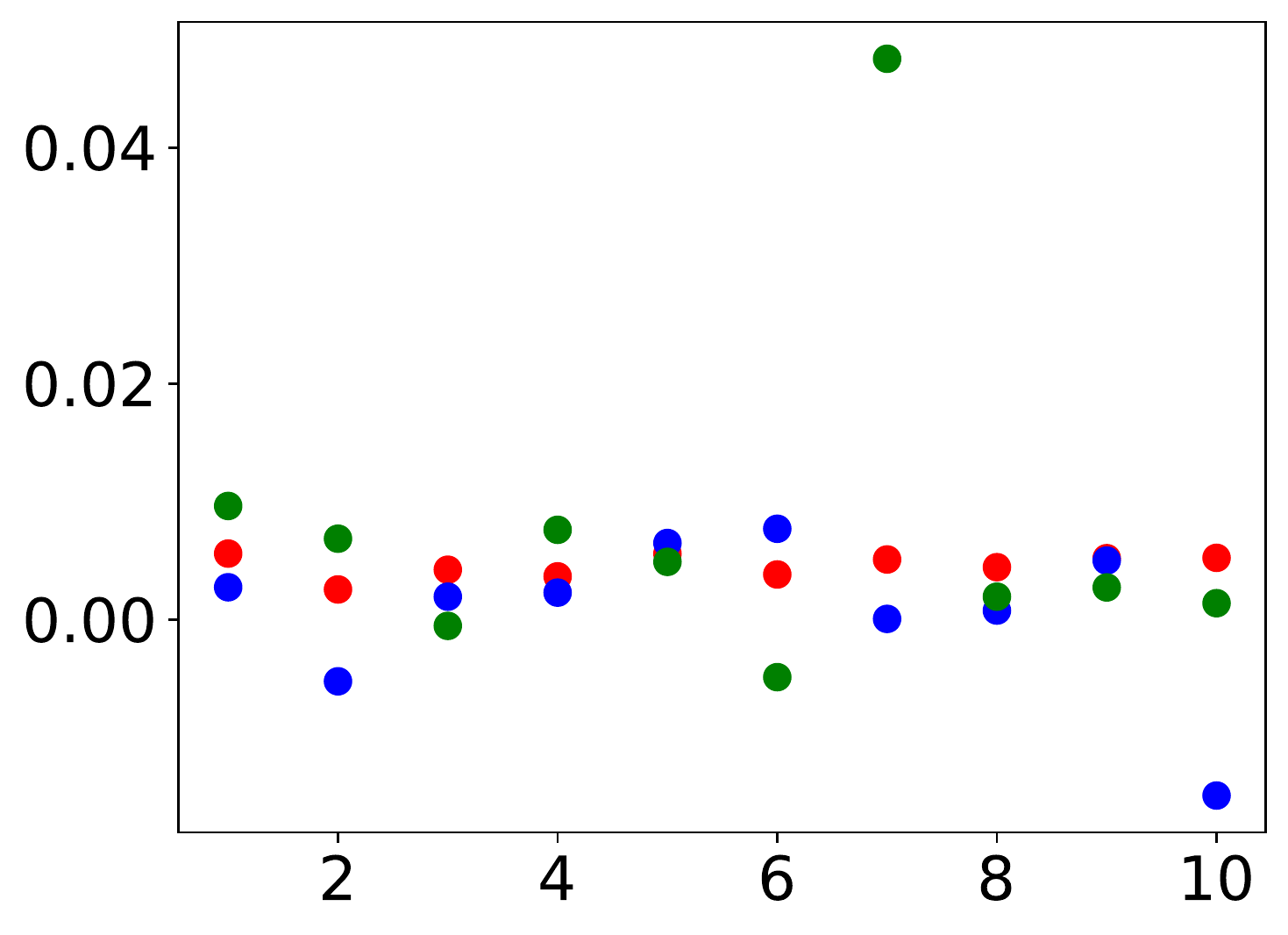}
			\caption{Output layer of $N_2$}
			\label{fig:output'}
		\end{subfigure}
	\end{tabular}
	\vspace{-2mm}
	\caption{Visualization of intermediate intervals during layer-by-layer propagation under different approximations (Red dot: $\frac{(u-l)-(u'-l')}{u'-l'}$; blue dot: $\frac{l-l'}{l'}$; green dot: $\frac{u-u'}{u'}$; $[l,u]$: interval computed by \textsc{VeriNet}, $[l',u']$: interval computed by \textsc{DeepCert}).}
	\vspace{-1mm}
	\label{fig:visual}
\end{figure*}

Table \ref{certified lower bound results on general models} shows the comparison results, where, surprisingly, none of these approaches surpass the others for all the networks. 
We also observe that \textsc{VeriNet} won the competition on 13 out of 20 networks, while \textsc{DeepCert} on the remaining ones. 
This indicates that the performances of these so-called tight approaches vary case by case. 
In this paper we do not intend to judge which approach is better experimentally but focus on seeking theoretical foundations for the tightness of approximations. 
The comparison result showed that existing tightness definitions do not rigorously guarantee that a tighter approximation can always produce a larger robust bound. 
This motivates us to seek a unified definition to characterize the  \emph{tightness} of approximations for robustness verification of neural networks. 

\vspace{-3.5mm}
\subsection{Empirical Analysis}
\label{subsec:empirical}
To validate our observation on the tightness issue and investigate its generality, we have performed empirical analysis of three state-of-the-art approximation approaches, i.e., \textsc{DeepCert} \cite{wu2021tightening}, \textsc{VeriNet} \cite{HenriksenL20}, and \textsc{RobustVerifier} \cite{lin2019robustness}, on 20 sigmoid neural networks collected from the public benchmarks. 
We evaluate the tightness of these approaches by computing the lower robust bound for each network using the tools, respectively. We randomly selected 100 images for each network, computed their lower robust bounds, and took the average value. Computing averaged lower bounds is a widely-adopted approach to reduce the affect of floating-point errors \cite{wu2021tightening,lyu2020fastened,wang2021beta}. 
Thus, even a small difference in the averaged value reflects a big difference in individual input. In general, the larger bound indicates the corresponding tool has a better performance.

We empirically analyzed the layer-by-layer propagation in the verification process of the best two tools \textsc{VeriNet} and \textsc{DeepCert} and identified the missing factor that influences verification results. We tracked the computation of the intermediate intervals of the neurons on hidden layers during their layer-by-layer propagation and compared their tightness under different approximation strategies.

Figure \ref{fig:visual} shows the layer-by-layer comparison of the intermediate intervals on the hidden neurons and output neurons. These intervals are computed during verification using the approximation approaches in \cite{HenriksenL20,wu2021tightening}, respectively. The figures from Figure \ref{fig:layer2} to \ref{fig:output} show one 5-hidden-layer network named $N_1$ on which \textsc{VeriNet} computes a larger bound than \textsc{DeepCert}, while those from Figure \ref{fig:layer2'} to \ref{fig:output'} show another 5-hidden-layer network named $N_2$ on which \textsc{DeepCert} computes a larger bound than \textsc{VeriNet}. For each neuron, we use $[l,u]$ and $[l',u']$ to represent the intervals computed by \textsc{VeriNet} and \textsc{DeepCert}, respectively. 
We introduce three dots in red, blue and green to represent $\frac{(u-l)-(u'-l')}{u'-l'}$, $\frac{l-l'}{l'}$ and $\frac{u-u'}{u'}$, respectively. The $x$-axis represents the neurons on the corresponding layer, and the $y$-axis represents the differences of the interval length, lower bounds, and upper bounds of the intervals.

Horizontally on $N_1$, most of the red dots are below 0 from layer 2 to the output layer, except layer 6. That indicates the intervals computed by \textsc{VeriNet} usually have smaller (tighter) sizes than those by \textsc{DeepCert}. 
The blue dots gradually move up above 0, indicating that 
the lower bounds of the intervals computed by \textsc{VeriNet} 
become greater than the those by \textsc{DeepCert}. Similarly, the green dots gradually move down below 0, indicating that 
the upper bounds computed by \textsc{VeriNet} become smaller. 
The trends of the three values reflect that the intermediate intervals computed by \textsc{VeriNet} are statistically tighter than those by \textsc{DeepCert} on network $N_1$. 

The trend of intermediate intervals on network $N_2$ is an opposite of the one on $N_1$. The red dots move up above 0 layer by layer, indicating that the interval sizes computed by \textsc{VeriNet} become larger than those by \textsc{DeepCert}. 
The blue dots gradually move down below 0 and the green ones move up above 0. 

The analysis result from Figure \ref{fig:visual} reveals 
that the intermediate intervals are statistically tighter than other tools when a tool produces larger certified lower bounds.

\subsection{The Network-Wise Tightest Approximation}
\label{sec:network-wise}
Existing characterizations of tightness are individually heuristic under a presumption that a tighter approximation gives rise to a more precise verification result. Unfortunately, the above examples show that this presumption does not always hold. It means that defining the tightness on each individual neuron is neither sufficient nor necessary for achieving tight approximations. That is because the tightness on neurons cannot guarantee the output intervals of neural networks are always precise. However, the output intervals are the basis for judging whether a network is robust or not. 

To characterize the tightness of  approximations to the activation functions in a neural network, we introduce the notion of \emph{network-wise tightness}. We ensure that, by a network-wise tighter approximation of the activation functions, the approximated neural network must produce more precise output intervals and consequently more precise verification results. 

\begin{definition}[Network-wise tightness]\label{global_opt}
	Given a neural network $f:{\mathbb R}^n\rightarrow {\mathbb R}^m$ and $x \in \mathbb{B}_p(x_0,\epsilon)$, let $(f_L,f_U)$ be a linear approximation of $f$ with $f_U$ and $f_L$ the upper and lower bounds, respectively.  $(f_L,f_U)$ is  network-wise tightest if, for any different linear approximation $(\hat{f}_L,\hat{f}_U)$,
	\begin{align*}
		\forall s \in S, \quad &\min\limits_{x\in \mathbb{B}_p(x_0,\epsilon)} f_{L,s}(x)\ge \min\limits_{x\in \mathbb{B}_p(x_0,\epsilon)} \hat{f}_{L, s}(x),\\
		&\max\limits_{x\in \mathbb{B}_p(x_0,\epsilon)} f_{U,s}(x)\le \max\limits_{x\in \mathbb{B}_p(x_0,\epsilon)} \hat{f}_{U,s}(x),
	\end{align*}
	where $f_{L,s}(x),f_{U,s}(x)$ denote  $s$-th item of $f_L(x),f_U(x)$, respectively.
\end{definition}
\noindent Intuitively, $(f_L,f_U)$ is tighter than $(\hat{f_L},\hat{f_U})$ in that for all output neurons $s$, $f_L$ (\textit{resp.} $f_U$) always computes a lower (\textit{resp.} an upper) bound that is greater (\textit{resp.} less) than the one $\hat{f_L}$ (\textit{resp.} $\hat{f_U}$) does. Note that Definition \ref{global_opt} is universal in that it is applicable to (i) all activation functions and (ii) all $\ell_p$ norms.

\begin{example}
	By Definition \ref{global_opt}, the approximations 
	$x_{U,3}, x_{L,3}, x_{U,4},  x_{L,4}$ to the activation functions on $x_3$ and $x_4$ in Example \ref{exp:verify} are tighter than  $x'_{U,3}, x'_{L,3}, x'_{U,4},  x'_{L,4}$. This is consistent to the verification results. Using the former approximations, we can compute tighter output ranges for both $x_5$ and $x_6$  than those by the latter. However,  the former approximation can be proved to be less tight than the latter if we take the tightness definition with respect to the minimal area defined in  \cite{HenriksenL20}. 
\end{example}

Next, we give an important property of the network-wise tightness. That is, 
a network-wise tighter approximation always leads to more precise verification results. 

\begin{theorem}
	The approximation $(f_L,f_U)$ of a neural network always produces more precise robustness verification result than 
	$(\hat{f_L},\hat{f_U})$ if $(f_L,f_U)$ is tighter than $(\hat{f_L},\hat{f_U})$ by Definition \ref{global_opt}. 
\end{theorem}

\begin{proof}[Proof sketch]
	Let $s_0 = {\mathcal L}(f(x_0))$ with fixed $\epsilon$. We check for all $s$ other than $s_0$ whether the following condition: \begin{equation}
		\min\limits_{x\in \mathbb{B}_p(x_0,\epsilon)} f_{L, s_0}(x) > \max\limits_{x\in \mathbb{B}_p(x_0,\epsilon)} f_{U, s}(x)
		\label{robust_condition}
	\end{equation}
	holds. By Definition \ref{global_opt}, if $(\hat{f_L},\hat{f_U})$ satisfies (\ref{robust_condition}), then we have:
	\begin{align*}
		&\min\limits_{x\in \mathbb{B}_p(x_0,\epsilon)} f_{L,s_0}(x)\ge \min\limits_{x\in \mathbb{B}_p(x_0,\epsilon)} \hat{f}_{L, s_0}(x)>\\
		& \max\limits_{x\in \mathbb{B}_p(x_0,\epsilon)} \hat{f}_{U,s}(x)\ge \max\limits_{x\in \mathbb{B}_p(x_0,\epsilon)} f_{U,s}(x), 
	\end{align*}
	for all $s$ other than $s_0$. Thus, $(f_L,f_U)$ certainly satisfies (\ref{robust_condition}) as well. Namely, the result verified by $(\hat{f_L},\hat{f_U})$ can also be deduced by $(f_L,f_U)$. On the contrary, the result verified by $(f_L,f_U)$ may not be verified by $(\hat{f_L},\hat{f_U})$. Consequently, $(f_L,f_U)$ always produce more precise verification result than $(\hat{f_L},\hat{f_U})$.
\end{proof}

\vspace{-2mm}
Next, we show that computing the  network-wise  tightest approximation is essentially an optimization problem. Given a $k$-layer neural network $f:\mathbb{R}^n \to \mathbb{R}^m$, we use $\phi^t$ to denote the compound function of  $f$'s layers before $t$-th activation function is applied, \textit{i.e.,}
\begin{linenomath*}
	\begin{equation*}
		\phi^t = f^t \circ \sigma^{t-1} \circ f^{t-1} \circ \ldots\circ \sigma^1 \circ f^1.
	\end{equation*}
\end{linenomath*} 
For layer $t$ with $n^t$ neurons, let $\phi^{t}_r(x)$ indicate  $r$-th item of its output (with $r \in \mathbb{Z}$ and $1 \le r\le n^t$). For each activation function $\sigma(x)$ with $x\in[l,u]$, we denote the upper (resp. lower) bound of $\sigma(x)$ by $h_{U}(x)=\alpha_{U}x+\beta_{U}$ (resp. $h_{L}(x)=\alpha_{L}x+\beta_{L}$), with variables  $\alpha_{L},\alpha_{U},\beta_{L},\beta_{U}\in \mathbb{R}$. The problem of computing the network-wise tightest approximation can then be formalized as the following optimization problems:
\begin{align} 
	&\max (\min\limits_{x \in \mathbb{B}_\infty(x_0,\epsilon)} (A_{L,s}^k x+B_{L,s}^k)), \text{and} \label{opti-max}\\
	&\min (\max\limits_{x \in \mathbb{B}_\infty(x_0,\epsilon)} (A_{U,s}^k x+B_{U,s}^k))\label{opti-min}\\
	s.t. &\quad \forall r \in \mathbb{Z},1 \le r\le n^t,
	\forall t \in \mathbb{Z}, 1 \le t < k \notag, \\
	&\left\{
	\begin{array}{ll}
		\alpha_{L,r}^t z_r^t+\beta_{L,r}^t\le \sigma(z_r^t) \le \alpha_{U,r}^t z_r^t+\beta_{U,r}^t;\\
		\min\limits_{x \in \mathbb{B}_\infty(x_0,\epsilon)} A_{L,r}^{t} x+B_{L,r}^{t}\leq z_r^t \le \max\limits_{x \in \mathbb{B}_\infty(x_0,\epsilon)} A_{U,r}^{t} x+B_{U,r}^{t}.
	\end{array}
	\right.\notag 
\end{align}
Here, $A_{L,r}^t x+B_{L,r}^t$ and $A_{U,r}^t x+B_{U,r}^t$ are the lower and upper linear bounds of $\phi_{r}^t(x)$, respectively. $A_{L,r}^t, A_{U,r}^t, B_{L,r}^t$ and $B_{U,r}^t$ are constant tensors defined on $W^t,b^t$, where $W^t,b^t$ are the weights and biases of the $t$-th layer. $\phi_{r}^t$ is the compound function of $f$'s first $t$ layers. 
$\phi_{r}^t(x)$ can be approximated by a lower linear bound $A_{L,r}^t x+B_{L,r}^t$  and an upper linear bound $A_{U,r}^t x+B_{U,r}^t$ with: 
\begin{small}
	\begin{align*}
		A_{L,r}^t=&
		\left\{
		\begin{aligned}
			&W_r^t  ,\hspace{48mm} t=1\\
			&W_{\ge 0,r}^t \alpha_{L}^{t-1} \odot A_{L}^{t-1}+ W_{< 0,r}^t \alpha_{U}^{t-1} \odot A_{L}^{t-1} , \hspace{7mm}t\ge 2 \\
		\end{aligned}
		\right.\\
		B_{L,r}^t=&
		\left\{
		\begin{aligned}
			&b_r^t, &t=1\\
			&W_{\ge 0,r}^t (\alpha_{L}^{t-1} \odot B_{L}^{t-1}+\beta_{L}^{t-1})+ \\
			&W_{< 0,r}^t (\alpha_{U}^{t-1} \odot B_{L}^{t-1}+\beta_{L}^{t-1}) + b_r^t , \hspace{11mm} &t \ge 2 \\
		\end{aligned}
		\right.\\
		A_{U,r}^t=&
		\left\{
		\begin{aligned}
			&W_r^t  , \hspace{47mm} t=1\\
			& W_{\ge 0,r}^t \alpha_{U}^{t-1} \odot A_{U}^{t-1}+ W_{< 0,r}^t \alpha_{L}^{t-1} \odot A_{U}^{t-1} ,\hspace{6mm}t \ge 2 \\
		\end{aligned}
		\right.\\
		B_{U,r}^t=&
		\left\{
		\begin{aligned}
			&b_r^t, & t=1\\
			&W_{\ge 0,r}^t (\alpha_{U}^{t-1} \odot B_{U}^{t-1}+\beta_{U}^{t-1})+ \\
			&W_{< 0,r}^t (\alpha_{L}^{t-1} \odot B_{U}^{t-1}+\beta_{U}^{t-1}) + b_r^t,\hspace{10mm}  &\hfill t \ge 2 \\
		\end{aligned}
		\right.
	\end{align*}
\end{small}
where $\odot$ denotes Hadamard production.

The solutions to all $\alpha_L, \alpha_U, \beta_L, \beta_U$ are the linear bounds to all the activation functions in the network, and their composition is the network-wise tightest approximation. Note that the solutions may not guarantee that the approximation to an individual activation function is the tightest with respect to existing tightness definitions.

\section{Approach for 1-Hidden-Layer Networks}\label{sec:1layer}

Given a neural network, we can compute the network-wise tightest approximation by instantiating and solving the optimization problems  (\ref{opti-max}) and (\ref{opti-min}). For a one-hidden-layer network, the optimization problems can be simplified as follows:
\begin{align}
	&\max (\min\limits_{x \in \mathbb{B}_\infty(x_0,\epsilon)} (A_{L,s} x+B_{L,s})), \label{small_nn_max}\\
	&\min (\max\limits_{x \in \mathbb{B}_\infty(x_0,\epsilon)} (A_{U,s} x+B_{U,s})), \label{small_nn_min}\\
	s.t. &\quad \forall r \in \mathbb{Z},1 \le r\le n, \nonumber \\
	&\left\{
	\begin{array}{ll}
		\alpha_{L,r}z_r+\beta_{L,r}\le \sigma(z_r) \le \alpha_{U,r}z_r+\beta_{U,r};\\
		\min\limits_{x \in \mathbb{B}_\infty(x_0,\epsilon)} W^{1}x+b^{1} \le z_r \le \max\limits_{x \in \mathbb{B}_\infty(x_0,\epsilon)} W^{1}x+b^{1}. \nonumber
	\end{array}
	\right.
\end{align}
where $A_{L,s}=W_{\ge0,s}^2 (\alpha_{L} \odot W^1)+W_{<0,s}^2 (\alpha_{U} \odot W^1)$, $B_{L,s}= W_{\ge0,s}^2 (\alpha_{L}\odot b^1+\beta_L)+ W_{<0,s}^2 (\alpha_{U}\odot b^1+\beta_U)+b_{s}^2$, $n$ denotes the amount of neurons in the hidden layer. The above optimization problems are the instances of the problems (\ref{opti-max}) and (\ref{opti-min}).

\DecMargin{1em}
\begin{algorithm}[t]
	\SetKwData{Left}{left}\SetKwData{This}{this}\SetKwData{Up}{up}
	\SetKwFunction{Union}{Union}\SetKwFunction{FindCompress}{FindCompress}
	\SetKwInOut{Input}{Input}\SetKwInOut{Output}{Output}
	\caption{A gradient descent-based searching algorithm for the tightest approximations of 1-hidden-layer networks.}
	\label{Algorithm1}
	\Input{~$N$: a network; $x_0$: an input to $N$; $\epsilon$: a $\ell_\infty$-norm radius}
	\Output{~$\alpha_{L,r},\beta_{L,r},\alpha_{U,r},\beta_{U,r}$ for each hidden neuron $r$}
	\For{each neuron $r$}{
		Evaluate input range $[l_r,u_r]$ for $r$;\\
		Let $\omega$ denote the line connecting $(l_r, \sigma(l_r))$ and $(u_r, \sigma(u_r))$;\\
		$R_L\leftarrow \emptyset,R_U\leftarrow \emptyset$\tcp*{Empty the sets of optimizable neurons.}
		\If{$\omega$ can be an upper bound of $\sigma$}{
			Let $\alpha_{U,r},\beta_{U,r}$ be the slope and intercept of $\omega$;\\
			Add $(r,[l_r, u_r])$ to $R_L$\tcp*{$r$'s lower bound is optimizable.} 
		}\ElseIf{$\omega$ can be a lower bound of $\sigma$}{
			Let $\alpha_{L,r},\beta_{L,r}$ be the slope and intercept of $\omega$;\\
			Add $(r,[l_r, u_r])$ to $R_U$\tcp*{$r$'s upper bound is optimizable.} 
		}
		\Else{
				Let $z_{U,r},z_{L,r}$ be the cut-off points of the tangent lines of $\sigma$ crossing $(l_r,\sigma(l_r))$ and $(u_r,\sigma(u_r))$;\\
				Add $(r,[z_{U,r}, u_r])$ to $R_U$, and $(r,[l_r, z_{L,r}])$ to $R_L$;
		}
		Randomize the cut-off points for optimizable bounds of $r$\;
		Let $\alpha_{L,r},\beta_{L,r},\alpha_{U,r},\beta_{U,r}$ be the slope and intercept of tangent line of $\sigma$ at chosen cut-off points;
	}
	\For(\tcp*[f]{$k$ is the preset optimization round}){$1,\ldots,k$}{
		Compute $A_{L,s}, B_{L,s}$ of the lower bound of output neuron $s$;\\
		Let $G:=\min\limits_{x \in \mathbb{B}_\infty(x_0,\epsilon)} (A_{L,s} x+B_{L,s})$;\\
		Update the cut-off points for $r$'s bound through $-\nabla(G)$\; 
		Update $\alpha_{L,r},\beta_{L,r},\alpha_{U,r},\beta_{U,r}$ at chosen cut-off points\;
	}
\end{algorithm}

The optimization problem is a convex variant and thus efficiently solvable by leveraging the gradient descent-based searching algorithm \cite{lyu2020fastened}. Algorithm \ref{Algorithm1} shows the pseudo code of the algorithm for calculating the optimal solution for the optimization problem with objective function (\ref{small_nn_max}). 
A solution represents a network-wise tightest lower bound to the 1-hidden-layer networks.  For each activation function on a hidden neuron, it first determines whether the line $\omega$ crossing the two endpoints can be an upper (Lines 5-7) or lower bound (Lines 8-10). 
If those are the cases, the tangent line of the activation function is chosen to be lower bound (\textit{resp}. upper bound), and its cut-off point can be an optimization variable. Otherwise (Lines 11-13), the lower and upper bounds can both be optimized. The optimizing ranges for those cases are calculated.

Let $G:=\min\limits_{x \in \mathbb{B}_\infty(x_0,\epsilon)} (A_{L,s} x+B_{L,s})$ and $ANS = \max\limits_{\alpha_L,\alpha_U,\beta_L,\beta_U}(G)$. We use gradient descent steps (Lines 16-20) to optimize the target $ANS$. We conduct gradient descent and modify the value of ${\alpha_L,\alpha_U,\beta_L,\beta_U}$ if the $ANS$ achieves a larger result under the new bounds.

The optimization problem with objective function (\ref{small_nn_min}) can be solved by the same algorithm, with $ANS$ replaced by (\ref{small_nn_min}).

\begin{example}
	Let us revisit  Example \ref{exp:verify}. With Algorithm \ref{Algorithm1}, we  compute the network-wise tightest approximations for the network in Figure \ref{fig:veriexample}. Figure \ref{app3} shows the upper (\textit{resp.} lower) bounds, denoted by $x''_{U,3},x''_{U,4}$ (\textit{resp}.  $x''_{L,3},x''_{L,4}$). 
	By Definition \ref{global_opt}, 
	$x''_{U,3},x''_{U,4}$ are tighter than the other two approximations. 
	The resulting output range of neuron $x_7$ is $[0.307, 5.693]$, which is  more precise than both  $[-0.079, 6.073]$ and $[0.177,5.817]$ in Figure \ref{fig:minimal} and \ref{fig:endpoints}, respectively. 
\end{example}

\begin{figure}
	\centering
	\includegraphics[width=0.48\textwidth]{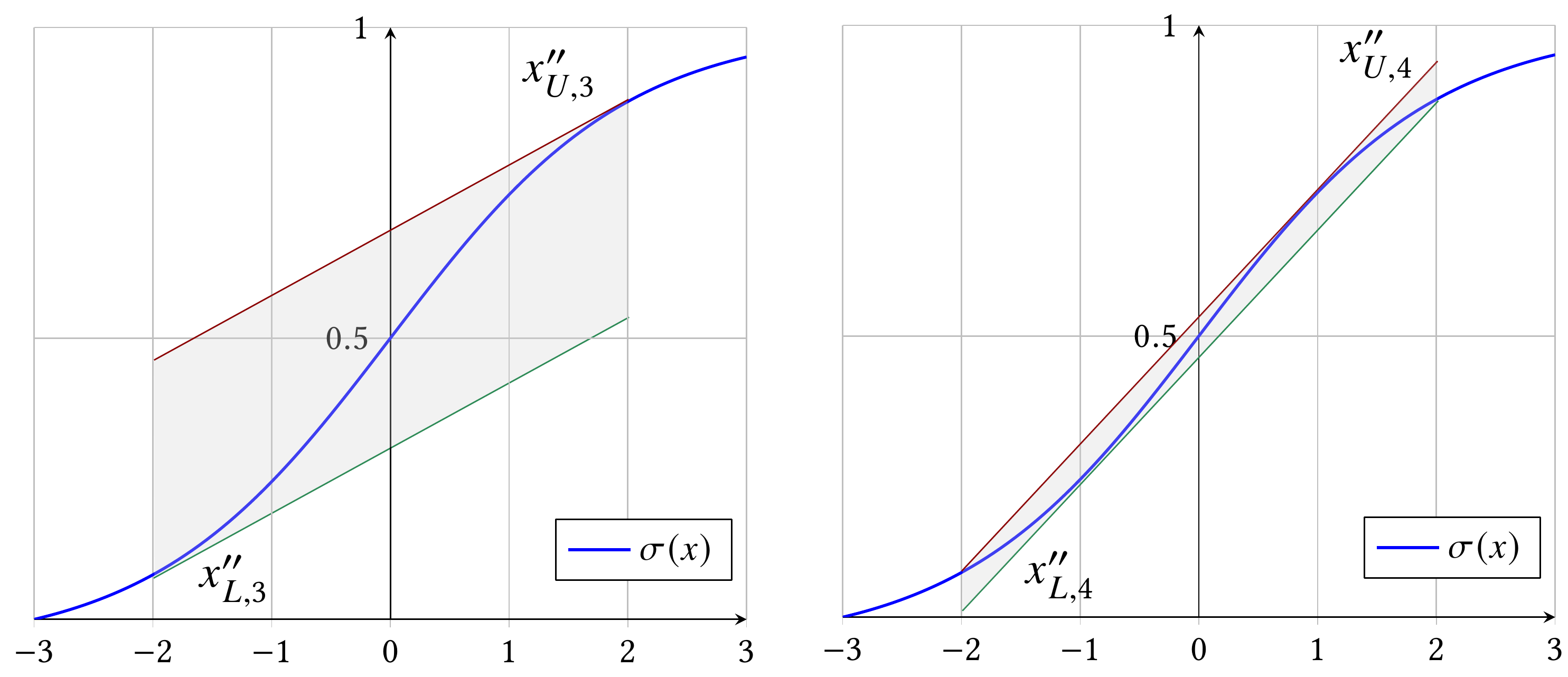}
 	\vspace{-5mm}
	\caption{The network-wise tightest approximation to the activation functions on  $x_3$ (left) and $x_4$ (right) in Example  \ref{exp:verify}.} 
	\label{app3}
	\vspace{-2mm}
\end{figure}

Note that the network-wise tightest approximations in the above example are a  hybrid of both kinds of approximations in Figure \ref{fig:minimal} and \ref{fig:endpoints}, which cannot be the tightest under a single tightness criterion in \cite{HenriksenL20,wu2021tightening}. This echoes our advocation that solely pursuing neuron-wise tightness under existing tightness definitions may not guarantee that they are the network-wise tightest, and consequently cannot achieve precise robust verification results.

\section{\hspace{-1mm}Approach for Multi-Hidden-Layer Networks}\label{sec:appr}

For the networks with two or more hidden layers, solving the optimization problems (\ref{opti-max}) and (\ref{opti-min}) becomes impractical due to its non-convexity. In \cite{lyu2020fastened}, Lyu \emph{et al.} proved that it is even non-convex to separately compute the tightest approximation for each neuron. The intractability lies in the accumulated constraints throughout the network: for any hidden layer, the input intervals of the  activation functions are constrained by the approximations to the activation functions for the previous hidden layer. Neither can the optimization problems be solved on a layer basis because the objective function are network-wise. To our knowledge, no efficient algorithms or tools exist for such optimization problems.  
In this section, we propose computable neuron-wise tightest approximations and identify the condition when all the weights in a neural network are non-negative, the neuron-wise tightest approximations lead to being network-wise tightest. 

 \vspace{-2mm}
\subsection{The Neuron-Wise Tightest Approximation}
\label{subsec:our-neuron-wise-method}
Our empirical analysis in Section \ref{sub:tightness} reveals an insight that preserving tighter intermediate intervals during layer-by-layer propagation usually produces larger certified lower robust bounds. In the same spirit, we heuristically define the tightness of an approximation to an individual activation function in terms of the overestimation caused by the approximation. Smaller overestimation 
implies a tighter approximation. Particularly, an approximation is 
the \emph{neuron-wise tightest} if it results in no overestimation of the output range of the activation function.

\begin{definition}[Neuron-wise Tightness]\label{def:neuronwise}
	Let  $\sigma(x)$ be an activation function with $x\in [l,u]$, and  $h_{U}(x), h_{L}(x)$ be its upper and lower bounds, with $\alpha_U,\alpha_L$ their slopes. $h_{U}(x)$ (resp. $h_{L}(x)$) is the neuron-wise tightest if $h_{U}(u)=\sigma(u)$ (resp.  $h_{L}(l)=\sigma(l)$) and $\int^u_l h_U(x)-\sigma(x)dx$ (resp.  $\int^u_l \sigma(x)-h_L(x)dx$) is minimal. 
\end{definition}

By Definition \ref{def:neuronwise}, we identify three cases of defining the neuron-wise tightest approximation for each individual activation function. The three cases are defined 
according to the relation between the slopes of activation function at two endpoints and the slope of the line crossing the two endpoints, as classified in \cite{wu2021tightening,HenriksenL20}. 
Given an input interval $[l,u]$ for $\sigma(x)$, the slopes of $\sigma(x)$ at $(l,\sigma(l))$ and $(u,\sigma(u))$ are represented by $\sigma'(l)$ and $\sigma'(u)$, respectively; the slope of the line crossing $(l,\sigma(l))$ and $(u,\sigma(u))$ is $k=\frac{\sigma(u)-\sigma(l)}{u-l}$.  
Figure \ref{optimal_approximation} depicts the neuron-wise tightest approximations in the following three different cases:
\paragraph{\textbf{Case 1.}} When $\sigma'(l)<k<\sigma'(u)$ (Figure \ref{optimal_approximation}a), the line  that connects the two endpoints is chosen as the upper bound, while the tangent line of $\sigma(x)$ at $(l,\sigma(l))$ as the lower bound. We then have $h_{U}(x)=k(x-l)+\sigma(l)$ and $h_{L}(x)=\sigma'(l)(x-l)+\sigma(l)$.

\paragraph{\textbf{Case 2.}} When $\sigma'(u)<k<\sigma'(l)$ (Figure \ref{optimal_approximation}b),  the tangent line of $\sigma(x)$ at  $(u,\sigma(u))$ and the line crossing two endpoints are considered as the upper and lower bounds, respectively. We then have $h_{U}(x)=\sigma'(u)(x-u)+\sigma(u)$ and $h_{L}(x)=k(x-u)+\sigma(u)$.

\paragraph{\textbf{Case 3.}} When $\sigma'(l)<k$ and $\sigma'(u)<k$ (Figure \ref{optimal_approximation}c), the tangent line of $\sigma(x)$ at $(u,\sigma(u))$  is taken as the upper bound, while the tangent line of $\sigma(x)$ at $(l,\sigma(l))$  as the lower bound. We then have $h_{U}(x)=\sigma'(u)(x-u)+\sigma(u)$ and $h_{L}(x)=\sigma'(l)(x-l)+\sigma(l)$.

\begin{figure}[t!]
	\centering
	\includegraphics[width=0.48\textwidth]{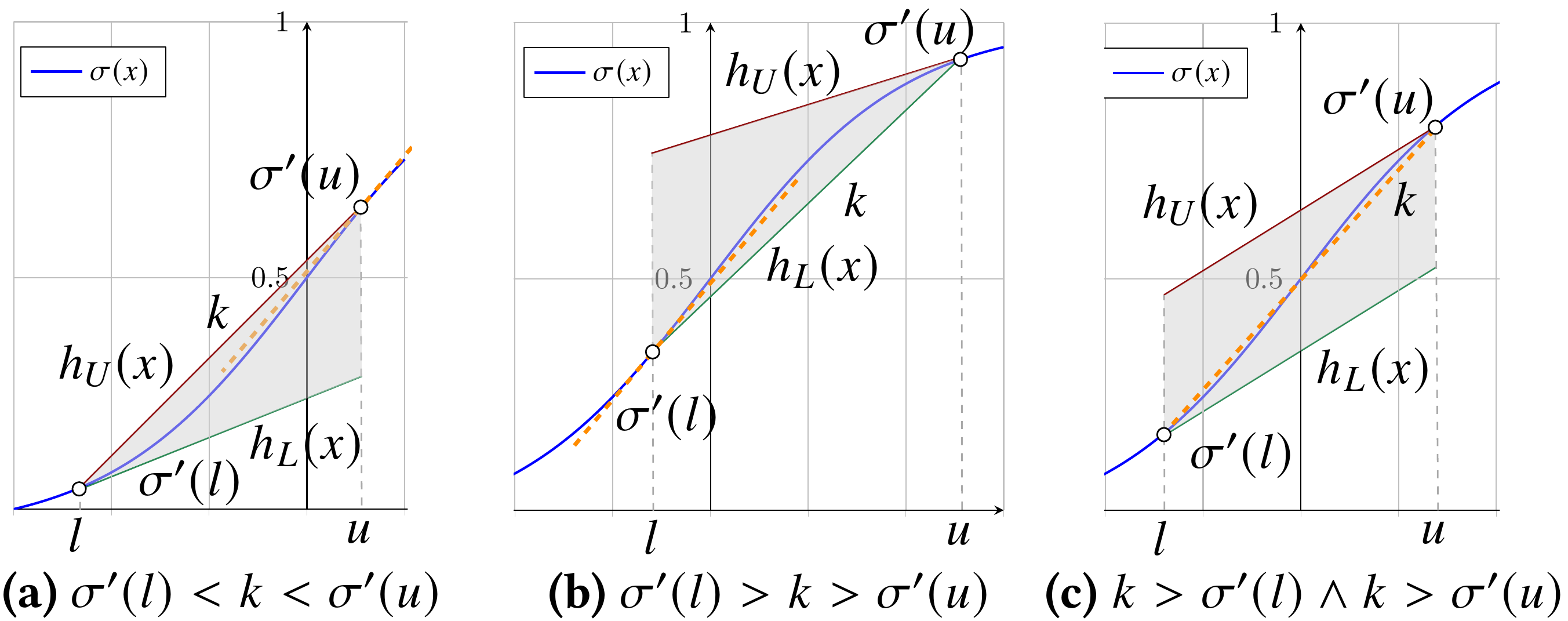}
 	\vspace{-6mm}
	\caption{The neuron-wise tightest linear approximation.}
	\label{optimal_approximation}
 	\vspace{-4mm}
\end{figure}

Note that, Definition \ref{def:neuronwise} also considers the tightness characterizations in \cite{HenriksenL20,lyu2020fastened}. It is easy to prove that 
any other linear bound crossing the endpoints is less tight than the one defined in the above three cases according to the tightness definitions in \cite{HenriksenL20,lyu2020fastened}.  

\vspace{-1mm}
\subsection{Neuron-Wise vs. Network-Wise}
\label{subsec:conditions}
In this section, we study the relation between neuron-wise tightness and network-wise tightness. 
Although the neuron-wise tightest approximation does not overestimate the output range of a single neuron, it cannot guarantee that the composition for all the neurons is the network-wise tightest because the monotonicity of a neuron cannot be preserved by the next layer. 
The monotonicity may be altered by the weights between layers because the input function of each neuron in any hidden layer is compounded by the output functions in the previous layer multiplied by the weights. Hence, a sufficient condition of passing neuron-wise tightness to the network is to avoid breaking monotonicity during the propagation.

\begin{table*}[ht]
	\footnotesize 
		\setlength{\tabcolsep}{1pt}
	\caption{Performance comparison  on non-negative Sigmoid networks between \textsc{NeWise} (NW) and existing tools, \textsc{DeepCert} (DC), \textsc{VeriNet} (VN), and \textsc{RobustVerifier} (RV).
		$t\times n$ refers to an FNN with $t$ layers and $n$ neurons per layer. $\mbox{CNN}_{t-c}$ denotes a CNN with $t$ layers and $c$ filters of size 3$\times$3. 
	}
 	\vspace{-2mm}
	\centering
	\resizebox{\textwidth}{!}{
		\begin{tabular}{|c|c|r|r|r|r|r|r|r|r|r|r|r|r|r|r|r|r|}
			\hline 
			\multirow{3}{*}{\textbf{Dataset}} & \multirow{3}{*}{\textbf{Model}} & \multirow{3}{*}{\textbf{\#Neur.}} &   \multicolumn{14}{c|}{\textbf{Certified Lower Bound}} & \multirow{3}{*}{\textbf{Time (s)}}  \\
			\hhline{~~~--------------~}
			& & & \multicolumn{7}{c|}{\textbf{Average}} & \multicolumn{7}{c|}{\textbf{ Standard Deviation}} & \\
			\hhline{~~~--------------~}
			& & & \textsc{NW} & \textsc{DC} &  Impr. (\%) & \textsc{VN} &  Impr. (\%) & \textsc{RV} &  Impr. (\%) & \textsc{NW} & \textsc{DC} &  Impr. (\%) & \textsc{VN} &  Impr. (\%) & \textsc{RV} &  Impr. (\%) &  \\
			\hline
			\hline
			\multirow{9}{*}{MNIST} 
			& 5x100              & 510    & 0.0091 & 0.0071 & \cellcolor{atomictangerine} 28.15  $\uparrow$ & 0.0071 & \cellcolor{atomictangerine} 27.25 $\uparrow$ & 0.0064 & \cellcolor{atomictangerine} 40.90 $\uparrow$ & 0.0057 & 0.0042 & \cellcolor{atomictangerine} 37.11 $\uparrow$ & 0.0042 & \cellcolor{atomictangerine} 35.48 $\uparrow$ & 0.0034 & \cellcolor{atomictangerine} 69.35 $\uparrow$ &   4.30 $\pm$0.02    \\ 
			& 3x700              & 2,110  & 0.0037 & 0.0030 & \cellcolor{atomictangerine} 24.92  $\uparrow$ & 0.0030 & \cellcolor{atomictangerine} 22.85 $\uparrow$ & 0.0029 & \cellcolor{atomictangerine} 27.05 $\uparrow$ & 0.0018 & 0.0013 & \cellcolor{atomictangerine} 41.86 $\uparrow$ & 0.0014 & \cellcolor{atomictangerine} 34.56 $\uparrow$ & 0.0013 & \cellcolor{atomictangerine} 41.86 $\uparrow$ & 117.94 $\pm$0.31  \\ 
			& $\mbox{CNN}_{6-5}$& 12,300  & 0.0968 & 0.0788 & \cellcolor{atomictangerine} 22.82  $\uparrow$ & 0.0778 & \cellcolor{atomictangerine} 24.37 $\uparrow$ & 0.0699 & \cellcolor{atomictangerine} 38.50 $\uparrow$ & 0.0372 & 0.0280 & \cellcolor{atomictangerine} 32.92 $\uparrow$ & 0.0276 & \cellcolor{atomictangerine} 35.09 $\uparrow$ & 0.0212 & \cellcolor{atomictangerine} 75.86 $\uparrow$ &   5.70 $\pm$0.42  \\ 
			& 3x50               & 160    & 0.0105 & 0.0088 & \cellcolor{apricot}         19.23  $\uparrow$ & 0.0088 & \cellcolor{apricot}         19.50 $\uparrow$ & 0.0080 & \cellcolor{atomictangerine} 31.42 $\uparrow$ & 0.0051 & 0.0038 & \cellcolor{atomictangerine} 32.72 $\uparrow$ & 0.0038 & \cellcolor{atomictangerine} 32.72 $\uparrow$ & 0.0029 & \cellcolor{atomictangerine} 71.86 $\uparrow$ &   0.14 $\pm$0.00  \\ 
			& 3x100              & 310    & 0.0139 & 0.0120 & \cellcolor{apricot}         15.46  $\uparrow$ & 0.0120 & \cellcolor{apricot}         15.56 $\uparrow$ & 0.0111 & \cellcolor{atomictangerine} 25.47 $\uparrow$ & 0.0071 & 0.0057 & \cellcolor{apricot}         24.82 $\uparrow$ & 0.0057 & \cellcolor{apricot}         23.30 $\uparrow$ & 0.0046 & \cellcolor{atomictangerine} 53.46 $\uparrow$ &   2.22 $\pm$0.02  \\ 
			& $\mbox{CNN}_{5-5}$& 10,680  & 0.0801 & 0.0708 & \cellcolor{apricot}         13.14  $\uparrow$ & 0.0704 & \cellcolor{apricot}         13.75 $\uparrow$ & 0.0683 & \cellcolor{apricot}         17.30 $\uparrow$ & 0.0238 & 0.0200 & \cellcolor{apricot}         18.87 $\uparrow$ & 0.0198 & \cellcolor{apricot}         20.50 $\uparrow$ & 0.0180 & \cellcolor{apricot}         32.20 $\uparrow$ &   2.88 $\pm$0.32  \\ 
			& $\mbox{CNN}_{3-2}$& 2,514   & 0.0521 & 0.0483 & \cellcolor{antiquewhite}     7.82  $\uparrow$ & 0.0483 & \cellcolor{antiquewhite}     7.94 $\uparrow$ & 0.0478 & \cellcolor{antiquewhite}     8.88 $\uparrow$ & 0.0180 & 0.0161 & \cellcolor{apricot}         12.13 $\uparrow$ & 0.0160 & \cellcolor{antiquewhite}    12.41 $\uparrow$ & 0.0156 & \cellcolor{apricot}         15.44 $\uparrow$ &   0.17 $\pm$0.04  \\ 
			& $\mbox{CNN}_{4-5}$& 8,680   & 0.0505 & 0.0473 & \cellcolor{antiquewhite}     6.68  $\uparrow$ & 0.0471 & \cellcolor{antiquewhite}     7.26 $\uparrow$ & 0.0464 & \cellcolor{antiquewhite}     8.81 $\uparrow$ & 0.0207 & 0.0186 & \cellcolor{apricot}         11.26 $\uparrow$ & 0.0183 & \cellcolor{antiquewhite}    12.84 $\uparrow$ & 0.0175 & \cellcolor{apricot}         17.87 $\uparrow$ &   1.17 $\pm$0.20  \\ 
			& $\mbox{CNN}_{3-4}$& 5,018   & 0.0448 & 0.0422 & \cellcolor{antiquewhite}     6.09  $\uparrow$ & 0.0421 & \cellcolor{antiquewhite}     6.24 $\uparrow$ & 0.0418 & \cellcolor{antiquewhite}     6.98 $\uparrow$ & 0.0156 & 0.0142 & \cellcolor{antiquewhite}     9.71 $\uparrow$ & 0.0141 & \cellcolor{antiquewhite}    10.18 $\uparrow$ & 0.0138 & \cellcolor{antiquewhite}    12.56 $\uparrow$ &   0.30 $\pm$0.08  \\ 
			\hline
			\multirow{6}{*}{\makecell{Fashion \\ MNIST} }
			& 4x100              & 410    & 0.0312 & 0.0188 & \cellcolor{atomictangerine} 65.48  $\uparrow$ & 0.0194 & \cellcolor{atomictangerine} 60.62 $\uparrow$ & 0.0159 & \cellcolor{atomictangerine} 96.22 $\uparrow$ & 0.0403 & 0.0210 & \cellcolor{atomictangerine} 92.28 $\uparrow$ & 0.0220 & \cellcolor{atomictangerine} 83.20 $\uparrow$ & 0.0176 & \cellcolor{atomictangerine} 129.47 $\uparrow$ & 3.31 $\pm$0.04   \\ 
			& 3x100              & 310    & 0.0326 & 0.0263 & \cellcolor{atomictangerine}         24.02  $\uparrow$ & 0.0270 & \cellcolor{atomictangerine} 21.03 $\uparrow$ & 0.0238 & \cellcolor{atomictangerine} 36.81 $\uparrow$ & 0.0335 & 0.0262 & \cellcolor{atomictangerine} 27.67 $\uparrow$ & 0.0282 & \cellcolor{apricot}         18.92 $\uparrow$ & 0.0234 & \cellcolor{atomictangerine}  43.22 $\uparrow$ & 2.22 $\pm$0.01 \\ 
			& $\mbox{CNN}_{5-5}$& 10,680  & 0.1303 & 0.1155 & \cellcolor{apricot}         12.81  $\uparrow$ & 0.1151 & \cellcolor{apricot}         13.22 $\uparrow$ & 0.1088 & \cellcolor{apricot}         19.72 $\uparrow$ & 0.0830 & 0.0714 & \cellcolor{apricot}         16.23 $\uparrow$ & 0.0721 & \cellcolor{apricot}         15.10 $\uparrow$ & 0.0636 & \cellcolor{apricot}          30.51 $\uparrow$ & 2.89 $\pm$0.33 \\ 
			& $\mbox{CNN}_{3-2}$& 2,514   & 0.0790 & 0.0713 & \cellcolor{apricot}         10.79  $\uparrow$ & 0.0713 & \cellcolor{apricot}         10.74 $\uparrow$ & 0.0695 & \cellcolor{apricot}         13.68 $\uparrow$ & 0.0497 & 0.0416 & \cellcolor{apricot}         19.55 $\uparrow$ & 0.0418 & \cellcolor{apricot}         19.06 $\uparrow$ & 0.0386 & \cellcolor{apricot}          28.98 $\uparrow$ & 0.17 $\pm$0.04 \\ 
			& $\mbox{CNN}_{4-5}$& 8,680   & 0.0959 & 0.0868 & \cellcolor{apricot}         10.40  $\uparrow$ & 0.0864 & \cellcolor{apricot}         10.90 $\uparrow$ & 0.0839 & \cellcolor{apricot}         14.19 $\uparrow$ & 0.0561 & 0.0486 & \cellcolor{apricot}         15.51 $\uparrow$ & 0.0482 & \cellcolor{apricot}         16.52 $\uparrow$ & 0.0453 & \cellcolor{apricot}          24.03 $\uparrow$ & 1.18 $\pm$0.21 \\ 
			& $\mbox{CNN}_{3-4}$& 5,018   & 0.0747 & 0.0694 & \cellcolor{antiquewhite}     7.52  $\uparrow$ & 0.0693 & \cellcolor{antiquewhite}     7.72 $\uparrow$ & 0.0681 & \cellcolor{antiquewhite}     9.70 $\uparrow$ & 0.0465 & 0.0410 & \cellcolor{antiquewhite}    13.32 $\uparrow$ & 0.0409 & \cellcolor{antiquewhite}    13.59 $\uparrow$ & 0.0391 & \cellcolor{antiquewhite}     18.85 $\uparrow$ & 0.30 $\pm$0.09 \\ 
			\hline
			\multirow{7}{*}{CIFAR10} 
			& 9x100 			  & 910   & 0.0315 & 0.0211 & \cellcolor{atomictangerine} 49.03  $\uparrow$ & 0.0214 & \cellcolor{atomictangerine} 46.94 $\uparrow$ & 0.0192 & \cellcolor{atomictangerine} 63.58 $\uparrow$ & 0.0280 & 0.0183 & \cellcolor{atomictangerine} 52.70 $\uparrow$ & 0.0186 & \cellcolor{atomictangerine} 50.32 $\uparrow$ & 0.0133 & \cellcolor{atomictangerine} 110.07 $\uparrow$ & 4.92 $\pm$0.01   \\ 
			& 6x100 			  & 610   & 0.0221 & 0.0174 & \cellcolor{atomictangerine} 27.08  $\uparrow$ & 0.0176 & \cellcolor{atomictangerine} 26.14 $\uparrow$ & 0.0170 & \cellcolor{atomictangerine} 30.22 $\uparrow$ & 0.0165 & 0.0118 & \cellcolor{atomictangerine} 40.05 $\uparrow$ & 0.0120 & \cellcolor{atomictangerine} 37.82 $\uparrow$ & 0.0111 & \cellcolor{atomictangerine}  48.24 $\uparrow$ & 3.04 $\pm$0.02   \\ 
			& 5x100 			  & 510   & 0.0200 & 0.0167 & \cellcolor{apricot}         19.76  $\uparrow$ & 0.0167 & \cellcolor{apricot}         19.47 $\uparrow$ & 0.0163 & \cellcolor{atomictangerine} 22.40 $\uparrow$ & 0.0137 & 0.0104 & \cellcolor{apricot}         31.80 $\uparrow$ & 0.0104 & \cellcolor{apricot}         31.42 $\uparrow$ & 0.0099 & \cellcolor{apricot}          38.45 $\uparrow$ & 2.44 $\pm$0.01   \\ 
			& 3x50  			  & 160   & 0.0206 & 0.0178 & \cellcolor{apricot}         15.43  $\uparrow$ & 0.0179 & \cellcolor{apricot}         14.66 $\uparrow$ & 0.0176 & \cellcolor{apricot}         16.88 $\uparrow$ & 0.0144 & 0.0113 & \cellcolor{apricot}         27.57 $\uparrow$ & 0.0115 & \cellcolor{apricot}         25.24 $\uparrow$ & 0.0110 & \cellcolor{apricot}          31.30 $\uparrow$ & 0.43 $\pm$0.00   \\ 
			& 4x100 			  & 410   & 0.0161 & 0.0140 & \cellcolor{apricot}         15.23  $\uparrow$ & 0.0140 & \cellcolor{apricot}         14.81 $\uparrow$ & 0.0138 & \cellcolor{apricot}         16.56 $\uparrow$ & 0.0111 & 0.0089 & \cellcolor{apricot}         24.61 $\uparrow$ & 0.0090 & \cellcolor{apricot}         23.63 $\uparrow$ & 0.0087 & \cellcolor{apricot}          27.62 $\uparrow$ & 1.85 $\pm$0.01   \\ 
			& $\mbox{CNN}_{3-4}$ & 6,746  & 0.0187 & 0.0181 & \cellcolor{antiquewhite}     3.38  $\uparrow$ & 0.0181 & \cellcolor{antiquewhite}     3.32 $\uparrow$ & 0.0181 & \cellcolor{antiquewhite}     3.43 $\uparrow$ & 0.0109 & 0.0103 & \cellcolor{antiquewhite}     5.93 $\uparrow$ & 0.0103 & \cellcolor{antiquewhite}     5.83 $\uparrow$ & 0.0103 & \cellcolor{antiquewhite}      6.13 $\uparrow$ & 0.56 $\pm$0.08  \\ 
			& $\mbox{CNN}_{3-2}$ & 3,378  & 0.0185 & 0.0180 & \cellcolor{antiquewhite}     2.49  $\uparrow$ & 0.0180 & \cellcolor{antiquewhite}     2.55 $\uparrow$ & 0.0180 & \cellcolor{antiquewhite}     2.67 $\uparrow$ & 0.0125 & 0.0120 & \cellcolor{antiquewhite}     4.34 $\uparrow$ & 0.0120 & \cellcolor{antiquewhite}     4.34 $\uparrow$ & 0.0120 & \cellcolor{antiquewhite}      4.60 $\uparrow$ & 0.30 $\pm$0.06  \\ 
			\hline
	\end{tabular}}
	\label{certified lower bound on Sigmoid models}
 	\vspace{-2mm}
\end{table*}

\begin{definition}[Network-wise monotonous]\label{stren_condi}
	Given a $k$-layer neural network $f:\mathbb{R}^n \to \mathbb{R}^m$ and its input $x=[x_1,...,x_n]$, f is called network-wise monotonous if the following three conditions hold:
	\begin{enumerate}
		\item $\forall t_1, t_2\in \mathbb{Z}, 1 \le t_1 \le k \wedge 1 \le t_2 \le k$, 
		\item $\forall r_1, r_2\in \mathbb{Z}, 1 \le r_1 \le n^{t_1} \wedge  1 \le r_2 \le n^{t_2}$,
		\item $\forall i \in \mathbb{Z}, 1 \le i \le n, \phi_{r_1}^{t_1}(x_i),\phi_{r_2}^{t_2}(x_i)$ are both either monotonically increasing or decreasing.
	\end{enumerate}

\end{definition}
\noindent Intuitively, a monotonous network requires all the neurons to share the same monotonicity w.r.t. the input so that they can achieve the  maximum or minimum on the same input. 
\begin{lemma}\label{2con->mono}
	A neural network is network-wise monotonous if the network satisfies the following two conditions:
	\begin{enumerate}
		\item For the first layer, for any selected $i\in \mathbb{Z}, 1\le i \le n$, items in the $i$-th column of $W^1$ are all positive or all negative;
		\item Every item in weights from the second layer to the last layer is non-negative.
	\end{enumerate}
\end{lemma}

Next, we formulate the most important property of our neuron-wise  approximation approach as the following theorem, stating that the composition of all neuron-wise tightest approximations is the network-wise tightest if the network is monotonous. 

\begin{theorem}\label{neuron->network}
	The composition of the neuron-wise tightest approximations is the network-wise  tightest if the network satisfies the following two conditions:
	\begin{enumerate}
		\item For the first layer, the items in each column of the weight matrix are all positive or negative;
		\item Every item in weights between remained layers is non-negative.
	\end{enumerate}
\end{theorem}

Theorem \ref{neuron->network} holds as  both conditions guarantee the monotonicity of  the network. 
If a neural network is monotonous, then the composition of the neuron-wise tightest approximations is a network-wise tightest  approximation with respect to the robustness verification. 
In particular, starting from the second layer, the neuron-wise tightness  is preserved with only non-negative weights during the layer-wise propagation (the second condition). See Appendix B for the complete proof.

\begin{example}
	Assume that we replace the three negative weights of the neural network in Figure \ref{fig:veriexample} with $1$, $5$, and $1$, respectively. The tightest approximations to  $x_3$ and $x_4$, returned by Algorithm \ref{Algorithm1}, are exactly the same as those returned by the neuron-wise tightest approximations in Section \ref{subsec:our-neuron-wise-method} (i.e., $[1.430, 10.570]$). 
\end{example}
 \vspace{-1mm}

	\section{ Experiments}\label{sec:exp}
We evaluate our  approximation method concerning both precision and efficiency in the robustness verification of Sigmoid-like neural networks. Our goal is threefold:
\begin{enumerate}
	\item  To validate our mathematical proof of Theorem \ref{neuron->network} 
	via extensive experimental results (i.e., always returning the largest certified lower bounds for non-negative networks);
	\item To demonstrate that Algorithm \ref{Algorithm1} can always compute tighter lower bounds for 1-hidden-layer networks; 
	\item To explore our approach's effectiveness under general neural networks
	with \emph{mixed}  weights.
	
\end{enumerate}

 \vspace{-2mm}
\subsection{Benchmarks and Experimental Setup}
\noindent \textbf{Competitors.} 
We consider three representative approximations in the literature:  \textsc{DeepCert} \cite{wu2021tightening}, \textsc{VeriNet} \cite{HenriksenL20}, and \textsc{RobustVerifier} \cite{lin2019robustness}. For a fair comparison, we implemented in Python
all the competing approaches including our new algorithm called 
\textsc{NeWise}.

 \vspace{1ex}
\noindent \textbf{Datasets and Networks.}
We have conducted three sets of experiments on fully connected (FNNs) and convolutional (CNNs) networks: We focus on CNNs due to their effectiveness
in a wide range of visual recognition applications~\cite{DBLP:journals/cacm/KrizhevskySH17,DBLP:conf/aaai/PanWDY17,DBLP:conf/cvpr/LongSD15, DBLP:journals/pami/RenHG017, DBLP:conf/cvpr/ToshevS14}; we also consider FNNs to 
expand the architecture variety. We trained all the networks 
on the image databases
MNIST \cite{lecun1998gradient}, Fashion MNIST \cite{xiao2017/online}, and CIFAR10~\cite{krizhevsky2009learning}.
We chose the first 100 images from the test set of each dataset as in \cite{boopathy2019cnn,zhang2018efficient,wu2021tightening}, among which only correctly-classified images by the neural network are considered in our experiments.

For each network architecture, we trained three variant neural networks using the Sigmoid, Tanh, and Arctan activation functions, respectively. 
In Experiment \textbf{I} the networks contain only non-negative weights. We used Adam or SGD optimizer with at least 50 epochs of batch size 128. The test set accuracy of networks trained on MNIST, Fashion MNIST, and CIFAR10 is around 0.9, 0.85, and 0.4, respectively.
In Experiment \textbf{II} and \textbf{III}  we trained 1-hidden-layer networks and used pre-trained  models  \cite{singh2019abstract,wu2021tightening, tjx_models} (as in Table~\ref{certified lower bound results on general models}, Section \ref{subsec:empirical})  with no constraint on the weights. Note that the number of neurons in FNNs can be considerably fewer than that in CNNs~\cite{DBLP:journals/pieee/LeCunBBH98}, while the networks can still achieve up to 0.99 test accuracy.

 \vspace{1ex}
\noindent \textbf{Metrics.}
We use certified lower bound to assess  \emph{effectiveness}, and $(\epsilon'-\epsilon)/\epsilon$ to quantify the precision improvement, where $\epsilon'$ and $\epsilon$ denote the lower bounds certified by  \textsc{NeWise} and each competing approach, respectively. 
We consider both \emph{average} and \emph{standard deviation} (SD) of certified lower bounds; in particular, SD  is a suitable measure of sensitivity of the approximations to input images: A larger SD implies a better sensitivity \cite{wu2021tightening}. 
For \emph{efficiency}, we record the average computation time over the (correctly-classified) images. 

\vspace{1ex}
\noindent \textbf{Experimental Setup.}
All the experiments were conducted on a workstation running Ubuntu 18.04 with a 2.35GHz 32-core AMD EPYC 7452 CPU and 128 GB memory.

\subsection{Experimental Results}


\noindent 
\textbf{Experiment I.}
Table \ref{certified lower bound on Sigmoid models} shows the comparison results  for 22  Sigmoid models with non-negative weights. Regarding the precision of verification results,  our \textsc{NeWise} computes significantly larger certified lower bounds  than the competitors for \emph{all} the models. In particular, for \emph{average}, \textsc{NeWise} achieves up to 96.22\%  improvement, i.e., FNNs with 4 hidden layers trained on Fashion MNIST. 
\textsc{NeWise}
improves the precision even more with 
\emph{standard deviation} (up to 129.33\%).  This indicates that our approach is more sensitive to input images compared to the other approaches:
The more the certified lower bound is improved,
the larger deviation the network exhibits.

\begin{table}
	\footnotesize 
	\setlength{\tabcolsep}{2pt}
	\caption{Performance comparison of \textsc{NeWise} (NW) with \textsc{DeepCert} (DC), \textsc{VeriNet} (VN), and \textsc{RobustVerifier} (RV) on non-negative Tanh networks. 
		$t\times n$ refers to an FNN with $t$ layers and $n$ neurons per layer. $\mbox{CNN}_{t-c}$ denotes a CNN with $t$ layers and $c$ filters of size 3$\times$3.}
	\centering
	\begin{tabular}{|c|l|r|r|r|r|r|r|r|}
		\hline 
		\multirow{3}{*}{\textbf{Dataset}} & \multirow{3}{*}{\textbf{Model}} &   \multicolumn{7}{c|}{\textbf{Certified Lower Bound (Standard Deviation)}}   \\
		\hhline{~~-------}
		& & \textsc{NW} & \textsc{DC} &  Impr. (\%) & \textsc{VN} &  Impr. (\%) & \textsc{RV} &  Impr. (\%)   \\
		\hline
		\hline
		\multirow{8}{*}{MNIST} 
		& 5x100               & 0.0018 & 0.0005 & \cellcolor{atomictangerine} 233.96 $\uparrow$ & 0.0006 & \cellcolor{atomictangerine} 195.00 $\uparrow$ & 0.0006 & \cellcolor{atomictangerine} 216.07 $\uparrow$   \\
		& 3x700               & 0.0027 & 0.0008 & \cellcolor{atomictangerine} 251.28 $\uparrow$ & 0.0009 & \cellcolor{atomictangerine} 191.49 $\uparrow$ & 0.0009 & \cellcolor{atomictangerine} 191.49 $\uparrow$   \\ 
		& 3x400               & 0.0018 & 0.0007 & \cellcolor{atomictangerine} 166.67 $\uparrow$ & 0.0008 & \cellcolor{atomictangerine} 128.57 $\uparrow$ & 0.0007 & \cellcolor{atomictangerine} 137.84 $\uparrow$   \\ 
		& $\mbox{CNN}_{6-5}$  & 0.0243 & 0.0115 & \cellcolor{atomictangerine} 111.84 $\uparrow$ & 0.0131 & \cellcolor{apricot}          85.94 $\uparrow$ & 0.0087 & \cellcolor{atomictangerine} 179.13 $\uparrow$   \\ 
		& 3x50                & 0.0012 & 0.0009 & \cellcolor{apricot}          29.79 $\uparrow$ & 0.0010 & \cellcolor{apricot}          27.08 $\uparrow$ & 0.0008 & \cellcolor{apricot}          45.24 $\uparrow$   \\ 
		& $\mbox{CNN}_{3-4}$  & 0.0067 & 0.0055 & \cellcolor{antiquewhite}     21.05 $\uparrow$ & 0.0058 & \cellcolor{antiquewhite}     14.80 $\uparrow$ & 0.0055 & \cellcolor{antiquewhite}     20.83 $\uparrow$   \\ 
		& $\mbox{CNN}_{5-5}$  & 0.0108 & 0.0090 & \cellcolor{antiquewhite}     20.24 $\uparrow$ & 0.0092 & \cellcolor{antiquewhite}     17.50 $\uparrow$ & 0.0087 & \cellcolor{antiquewhite}     24.54 $\uparrow$   \\ 
		& $\mbox{CNN}_{4-5}$  & 0.0074 & 0.0061 & \cellcolor{antiquewhite}     21.21 $\uparrow$ & 0.0063 & \cellcolor{antiquewhite}     17.19 $\uparrow$ & 0.0060 & \cellcolor{antiquewhite}     23.63 $\uparrow$   \\ 
		\hline 
		\multirow{6}{*}{\makecell{Fashion\\MNIST}} 
		& 3x100               & 0.0156 & 0.0105 & \cellcolor{atomictangerine} 48.81 $\uparrow$ & 0.0110 & \cellcolor{atomictangerine} 42.04 $\uparrow$ & 0.0091 & \cellcolor{atomictangerine} 70.79 $\uparrow$   \\
		& $\mbox{CNN}_{4-5}$  & 0.0188 & 0.0134 & \cellcolor{atomictangerine} 41.12 $\uparrow$ & 0.0139 & \cellcolor{apricot}         35.83 $\uparrow$ & 0.0129 & \cellcolor{atomictangerine} 45.82 $\uparrow$   \\ 
		& $\mbox{CNN}_{6-5}$  & 0.0329 & 0.0237 & \cellcolor{apricot}         38.83 $\uparrow$ & 0.0242 & \cellcolor{apricot}         35.67 $\uparrow$ & 0.0180 & \cellcolor{atomictangerine} 82.66 $\uparrow$   \\ 
		& 2x100               & 0.0109 & 0.0081 & \cellcolor{apricot}         35.27 $\uparrow$ & 0.0085 & \cellcolor{antiquewhite}    28.59 $\uparrow$ & 0.0077 & \cellcolor{apricot}         41.95 $\uparrow$   \\ 
		& 2x200               & 0.0102 & 0.0076 & \cellcolor{antiquewhite}    33.38 $\uparrow$ & 0.0080 & \cellcolor{antiquewhite}    27.06 $\uparrow$ & 0.0076 & \cellcolor{antiquewhite}    33.55 $\uparrow$   \\ 
		& $\mbox{CNN}_{5-5}$  & 0.0201 & 0.0153 & \cellcolor{antiquewhite}    31.52 $\uparrow$ & 0.0158 & \cellcolor{antiquewhite}    27.03 $\uparrow$ & 0.0125 & \cellcolor{atomictangerine} 60.69 $\uparrow$   \\  
		\hline 
		\multirow{7}{*}{CIFAR10} 
		& 3x200              & 0.0176 & 0.0083 & \cellcolor{atomictangerine} 113.30 $\uparrow$ & 0.0092 & \cellcolor{atomictangerine} 90.91 $\uparrow$ & 0.0080 & \cellcolor{atomictangerine} 119.40 $\uparrow$   \\
		& 3x50               & 0.0111 & 0.0073 & \cellcolor{apricot}          53.52 $\uparrow$ & 0.0077 & \cellcolor{apricot}         44.73 $\uparrow$ & 0.0071 & \cellcolor{apricot}          56.32 $\uparrow$   \\ 
		& 3x100              & 0.0435 & 0.0243 & \cellcolor{apricot}          78.79 $\uparrow$ & 0.0270 & \cellcolor{apricot}         61.35 $\uparrow$ & 0.0256 & \cellcolor{apricot}          69.66 $\uparrow$   \\ 
		& 3x400              & 0.0441 & 0.0245 & \cellcolor{apricot}          79.67 $\uparrow$ & 0.0291 & \cellcolor{apricot}         51.58 $\uparrow$ & 0.0253 & \cellcolor{apricot}          74.69 $\uparrow$   \\ 
		& $\mbox{CNN}_{3-2}$ & 0.0106 & 0.0102 & \cellcolor{antiquewhite}      4.01 $\uparrow$ & 0.0102 & \cellcolor{antiquewhite}     4.01 $\uparrow$ & 0.0102 & \cellcolor{antiquewhite}      4.21 $\uparrow$   \\  
		& $\mbox{CNN}_{3-4}$ & 0.0062 & 0.0061 & \cellcolor{antiquewhite}      1.80 $\uparrow$ & 0.0061 & \cellcolor{antiquewhite}     1.80 $\uparrow$ & 0.0061 & \cellcolor{antiquewhite}      1.96 $\uparrow$   \\ 
		& $\mbox{CNN}_{3-5}$ & 0.0066 & 0.0065 & \cellcolor{antiquewhite}      1.86 $\uparrow$ & 0.0065 & \cellcolor{antiquewhite}     1.86 $\uparrow$ & 0.0065 & \cellcolor{antiquewhite}      2.02 $\uparrow$   \\ 
		\hline 
	\end{tabular}
	\label{certified lower bound on nn Tanh models}
	\vspace{-3mm}
\end{table}

Regarding efficiency, all the approaches incur similar overhead as expected (they share the same complexity, i.e., $O(1)$ on each neuron).
We use $p\pm q$ to denote their average time cost $p$  and the size  of the interval $2q$. Table \ref{certified lower bound on nn Tanh models} presents the results on the Tanh models. \textsc{NeWise} computes \emph{even larger} certified lower  bounds, e.g., with up to 251.28\%. We omit the similar time overheads in Table \ref{certified lower bound on nn Tanh models}. See Appendix C for the complete results, including those on the Arctan models. 

All these experimental results provide strong independent validation of our  mathematical proof of Theorem \ref{neuron->network}.

\noindent 
\textbf{Experiment II.} We evaluate the performance of Algorithm \ref{Algorithm1} on 1-hidden-layer networks. Table \ref{alg on Sigmoid models} shows the comparison results with the other three tools. We only show the metric of standard deviation due to space limit. 
As shown in the table, Algorithm \ref{Algorithm1} can compute larger bounds with up to 160.66\% improvement.
Complete results are available in the 
Appendix C
. 

Regarding efficiency, the searching algorithm needs more time because it is in polynomial time, unlike the constant-time approach for the non-negative models. 
Nevertheless, the gradient-decent-based approach has been proven an efficient and practical solution to such convex optimization problems \cite{haji2021comparison}. 
When the size of a network is reasonably small, such overhead is acceptable,  compared with the improvement of the verification results. 

\noindent 
\textbf{Experiment III.} 
Despite the infeasibility of network-wise tightest approximations in the general case (Section \ref{sec:network-wise}), we have explored the performance of our approximation method and the competitors on the networks of mixed weights. 
Table \ref{certified lower bound on general models} shows the certified lower bounds returned by each approach for 11 CNNs and 9 FNNs.  

First,
the performance of each approach as compared with the others  varies under different mixed-weight models. This 
coincides with our analysis in Section \ref{sec:network-wise}: 
Pure neuron-wise tightness does not imply a network-wise tightness. Moreover, we observe that 
our \textsc{NeWise} performs surprisingly better than other approaches on all the experimented CNNs, while \textsc{DeepCert}  and \textsc{VeriNet} return larger certified lower bounds on the FNNs. The results evidenced network architecture is another factor influencing the verification. One possible reason is that a convolutional neural network is more possible to be monotonic based on the fact that  the neurons' weights on the same layer are constrained to be identical~\cite{DBLP:journals/pieee/LeCunBBH98}.

\begin{table}[t]
	\footnotesize 
	\setlength{\tabcolsep}{2pt}
	\caption{Performance comparison of Algorithm \ref{Algorithm1} with \textsc{DeepCert} (DC), \textsc{VeriNet} (VN), and \textsc{RobustVerifier} (RV) on 1-hidden-layer Sigmoid networks. $t\times n$ refers to an FNN with $t$ layers and $n$ neurons per layer. $\mbox{CNN}_{t-c-f}$ denotes a CNN with $t$ layers and $c$ filters of size $f\times f$. $\ast$ and ${+}$ mark the models trained on MNIST and Fashion MNIST, respectively.}
	\centering
	\begin{tabular}{|c|c|r|r|r|r|r|r|r|}
		\hline 
		\multirow{3}{*}{\textbf{Arch.}} & \multirow{3}{*}{\textbf{Model}} &   \multicolumn{7}{c|}{\textbf{Certified Lower Bound (Standard Deviation)}}   \\
		\hhline{~~-------}
		& & \textsc{Alg.1} & \textsc{DC} &  Impr. (\%) & \textsc{VN} &  Impr. (\%) & \textsc{RV} &  Impr. (\%)   \\
		\hline
		\hline
		\multirow{10}{*}{CNN} 
		& $\mbox{CNN}_{2-1-5}$$^\ast$ & 0.0358 & 0.0145 & \cellcolor{atomictangerine} 146.32 $\uparrow$ & 0.0143 & \cellcolor{atomictangerine} 150.98 $\uparrow$ & 0.0137 & \cellcolor{atomictangerine} 160.66$\uparrow$   \\
		& $\mbox{CNN}_{2-2-5}$$^\ast$ & 0.0308 & 0.0208 & \cellcolor{apricot}          47.82 $\uparrow$ & 0.0207 & \cellcolor{apricot}          48.96 $\uparrow$ & 0.0187 & \cellcolor{apricot}          64.23$\uparrow$  \\ 
		& $\mbox{CNN}_{2-3-5}$$^\ast$ & 0.0305 & 0.0197 & \cellcolor{apricot}          54.80 $\uparrow$ & 0.0196 & \cellcolor{apricot}          55.28 $\uparrow$ & 0.0176 & \cellcolor{apricot}          73.57$\uparrow$  \\ 
		& $\mbox{CNN}_{2-4-5}$$^\ast$ & 0.0419 & 0.0233 & \cellcolor{apricot}          79.70 $\uparrow$ & 0.0232 & \cellcolor{apricot}          80.56 $\uparrow$ & 0.0210 & \cellcolor{apricot}          99.40$\uparrow$   \\
		& $\mbox{CNN}_{2-5-3}$$^\ast$ & 0.0319 & 0.0182 & \cellcolor{apricot}          75.22 $\uparrow$ & 0.0182 & \cellcolor{apricot}          75.89 $\uparrow$ & 0.0176 & \cellcolor{apricot}          81.59$\uparrow$   \\
		& $\mbox{CNN}_{2-1-5}$$^+$    & 0.0497 & 0.0385 & \cellcolor{antiquewhite}     29.08 $\uparrow$ & 0.0386 & \cellcolor{antiquewhite}     28.74 $\uparrow$ & 0.0348 & \cellcolor{apricot}          42.52$\uparrow$  \\ 
		& $\mbox{CNN}_{2-2-5}$$^+$    & 0.0547 & 0.0353 & \cellcolor{apricot}          54.77 $\uparrow$ & 0.0355 & \cellcolor{apricot}          53.94 $\uparrow$ & 0.0311 & \cellcolor{apricot}          75.66$\uparrow$  \\ 
		& $\mbox{CNN}_{2-3-5}$$^+$    & 0.0541 & 0.0371 & \cellcolor{apricot}          45.76 $\uparrow$ & 0.0374 & \cellcolor{apricot}          44.71 $\uparrow$ & 0.0344 & \cellcolor{apricot}          57.28$\uparrow$  \\ 
		& $\mbox{CNN}_{2-4-5}$$^+$    & 0.0540 & 0.0366 & \cellcolor{apricot}          47.48 $\uparrow$ & 0.0367 & \cellcolor{apricot}          47.00 $\uparrow$ & 0.0336 & \cellcolor{apricot}          60.41$\uparrow$  \\ 
		& $\mbox{CNN}_{2-5-3}$$^+$    & 0.0598 & 0.0340 & \cellcolor{apricot}          75.97 $\uparrow$ & 0.0340 & \cellcolor{apricot}          75.71 $\uparrow$ & 0.0312 & \cellcolor{apricot}          91.34$\uparrow$ \\ 
		\hline
		\multirow{10}{*}{FNN} 
		& 1x50 $\ast$ & 0.0122 & 0.0082 & \cellcolor{apricot}       49.16 $\uparrow$ & 0.0085 & \cellcolor{apricot}      43.89 $\uparrow$ & 0.0062 & \cellcolor{atomictangerine}  96.32 $\uparrow$  \\
		& 1x100$\ast$ & 0.0107 & 0.0064 & \cellcolor{apricot}       67.67 $\uparrow$ & 0.0066 & \cellcolor{apricot}      61.10 $\uparrow$ & 0.0050 & \cellcolor{atomictangerine} 114.80 $\uparrow$  \\
		& 1x150$\ast$ & 0.0124 & 0.0083 & \cellcolor{apricot}       50.17 $\uparrow$ & 0.0085 & \cellcolor{apricot}      45.59 $\uparrow$ & 0.0064 & \cellcolor{atomictangerine}  93.75 $\uparrow$ \\ 
		& 1x200$\ast$ & 0.0127 & 0.0074 & \cellcolor{apricot}       71.99 $\uparrow$ & 0.0076 & \cellcolor{apricot}      68.35 $\uparrow$ & 0.0058 & \cellcolor{atomictangerine} 120.58 $\uparrow$  \\
		& 1x250$\ast$ & 0.0120 & 0.0075 & \cellcolor{apricot}       60.93 $\uparrow$ & 0.0076 & \cellcolor{apricot}      58.37 $\uparrow$ & 0.0060 & \cellcolor{atomictangerine} 100.82 $\uparrow$ \\ 
		& 1x50 $^+$   & 0.0184 & 0.0117 & \cellcolor{apricot}       56.83 $\uparrow$ & 0.0122 & \cellcolor{apricot}      51.04 $\uparrow$ & 0.0089 & \cellcolor{atomictangerine} 107.35 $\uparrow$ \\ 
		& 1x100$^+$   & 0.0149 & 0.0119 & \cellcolor{antiquewhite}  25.06 $\uparrow$ & 0.0123 & \cellcolor{antiquewhite} 21.89 $\uparrow$ & 0.0088 & \cellcolor{apricot}          70.46 $\uparrow$ \\ 
		& 1x150$^+$   & 0.0183 & 0.0120 & \cellcolor{apricot}       52.83 $\uparrow$ & 0.0123 & \cellcolor{apricot}      49.35 $\uparrow$ & 0.0090 & \cellcolor{atomictangerine} 103.61 $\uparrow$ \\ 
		& 1x200$^+$   & 0.0216 & 0.0129 & \cellcolor{apricot}       67.41 $\uparrow$ & 0.0132 & \cellcolor{apricot}      63.74 $\uparrow$ & 0.0096 & \cellcolor{atomictangerine} 125.31 $\uparrow$  \\
		& 1x250$^+$   & 0.0170 & 0.0126 & \cellcolor{antiquewhite}  34.67 $\uparrow$ & 0.0128 & \cellcolor{antiquewhite} 32.88 $\uparrow$ & 0.0095 & \cellcolor{apricot}          77.96 $\uparrow$ \\ 
		\hline
	\end{tabular}
	\label{alg on Sigmoid models}
		\vspace{-3mm}
\end{table}

Finally, we observe that average and standard deviation share the same increase/decrease trends. This indicates that a tighter approximation is more sensitive to the input images, which  conforms to our conclusion in Experiment \textbf{I}. 

\begin{table*}[t]
	\caption{Performance comparison with \textsc{DeepCert} (DC), \textsc{VeriNet} (VN), and \textsc{RobustVerifier} (RV) on mixed-weights Sigmoid networks. $\ast$, ${+}$, and  $\#$ mark the models trained on MNIST, Fashion MNIST, and CIFAR10, respectively.}
	\label{certified lower bound on general models}
\footnotesize
	\setlength{\tabcolsep}{2.9pt}
		\begin{tabular}{|c|c|r|r|r|r|r|r|r|r|r|r|r|r|r|r|r|r|}
			\hline 
			\multirow{3}{*}{\textbf{Arch.}} & \multirow{3}{*}{\textbf{Model}} & \multirow{3}{*}{\textbf{\#Neur.}} &   \multicolumn{14}{c|}{\textbf{Certified Lower Bound}} & \multirow{3}{*}{\textbf{Time (s)}}  \\
			\hhline{~~~--------------~}
			& & & \multicolumn{7}{c|}{\textbf{Average}} & \multicolumn{7}{c|}{\textbf{ Standard Deviation}} & \\
			\hhline{~~~--------------~}
			& & & \textsc{NW} & \textsc{DC} &  Impr. (\%) & \textsc{VN} &  Impr. (\%) & \textsc{RV} &  Impr. (\%) & \textsc{NW} & \textsc{DC} &  Impr. (\%) & \textsc{VN} &  Impr. (\%) & \textsc{RV} &  Impr. (\%) &  \\
			\hline
			\hline
			\multirow{11}{*}{CNN} 
			& $\mbox{CNN}_{3\_2}$$^\ast$ & 2,514  & 0.0607 & 0.0579 & \cellcolor{apricot} 4.92  $\uparrow$  & 0.0580 & \cellcolor{apricot} 4.67  $\uparrow$   & 0.0569 & \cellcolor{apricot}  6.82 $\uparrow$   & 0.0219 & 0.0202 & \cellcolor{apricot} 8.06  $\uparrow$   & 0.0204 & \cellcolor{apricot} 7.37  $\uparrow$   & 0.0192 & \cellcolor{apricot}  13.79 $\uparrow$ &  0.17 $\pm$0.04  \\ 
			& $\mbox{CNN}_{3\_4}$$^\ast$ & 5,018  & 0.0478 & 0.0472 & \cellcolor{apricot} 1.17  $\uparrow$  & 0.0472 & \cellcolor{apricot} 1.29  $\uparrow$   & 0.0464 & \cellcolor{apricot}  2.95 $\uparrow$   & 0.0155 & 0.0153 & \cellcolor{apricot} 1.11  $\uparrow$   & 0.0153 & \cellcolor{apricot} 1.44  $\uparrow$   & 0.0146 & \cellcolor{apricot}   5.87 $\uparrow$ &  0.31 $\pm$0.08  \\ 
			& $\mbox{CNN}_{4\_5}$$^\ast$ & 8,680  & 0.0570 & 0.0539 & \cellcolor{apricot} 5.64  $\uparrow$  & 0.0543 & \cellcolor{apricot} 5.03  $\uparrow$   & 0.0522 & \cellcolor{apricot}  9.16 $\uparrow$   & 0.0157 & 0.0145 & \cellcolor{apricot} 8.64  $\uparrow$   & 0.0146 & \cellcolor{apricot} 7.52  $\uparrow$   & 0.0132 & \cellcolor{apricot}  19.45 $\uparrow$ &  1.18 $\pm$0.20  \\ 
			& $\mbox{CNN}_{5\_5}$$^\ast$ & 10,680 & 0.0581 & 0.0548 & \cellcolor{apricot} 6.06  $\uparrow$  & 0.0550 & \cellcolor{apricot} 5.63  $\uparrow$   & 0.0512 & \cellcolor{apricot} 13.42 $\uparrow$   & 0.0157 & 0.0142 & \cellcolor{apricot} 10.48 $\uparrow$   & 0.0144 & \cellcolor{apricot} 8.80  $\uparrow$   & 0.0120 & \cellcolor{apricot}  30.81 $\uparrow$ &  2.99 $\pm$0.38  \\ 
			& $\mbox{CNN}_{6\_5}$$^\ast$ & 12,300 & 0.0624 & 0.0590 & \cellcolor{apricot} 5.71  $\uparrow$  & 0.0588 & \cellcolor{apricot} 6.00  $\uparrow$   & 0.0541 & \cellcolor{apricot} 15.27 $\uparrow$   & 0.0171 & 0.0153 & \cellcolor{apricot} 12.03 $\uparrow$   & 0.0153 & \cellcolor{apricot} 11.96 $\uparrow$   & 0.0123 & \cellcolor{apricot}  39.72 $\uparrow$ &  5.72 $\pm$0.46  \\ 
			& $\mbox{CNN}_{8\_5}$$^\ast$ & 14,570 & 0.1191 & 0.0878 & \cellcolor{apricot} 35.58 $\uparrow$  & 0.0882 & \cellcolor{apricot} 35.02 $\uparrow$   & 0.0685 & \cellcolor{apricot} 73.75 $\uparrow$   & 0.0361 & 0.0248 & \cellcolor{apricot} 45.60 $\uparrow$   & 0.0255 & \cellcolor{apricot} 41.66 $\uparrow$   & 0.0163 & \cellcolor{apricot} 122.22 $\uparrow$ & 15.27 $\pm$0.78  \\ 
			& $\mbox{CNN}_{4\_5}$$^{+}$  & 8,680  & 0.0747 & 0.0720 & \cellcolor{apricot} 3.73  $\uparrow$  & 0.0720 & \cellcolor{apricot} 3.79  $\uparrow$   & 0.0666 & \cellcolor{apricot} 12.16 $\uparrow$   & 0.0413 & 0.0376 & \cellcolor{apricot} 9.85  $\uparrow$   & 0.0378 & \cellcolor{apricot} 9.12  $\uparrow$   & 0.0313 & \cellcolor{apricot}  31.73 $\uparrow$ &  1.19 $\pm$0.21    \\ 
			& $\mbox{CNN}_{5\_5}$$^{+}$  & 10,680 & 0.0704 & 0.0676 & \cellcolor{apricot} 4.14  $\uparrow$  & 0.0676 & \cellcolor{apricot} 4.14  $\uparrow$   & 0.0605 & \cellcolor{apricot} 16.51 $\uparrow$   & 0.0347 & 0.0318 & \cellcolor{apricot} 9.03  $\uparrow$   & 0.0320 & \cellcolor{apricot} 8.41  $\uparrow$   & 0.0244 & \cellcolor{apricot}  41.82 $\uparrow$ &  2.99 $\pm$0.40    \\ 
			& $\mbox{CNN}_{6\_5}$$^{+}$  & 12,300 & 0.0735 & 0.0695 & \cellcolor{apricot} 5.77  $\uparrow$  & 0.0691 & \cellcolor{apricot} 6.37  $\uparrow$   & 0.0626 & \cellcolor{apricot} 17.32 $\uparrow$   & 0.0368 & 0.0341 & \cellcolor{apricot} 7.97  $\uparrow$   & 0.0340 & \cellcolor{apricot} 8.35  $\uparrow$   & 0.0278 & \cellcolor{apricot}  32.57 $\uparrow$ &  5.81 $\pm$0.55  \\ 
			& $\mbox{CNN}_{3\_2}$$^\#$   & 3,378  & 0.0314 & 0.0312 & \cellcolor{apricot} 0.58  $\uparrow$  & 0.0312 & \cellcolor{apricot} 0.61  $\uparrow$   & 0.0311 & \cellcolor{apricot}  1.06 $\uparrow$   & 0.0172 & 0.0169 & \cellcolor{apricot} 1.65  $\uparrow$   & 0.0169 & \cellcolor{apricot} 1.84  $\uparrow$   & 0.0168 & \cellcolor{apricot}   2.69 $\uparrow$ &  0.31 $\pm$0.06  \\ 
			& $\mbox{CNN}_{6\_5}$$^\#$   & 17,110 & 0.0229 & 0.0224 & \cellcolor{apricot} 2.19  $\uparrow$  & 0.0223 & \cellcolor{apricot} 2.46  $\uparrow$   & 0.0212 & \cellcolor{apricot}  7.77 $\uparrow$   & 0.0158 & 0.0153 & \cellcolor{apricot} 3.20  $\uparrow$   & 0.0153 & \cellcolor{apricot} 3.20  $\uparrow$   & 0.0141 & \cellcolor{apricot}  12.13 $\uparrow$ & 10.31 $\pm$0.68  \\ 
			\hline
			\multirow{9}{*}{FNN} 
			& 3x50 $^\ast$ & 160   & 0.0069 & 0.0076 & \cellcolor{antiquewhite} -8.82 $\downarrow$ & 0.0077 & \cellcolor{antiquewhite} -9.77 $\downarrow$ & 0.0065 & \cellcolor{apricot}       6.62 $\uparrow$    & 0.0025 & 0.0027 & \cellcolor{antiquewhite}  -6.37 $\downarrow$ & 0.0028 & \cellcolor{antiquewhite}  -9.42 $\downarrow$ & 0.0021 & \cellcolor{apricot} 18.48 $\uparrow$ &   0.14 $\pm$0.00   \\ 
			& 3x100$^\ast$ & 310   & 0.0078 & 0.0086 & \cellcolor{antiquewhite} -9.44 $\downarrow$ & 0.0087 & \cellcolor{antiquewhite} -10.79$\downarrow$ & 0.0074 & \cellcolor{apricot}       4.44 $\uparrow$    & 0.0026 & 0.0029 & \cellcolor{antiquewhite} -10.14 $\downarrow$ & 0.0029 & \cellcolor{antiquewhite} -12.88 $\downarrow$ & 0.0023 & \cellcolor{apricot} 10.30 $\uparrow$ &   2.14 $\pm$0.03   \\ 
			& 3x200$^\ast$ & 610   & 0.0080 & 0.0091 & \cellcolor{antiquewhite} -11.69$\downarrow$ & 0.0091 & \cellcolor{antiquewhite} -12.36$\downarrow$ & 0.0079 & \cellcolor{apricot}       1.01 $\uparrow$    & 0.0026 & 0.0030 & \cellcolor{antiquewhite} -14.14 $\downarrow$ & 0.0031 & \cellcolor{antiquewhite} -16.39 $\downarrow$ & 0.0024 & \cellcolor{apricot}  5.37 $\uparrow$ &  10.77 $\pm$0.01   \\ 
			& 5x100$^\ast$ & 510   & 0.0057 & 0.0061 & \cellcolor{antiquewhite} -5.27 $\downarrow$ & 0.0062 & \cellcolor{antiquewhite} -6.66 $\downarrow$ & 0.0052 & \cellcolor{apricot}      10.79 $\uparrow$    & 0.0024 & 0.0025 & \cellcolor{antiquewhite}  -5.98 $\downarrow$ & 0.0026 & \cellcolor{antiquewhite}  -8.88 $\downarrow$ & 0.0021 & \cellcolor{apricot} 12.38 $\uparrow$ &   4.38 $\pm$0.03   \\ 
			& 6x500$^\ast$ & 3,010 & 0.0685 & 0.0778 & \cellcolor{antiquewhite} -11.95$\downarrow$ & 0.0776 & \cellcolor{antiquewhite} -11.73$\downarrow$ & 0.0665 & \cellcolor{apricot}       3.05 $\uparrow$    & 0.0186 & 0.0210 & \cellcolor{antiquewhite} -11.56 $\downarrow$ & 0.0210 & \cellcolor{antiquewhite} -11.69 $\downarrow$ & 0.0152 & \cellcolor{apricot} 21.98 $\uparrow$ & 154.39 $\pm$0.36   \\ 
			& 3x50 $^{+}$  & 160   & 0.0092 & 0.0101 & \cellcolor{antiquewhite} -9.67 $\downarrow$ & 0.0102 & \cellcolor{antiquewhite} -10.29$\downarrow$ & 0.0086 & \cellcolor{apricot}       6.64 $\uparrow$    & 0.0035 & 0.0037 & \cellcolor{antiquewhite}  -6.74 $\downarrow$ & 0.0038 & \cellcolor{antiquewhite}  -9.90 $\downarrow$ & 0.0030 & \cellcolor{apricot} 15.72 $\uparrow$ &   0.14 $\pm$0.00   \\ 
			& 5x100$^{+}$  & 510   & 0.0071 & 0.0078 & \cellcolor{antiquewhite} -8.51 $\downarrow$ & 0.0079 & \cellcolor{antiquewhite} -10.01$\downarrow$ & 0.0066 & \cellcolor{apricot}       8.23 $\uparrow$    & 0.0036 & 0.0040 & \cellcolor{antiquewhite}  -8.79 $\downarrow$ & 0.0041 & \cellcolor{antiquewhite} -11.68 $\downarrow$ & 0.0033 & \cellcolor{apricot} 11.35 $\uparrow$ &   4.44 $\pm$0.03   \\ 
			& 3x50 $^\#$   & 160   & 0.0041 & 0.0045 & \cellcolor{antiquewhite} -10.57$\downarrow$ & 0.0045 & \cellcolor{antiquewhite} -10.18$\downarrow$ & 0.0042 & \cellcolor{antiquewhite} -2.17 $\downarrow$  & 0.0018 & 0.0021 & \cellcolor{antiquewhite} -14.90 $\downarrow$ & 0.0021 & \cellcolor{antiquewhite} -14.49 $\downarrow$ & 0.0017 & \cellcolor{apricot}  2.91 $\uparrow$ &   0.43 $\pm$0.00   \\ 
			& 5x100$^\#$   & 510   & 0.0033 & 0.0037 & \cellcolor{antiquewhite} -10.60$\downarrow$ & 0.0037 & \cellcolor{antiquewhite} -10.60$\downarrow$ & 0.0033 & \cellcolor{antiquewhite} -0.60 $\downarrow$  & 0.0014 & 0.0017 & \cellcolor{antiquewhite} -15.06 $\downarrow$ & 0.0016 & \cellcolor{antiquewhite} -14.55 $\downarrow$ & 0.0013 & \cellcolor{apricot}  5.22 $\uparrow$ &   2.45 $\pm$0.01 \\ 
			\hline
	\end{tabular}
\end{table*}

\vspace{-1mm}
\subsection{Threats to Validity}
We discuss  potential threats to the validity of our approach in terms of its application domains.

\noindent \textbf{Neural Networks with ReLU Activation Functions.}
Despite the focus on the Sigmoid-like activation functions,
our approach is also applicable to the ReLU activation functions. A ReLU function $\sigma(x)=max(x,0)$, with $x\in [l,u]$, only needs approximation when $l<0$ and $u>0$; the upper, resp. lower, linear bound would be then  $y=\frac{u}{u-l}(x-l)$, resp. $y=0$. Hence, the approximation is the tightest for non-negative neural networks. However, linear approximation is not a necessity for ReLU due to its piece-wise linearity. There are more precise (both sound and complete) verification approaches (by using, e.g., SMT \cite{katz2017reluplex} and Mixed Integer Linear Programming~\cite{BotoevaKKLM20}) which could compute larger certified lower bounds.

\noindent \textbf{
	FNNs with 
	Mixed Weights.}
For such networks, it is generally unpredictable which approach would compute the most precise verification result (despite a 10\% decrease on average in our approach). To the best of our knowledge, the only feasible way to examine a proposed approximation under non-trivial FNNs is by empirical analysis. Tackling this fundamentally and efficiently remains to be an open research problem.

\section{Related Work}\label{sec:work}
This work is a sequel to many pioneering efforts, which we classify into the following three categories.

\noindent \textbf{Linear Approximations of Sigmoid-like activation functions.}
\textsc{NeVer} \cite{pulina2010abstraction} uses piece-wise linear constraints for  approximation and is therefore unscalable. Both CROWN \cite{zhang2018efficient} and CNN-Cert \cite{boopathy2019cnn} consider  the tangent line at the midpoint of $[l,u]$ as one of the linear bounds. \textsc{DeepCert} \cite{wu2021tightening}  defines a fine-grained approximation strategy by calculating the slopes of the two linear constraints according to $l$ and $u$.  \textsc{RobustVerifier} \cite{lin2019robustness} leverages Taylor expansion at the midpoint of $[l,u]$. These approximations are  intuitive but lack  rigorous justifications or proofs for their better performance.

Lyu \textit{et al.} \cite{lyu2020fastened} characterized the tightness of approximations in terms of the overestimation of output range of each hidden neuron. But they observed and admitted that by their definition tighter bounding lines do not ensure more precise results. By our definition, we show that in the case of one hidden layer, their definition also guarantees to be  network-wise tightest. They proposed a gradient-based searching algorithm for near-tightest approximations under their definition. However, the algorithm has been shown difficult to scale up to large-size networks because it needs to perform on every neuron, compared with other existing constant-time approaches~\cite{wu2021tightening,HenriksenL20}.  Our work is, to the best of our knowledge, the first provably tightest, constant-time  linear approximation.

\noindent \textbf{Defining Tightness for Linear Approximations.}
There has been a shift of focus from individual neurons to multiple neurons w.r.t. defining tightness for linear approximations, but most of the work only concerns about the ReLU networks.
Tjandraatmadja \textit{et al.} \cite{DBLP:conf/nips/TjandraatmadjaA20} experimentally show that the success of approximations hinges on how closely they approximate the object that they are relaxing. Salman \textit{et al.} \cite{salman2019convex} reveal an inherent barrier of the approximation-based approaches for the ReLU networks and require for the tightest pre-activation upper and lower bounds of all the neurons in networks. Singh \textit{et al.} \cite{singh2019beyond} approximate multiple neurons simultaneously to obtain the tighter bounds. In contrast to this line of research, we have defined both neuron-wise and network-wise tightness to  characterize linear approximations of
Sigmoid-like activation functions.

\noindent \textbf{Other Robustness Verification Approaches.} In addition to approximation, 
 other techniques have also been used for  the robustness verification of neural networks. 
Abstract interpretation \cite{cousot1977abstract}, a technique that was  originally proposed for program verification, has been proven both effective and efficient in neural network verification~\cite{gehr2018ai2,singh2018fast,singh2019abstract}. 
 These approaches also rely on over-approximation but to transform the original verification problem into dedicated abstract domains. We believe that our approximation approach is also applicable to produce more precise verification results for non-negative neural networks. Other verification methods leverage the Lipschitz continuity feature of neural networks to estimate the output ranges \cite{ruan2018reachability,combettes2020lipschitz,lee2020lipschitz}.
 Although the approximation to an activation function can be bypassed using the Lipschitz constant, it would still be helpful to compute tighter Lipschitz constants by estimating the input range of the activation function via the approximation. 
	
\vspace{-2mm}
\section{Concluding Remarks}\label{sec:con}
We have presented \textit{network-wise tightness}, a novel and unified characterization of the tightness  of linear approximations in robustness verification of Sigmoid-like neural networks. We have shown that (i) to achieve precise verification results, activation functions in a network should \emph{not} be approximated with the same existing neuron-wise tightness criterion; (ii) computing the network-wise tightest approximation is computationally expensive and impractical due to its non-convexity; and (iii) how to bypass the complexity barrier via a neuron-wise tightest approximation. The experimental results demonstrate that our approximation approach outperforms state-of-the-art approaches under three scenarios, i.e.,  non-negative networks, 1-hidden-layer networks, and convolutional networks. 

Our work sheds light on the pursuit of robust neural networks via tightening linear approximations. 
The ineffectiveness of neuron-wise tightness on general networks calls for new, potentially hybrid, approximation strategies. The intrinsic high complexity in computing the network-wise tightest approximation motivates us to rethink of both fundamental and the heuristic  trade-offs between precision and efficiency in neural network verification. For mixed-weight neural networks, there may be latent  factors that could influence the tightness of approximations. One promising  direction is to explore possible combinations of existing tightness  characterizations to achieve network-wise tightness while taking into account the features of weight distributions and network architectures.

\vspace{-2mm}
\section*{Acknowledgments}
The authors thank the reviewers for their constructive comments. 
This work was supported in part by the National Key Research and Development (2019YFA0706404), the National Nature Science Foundation of China (61972150), the NSFC-ISF Joint Program (62161146001, 3420/21), the Fundamental Research Funds for Central Universities, and the Opening Project of Shanghai Trusted Industrial Control Platform.
Jing Liu and Min Zhang are the corresponding authors.

\onecolumn \begin{multicols}{2}


\begin{thebibliography}{55}

\ifx \showCODEN    \undefined \def \showCODEN     #1{\unskip}     \fi
\ifx \showDOI      \undefined \def \showDOI       #1{#1}\fi
\ifx \showISBNx    \undefined \def \showISBNx     #1{\unskip}     \fi
\ifx \showISBNxiii \undefined \def \showISBNxiii  #1{\unskip}     \fi
\ifx \showISSN     \undefined \def \showISSN      #1{\unskip}     \fi
\ifx \showLCCN     \undefined \def \showLCCN      #1{\unskip}     \fi
\ifx \shownote     \undefined \def \shownote      #1{#1}          \fi
\ifx \showarticletitle \undefined \def \showarticletitle #1{#1}   \fi
\ifx \showURL      \undefined \def \showURL       {\relax}        \fi
\providecommand\bibfield[2]{#2}
\providecommand\bibinfo[2]{#2}
\providecommand\natexlab[1]{#1}
\providecommand\showeprint[2][]{arXiv:#2}

\bibitem[\protect\citeauthoryear{Ali and Yangyu}{Ali and Yangyu}{2017}]%
        {ali2017automatic}
\bibfield{author}{\bibinfo{person}{Afan Ali} {and} \bibinfo{person}{Fan
  Yangyu}.} \bibinfo{year}{2017}\natexlab{}.
\newblock \showarticletitle{Automatic modulation classification using deep
  learning based on sparse autoencoders with nonnegativity constraints}.
\newblock \bibinfo{journal}{\emph{IEEE signal processing letters}}
  \bibinfo{volume}{24}, \bibinfo{number}{11} (\bibinfo{year}{2017}),
  \bibinfo{pages}{1626--1630}.
\newblock


\bibitem[\protect\citeauthoryear{aptx4869tjx}{aptx4869tjx}{2021}]%
        {tjx_models}
\bibfield{author}{\bibinfo{person}{aptx4869tjx}.}
  \bibinfo{year}{2021}\natexlab{}.
\newblock \bibinfo{title}{Pretrained Models}.
\newblock
  \bibinfo{howpublished}{\url{https://github.com/aptx4869tjx/train_network}}.
\newblock


\bibitem[\protect\citeauthoryear{Baier and Katoen}{Baier and Katoen}{2008}]%
        {baier2008principles}
\bibfield{author}{\bibinfo{person}{Christel Baier} {and}
  \bibinfo{person}{Joost-Pieter Katoen}.} \bibinfo{year}{2008}\natexlab{}.
\newblock \bibinfo{booktitle}{\emph{Principles of model checking}}.
\newblock \bibinfo{publisher}{MIT press}.
\newblock


\bibitem[\protect\citeauthoryear{Baluta, Chua, Meel, and Saxena}{Baluta
  et~al\mbox{.}}{2021}]%
        {BalutaCMS21}
\bibfield{author}{\bibinfo{person}{Teodora Baluta},
  \bibinfo{person}{Zheng~Leong Chua}, \bibinfo{person}{Kuldeep~S. Meel}, {and}
  \bibinfo{person}{Prateek Saxena}.} \bibinfo{year}{2021}\natexlab{}.
\newblock \showarticletitle{Scalable Quantitative Verification For Deep Neural
  Networks}. In \bibinfo{booktitle}{\emph{43rd {IEEE/ACM} International
  Conference on Software Engineering, {ICSE} 2021, Madrid, Spain, 22-30 May
  2021}}. \bibinfo{publisher}{{IEEE}}, \bibinfo{pages}{312--323}.
\newblock


\bibitem[\protect\citeauthoryear{Boopathy, Weng, Chen, Liu, and
  Daniel}{Boopathy et~al\mbox{.}}{2019}]%
        {boopathy2019cnn}
\bibfield{author}{\bibinfo{person}{Akhilan Boopathy}, \bibinfo{person}{Tsui-Wei
  Weng}, \bibinfo{person}{Pin-Yu Chen}, \bibinfo{person}{Sijia Liu}, {and}
  \bibinfo{person}{Luca Daniel}.} \bibinfo{year}{2019}\natexlab{}.
\newblock \showarticletitle{CNN-Cert: An Efficient Framework for Certifying
  Robustness of Convolutional Neural Networks}. In
  \bibinfo{booktitle}{\emph{AAAI Conference on Artificial Intelligence
  (AAAI)}}. \bibinfo{pages}{3240--3247}.
\newblock


\bibitem[\protect\citeauthoryear{Botoeva, Kouvaros, Kronqvist, Lomuscio, and
  Misener}{Botoeva et~al\mbox{.}}{2020}]%
        {BotoevaKKLM20}
\bibfield{author}{\bibinfo{person}{Elena Botoeva}, \bibinfo{person}{Panagiotis
  Kouvaros}, \bibinfo{person}{Jan Kronqvist}, \bibinfo{person}{Alessio
  Lomuscio}, {and} \bibinfo{person}{Ruth Misener}.}
  \bibinfo{year}{2020}\natexlab{}.
\newblock \showarticletitle{Efficient Verification of ReLU-Based Neural
  Networks via Dependency Analysis}. In \bibinfo{booktitle}{\emph{{AAAI}
  Conference on Artificial Intelligence (AAAI)}}. \bibinfo{publisher}{{AAAI}
  Press}, \bibinfo{pages}{3291--3299}.
\newblock


\bibitem[\protect\citeauthoryear{Carlini and Wagner}{Carlini and
  Wagner}{2017}]%
        {carlini2017towards}
\bibfield{author}{\bibinfo{person}{Nicholas Carlini} {and}
  \bibinfo{person}{David Wagner}.} \bibinfo{year}{2017}\natexlab{}.
\newblock \showarticletitle{Towards evaluating the robustness of neural
  networks}. In \bibinfo{booktitle}{\emph{IEEE symposium on security and
  privacy (S\&P)}}. IEEE, \bibinfo{pages}{39--57}.
\newblock


\bibitem[\protect\citeauthoryear{Ceschin, Botacin, Gomes, Oliveira, and
  Gr{\'e}gio}{Ceschin et~al\mbox{.}}{2019}]%
        {ceschin2019shallow}
\bibfield{author}{\bibinfo{person}{Fabr{\'\i}cio Ceschin},
  \bibinfo{person}{Marcus Botacin}, \bibinfo{person}{Heitor~Murilo Gomes},
  \bibinfo{person}{Luiz~S Oliveira}, {and} \bibinfo{person}{Andr{\'e}
  Gr{\'e}gio}.} \bibinfo{year}{2019}\natexlab{}.
\newblock \showarticletitle{Shallow security: On the creation of adversarial
  variants to evade machine learning-based malware detectors}. In
  \bibinfo{booktitle}{\emph{Proceedings of the 3rd Reversing and
  Offensive-oriented Trends Symposium}}. \bibinfo{pages}{1--9}.
\newblock


\bibitem[\protect\citeauthoryear{Clarke}{Clarke}{1997}]%
        {clarke1997model}
\bibfield{author}{\bibinfo{person}{Edmund~M Clarke}.}
  \bibinfo{year}{1997}\natexlab{}.
\newblock \showarticletitle{Model checking}. In
  \bibinfo{booktitle}{\emph{International Conference on Foundations of Software
  Technology and Theoretical Computer Science}}. Springer,
  \bibinfo{pages}{54--56}.
\newblock


\bibitem[\protect\citeauthoryear{Combettes and Pesquet}{Combettes and
  Pesquet}{2020}]%
        {combettes2020lipschitz}
\bibfield{author}{\bibinfo{person}{Patrick~L Combettes} {and}
  \bibinfo{person}{Jean-Christophe Pesquet}.} \bibinfo{year}{2020}\natexlab{}.
\newblock \showarticletitle{Lipschitz certificates for layered network
  structures driven by averaged activation operators}.
\newblock \bibinfo{journal}{\emph{SIAM Journal on Mathematics of Data Science}}
  \bibinfo{volume}{2}, \bibinfo{number}{2} (\bibinfo{year}{2020}),
  \bibinfo{pages}{529--557}.
\newblock


\bibitem[\protect\citeauthoryear{Cousot and Cousot}{Cousot and Cousot}{1977}]%
        {cousot1977abstract}
\bibfield{author}{\bibinfo{person}{Patrick Cousot} {and}
  \bibinfo{person}{Radhia Cousot}.} \bibinfo{year}{1977}\natexlab{}.
\newblock \showarticletitle{Abstract interpretation: a unified lattice model
  for static analysis of programs by construction or approximation of
  fixpoints}. In \bibinfo{booktitle}{\emph{ACM Symposium on Principles of
  Programming Languages (POPL)}}. \bibinfo{pages}{238--252}.
\newblock


\bibitem[\protect\citeauthoryear{Dunn, Pouget, Kroening, and Melham}{Dunn
  et~al\mbox{.}}{2021}]%
        {DunnPKM21}
\bibfield{author}{\bibinfo{person}{Isaac Dunn}, \bibinfo{person}{Hadrien
  Pouget}, \bibinfo{person}{Daniel Kroening}, {and} \bibinfo{person}{Tom
  Melham}.} \bibinfo{year}{2021}\natexlab{}.
\newblock \showarticletitle{Exposing previously undetectable faults in deep
  neural networks}. In \bibinfo{booktitle}{\emph{30th {ACM} {SIGSOFT}
  International Symposium on Software Testing and Analysis (ISSTA)}}.
  \bibinfo{publisher}{{ACM}}, \bibinfo{pages}{56--66}.
\newblock


\bibitem[\protect\citeauthoryear{Dutta, Jha, Sankaranarayanan, and
  Tiwari}{Dutta et~al\mbox{.}}{2018}]%
        {dutta2018output}
\bibfield{author}{\bibinfo{person}{Souradeep Dutta}, \bibinfo{person}{Susmit
  Jha}, \bibinfo{person}{Sriram Sankaranarayanan}, {and}
  \bibinfo{person}{Ashish Tiwari}.} \bibinfo{year}{2018}\natexlab{}.
\newblock \showarticletitle{Output Range Analysis for Deep Feedforward Neural
  Networks}. In \bibinfo{booktitle}{\emph{NASA Formal Methods Symposium
  (NFM)}}. Springer, \bibinfo{pages}{121--138}.
\newblock


\bibitem[\protect\citeauthoryear{Ehlers}{Ehlers}{2017}]%
        {ehlers2017formal}
\bibfield{author}{\bibinfo{person}{Ruediger Ehlers}.}
  \bibinfo{year}{2017}\natexlab{}.
\newblock \showarticletitle{Formal verification of piece-wise linear
  feed-forward neural networks}. In \bibinfo{booktitle}{\emph{International
  Symposium on Automated Technology for Verification and Analysis}}. Springer,
  \bibinfo{pages}{269--286}.
\newblock


\bibitem[\protect\citeauthoryear{Fleshman, Raff, Sylvester,
  et~al\mbox{.}}{Fleshman et~al\mbox{.}}{2018}]%
        {DBLP:journals/corr/abs-1806-06108}
\bibfield{author}{\bibinfo{person}{William Fleshman}, \bibinfo{person}{Edward
  Raff}, \bibinfo{person}{Jared Sylvester}, {et~al\mbox{.}}}
  \bibinfo{year}{2018}\natexlab{}.
\newblock \showarticletitle{Non-Negative Networks Against Adversarial Attacks}.
\newblock \bibinfo{journal}{\emph{CoRR}}  \bibinfo{volume}{abs/1806.06108}
  (\bibinfo{year}{2018}).
\newblock
\urldef\tempurl%
\url{http://arxiv.org/abs/1806.06108}
\showURL{%
\tempurl}


\bibitem[\protect\citeauthoryear{Gehr, Mirman, Drachsler-Cohen, Tsankov,
  Chaudhuri, and Vechev}{Gehr et~al\mbox{.}}{2018}]%
        {gehr2018ai2}
\bibfield{author}{\bibinfo{person}{Timon Gehr}, \bibinfo{person}{Matthew
  Mirman}, \bibinfo{person}{Dana Drachsler-Cohen}, \bibinfo{person}{Petar
  Tsankov}, \bibinfo{person}{Swarat Chaudhuri}, {and} \bibinfo{person}{Martin
  Vechev}.} \bibinfo{year}{2018}\natexlab{}.
\newblock \showarticletitle{AI2: Safety and Robustness Certification of Neural
  Networks with Abstract Interpretation}. In \bibinfo{booktitle}{\emph{IEEE
  Symposium on Security and Privacy (S\&P)}}. IEEE, \bibinfo{pages}{3--18}.
\newblock


\bibitem[\protect\citeauthoryear{Haji and Abdulazeez}{Haji and
  Abdulazeez}{2021}]%
        {haji2021comparison}
\bibfield{author}{\bibinfo{person}{Saad~Hikmat Haji} {and}
  \bibinfo{person}{Adnan~Mohsin Abdulazeez}.} \bibinfo{year}{2021}\natexlab{}.
\newblock \showarticletitle{Comparison of optimization techniques based on
  gradient descent algorithm: A review}.
\newblock \bibinfo{journal}{\emph{PalArch's Journal of Archaeology of
  Egypt/Egyptology}} \bibinfo{volume}{18}, \bibinfo{number}{4}
  (\bibinfo{year}{2021}), \bibinfo{pages}{2715--2743}.
\newblock


\bibitem[\protect\citeauthoryear{Henriksen and Lomuscio}{Henriksen and
  Lomuscio}{2020}]%
        {HenriksenL20}
\bibfield{author}{\bibinfo{person}{Patrick Henriksen} {and}
  \bibinfo{person}{Alessio~R. Lomuscio}.} \bibinfo{year}{2020}\natexlab{}.
\newblock \showarticletitle{Efficient Neural Network Verification via Adaptive
  Refinement and Adversarial Search}. In \bibinfo{booktitle}{\emph{European
  Conference on Artificial Intelligence (ECAI)}}. \bibinfo{publisher}{{IOS}
  Press}, \bibinfo{pages}{2513--2520}.
\newblock


\bibitem[\protect\citeauthoryear{Kargarnovin, Sadeghzadeh, and
  Jalili}{Kargarnovin et~al\mbox{.}}{2021}]%
        {kargarnovin2021mal2gcn}
\bibfield{author}{\bibinfo{person}{Omid Kargarnovin},
  \bibinfo{person}{Amir~Mahdi Sadeghzadeh}, {and} \bibinfo{person}{Rasool
  Jalili}.} \bibinfo{year}{2021}\natexlab{}.
\newblock \showarticletitle{Mal2GCN: A Robust Malware Detection Approach Using
  Deep Graph Convolutional Networks With Non-Negative Weights}.
\newblock \bibinfo{journal}{\emph{arXiv preprint arXiv:2108.12473}}
  (\bibinfo{year}{2021}).
\newblock


\bibitem[\protect\citeauthoryear{Katz, Barrett, Dill, Julian, and
  Kochenderfer}{Katz et~al\mbox{.}}{2017}]%
        {katz2017reluplex}
\bibfield{author}{\bibinfo{person}{Guy Katz}, \bibinfo{person}{Clark Barrett},
  \bibinfo{person}{David~L Dill}, \bibinfo{person}{Kyle Julian}, {and}
  \bibinfo{person}{Mykel~J Kochenderfer}.} \bibinfo{year}{2017}\natexlab{}.
\newblock \showarticletitle{Reluplex: An Efficient SMT Solver for Verifying
  Deep Neural Networks}. In \bibinfo{booktitle}{\emph{International Conference
  on Computer Aided Verification (CAV)}}. Springer, \bibinfo{pages}{97--117}.
\newblock


\bibitem[\protect\citeauthoryear{Krizhevsky, Hinton, et~al\mbox{.}}{Krizhevsky
  et~al\mbox{.}}{2009}]%
        {krizhevsky2009learning}
\bibfield{author}{\bibinfo{person}{Alex Krizhevsky}, \bibinfo{person}{Geoffrey
  Hinton}, {et~al\mbox{.}}} \bibinfo{year}{2009}\natexlab{}.
\newblock \showarticletitle{Learning Multiple Layers of Features from Tiny
  Images}.
\newblock  (\bibinfo{year}{2009}).
\newblock


\bibitem[\protect\citeauthoryear{Krizhevsky, Sutskever, and Hinton}{Krizhevsky
  et~al\mbox{.}}{2017}]%
        {DBLP:journals/cacm/KrizhevskySH17}
\bibfield{author}{\bibinfo{person}{Alex Krizhevsky}, \bibinfo{person}{Ilya
  Sutskever}, {and} \bibinfo{person}{Geoffrey~E. Hinton}.}
  \bibinfo{year}{2017}\natexlab{}.
\newblock \showarticletitle{ImageNet classification with deep convolutional
  neural networks}.
\newblock \bibinfo{journal}{\emph{Commun. ACM}} \bibinfo{volume}{60},
  \bibinfo{number}{6} (\bibinfo{year}{2017}), \bibinfo{pages}{84--90}.
\newblock


\bibitem[\protect\citeauthoryear{LeCun, Bottou, Bengio, and Haffner}{LeCun
  et~al\mbox{.}}{1998a}]%
        {lecun1998gradient}
\bibfield{author}{\bibinfo{person}{Yann LeCun}, \bibinfo{person}{L{\'e}on
  Bottou}, \bibinfo{person}{Yoshua Bengio}, {and} \bibinfo{person}{Patrick
  Haffner}.} \bibinfo{year}{1998}\natexlab{a}.
\newblock \showarticletitle{Gradient-Based Learning Applied to Document
  Recognition}.
\newblock \bibinfo{journal}{\emph{Proc. IEEE}} \bibinfo{volume}{86},
  \bibinfo{number}{11} (\bibinfo{year}{1998}), \bibinfo{pages}{2278--2324}.
\newblock


\bibitem[\protect\citeauthoryear{LeCun, Bottou, Bengio, and Haffner}{LeCun
  et~al\mbox{.}}{1998b}]%
        {DBLP:journals/pieee/LeCunBBH98}
\bibfield{author}{\bibinfo{person}{Yann LeCun}, \bibinfo{person}{L{\'{e}}on
  Bottou}, \bibinfo{person}{Yoshua Bengio}, {and} \bibinfo{person}{Patrick
  Haffner}.} \bibinfo{year}{1998}\natexlab{b}.
\newblock \showarticletitle{Gradient-based learning applied to document
  recognition}.
\newblock \bibinfo{journal}{\emph{Proc. IEEE}} \bibinfo{volume}{86},
  \bibinfo{number}{11} (\bibinfo{year}{1998}), \bibinfo{pages}{2278--2324}.
\newblock


\bibitem[\protect\citeauthoryear{Lee, Lee, and Park}{Lee et~al\mbox{.}}{2020}]%
        {lee2020lipschitz}
\bibfield{author}{\bibinfo{person}{Sungyoon Lee}, \bibinfo{person}{Jaewook
  Lee}, {and} \bibinfo{person}{Saerom Park}.} \bibinfo{year}{2020}\natexlab{}.
\newblock \showarticletitle{Lipschitz-certifiable training with a tight outer
  bound}.
\newblock \bibinfo{journal}{\emph{Annual Conference on Neural Information
  Processing Systems (NeurIPS)}}  \bibinfo{volume}{33} (\bibinfo{year}{2020}),
  \bibinfo{pages}{16891--16902}.
\newblock


\bibitem[\protect\citeauthoryear{Lin, Yang, Chen, Zhao, Li, Liu, and He}{Lin
  et~al\mbox{.}}{2019}]%
        {lin2019robustness}
\bibfield{author}{\bibinfo{person}{Wang Lin}, \bibinfo{person}{Zhengfeng Yang},
  \bibinfo{person}{Xin Chen}, \bibinfo{person}{Qingye Zhao},
  \bibinfo{person}{Xiangkun Li}, \bibinfo{person}{Zhiming Liu}, {and}
  \bibinfo{person}{Jifeng He}.} \bibinfo{year}{2019}\natexlab{}.
\newblock \showarticletitle{Robustness Verification of Classification Deep
  Neural Networks via Linear Programming}. In
  \bibinfo{booktitle}{\emph{IEEE/CVF Conference on Computer Vision and Pattern
  Recognition (CVPR)}}. \bibinfo{pages}{11418--11427}.
\newblock


\bibitem[\protect\citeauthoryear{Long, Shelhamer, and Darrell}{Long
  et~al\mbox{.}}{2015}]%
        {DBLP:conf/cvpr/LongSD15}
\bibfield{author}{\bibinfo{person}{Jonathan Long}, \bibinfo{person}{Evan
  Shelhamer}, {and} \bibinfo{person}{Trevor Darrell}.}
  \bibinfo{year}{2015}\natexlab{}.
\newblock \showarticletitle{Fully convolutional networks for semantic
  segmentation}. In \bibinfo{booktitle}{\emph{IEEE/CVF Conference on Computer
  Vision and Pattern Recognition (CVPR)}}. \bibinfo{publisher}{{IEEE} Computer
  Society}, \bibinfo{pages}{3431--3440}.
\newblock


\bibitem[\protect\citeauthoryear{Lyu, Ko, Kong, Wong, Lin, and Daniel}{Lyu
  et~al\mbox{.}}{2020}]%
        {lyu2020fastened}
\bibfield{author}{\bibinfo{person}{Zhaoyang Lyu}, \bibinfo{person}{Ching-Yun
  Ko}, \bibinfo{person}{Zhifeng Kong}, \bibinfo{person}{Ngai Wong},
  \bibinfo{person}{Dahua Lin}, {and} \bibinfo{person}{Luca Daniel}.}
  \bibinfo{year}{2020}\natexlab{}.
\newblock \showarticletitle{Fastened CROWN: Tightened Neural Network Robustness
  Certificates}. In \bibinfo{booktitle}{\emph{AAAI Conference on Artificial
  Intelligence (AAAI)}}. \bibinfo{pages}{5037--5044}.
\newblock


\bibitem[\protect\citeauthoryear{Neacsu, Pesquet, and Burileanu}{Neacsu
  et~al\mbox{.}}{2020}]%
        {DBLP:conf/icassp/NeacsuPB20}
\bibfield{author}{\bibinfo{person}{Ana Neacsu},
  \bibinfo{person}{Jean{-}Christophe Pesquet}, {and} \bibinfo{person}{Corneliu
  Burileanu}.} \bibinfo{year}{2020}\natexlab{}.
\newblock \showarticletitle{Accuracy-Robustness Trade-Off for Positively
  Weighted Neural Networks}. In \bibinfo{booktitle}{\emph{International
  Conference on Acoustics, Speech and Signal Processing (ICASSP)}}.
  \bibinfo{publisher}{{IEEE}}, \bibinfo{pages}{8389--8393}.
\newblock


\bibitem[\protect\citeauthoryear{Pan, Wang, Ding, and Yong}{Pan
  et~al\mbox{.}}{2017}]%
        {DBLP:conf/aaai/PanWDY17}
\bibfield{author}{\bibinfo{person}{Tianxiang Pan}, \bibinfo{person}{Bin Wang},
  \bibinfo{person}{Guiguang Ding}, {and} \bibinfo{person}{Jun{-}Hai Yong}.}
  \bibinfo{year}{2017}\natexlab{}.
\newblock \showarticletitle{Fully Convolutional Neural Networks with
  Full-Scale-Features for Semantic Segmentation}. In
  \bibinfo{booktitle}{\emph{{AAAI} Conference on Artificial Intelligence
  (AAAI)}}. \bibinfo{publisher}{{AAAI} Press}, \bibinfo{pages}{4240--4246}.
\newblock


\bibitem[\protect\citeauthoryear{Pathak and Paffenroth}{Pathak and
  Paffenroth}{2020}]%
        {pathak2020non}
\bibfield{author}{\bibinfo{person}{Harsh~Nilesh Pathak} {and}
  \bibinfo{person}{Randy~Clinton Paffenroth}.} \bibinfo{year}{2020}\natexlab{}.
\newblock \showarticletitle{Non-convex Optimization Using Parameter
  Continuation Methods for Deep Neural Networks}.
\newblock \bibinfo{journal}{\emph{Deep Learning Applications, Volume 2}}
  \bibinfo{volume}{1232} (\bibinfo{year}{2020}), \bibinfo{pages}{273--298}.
\newblock


\bibitem[\protect\citeauthoryear{Paulsen, Wang, Wang, and Wang}{Paulsen
  et~al\mbox{.}}{2020}]%
        {PaulsenWWW20}
\bibfield{author}{\bibinfo{person}{Brandon Paulsen}, \bibinfo{person}{Jingbo
  Wang}, \bibinfo{person}{Jiawei Wang}, {and} \bibinfo{person}{Chao Wang}.}
  \bibinfo{year}{2020}\natexlab{}.
\newblock \showarticletitle{{NEURODIFF:} Scalable Differential Verification of
  Neural Networks using Fine-Grained Approximation}. In
  \bibinfo{booktitle}{\emph{International Conference on Automated Software
  Engineering (ASE)}}. \bibinfo{publisher}{{IEEE}}, \bibinfo{pages}{784--796}.
\newblock


\bibitem[\protect\citeauthoryear{Pulina and Tacchella}{Pulina and
  Tacchella}{2010}]%
        {pulina2010abstraction}
\bibfield{author}{\bibinfo{person}{Luca Pulina} {and} \bibinfo{person}{Armando
  Tacchella}.} \bibinfo{year}{2010}\natexlab{}.
\newblock \showarticletitle{An Abstraction-Refinement Approach to Verification
  of Artificial Neural Networks}. In \bibinfo{booktitle}{\emph{International
  Conference on Computer Aided Verification (CAV)}}. Springer,
  \bibinfo{pages}{243--257}.
\newblock


\bibitem[\protect\citeauthoryear{Ren, He, Girshick, and Sun}{Ren
  et~al\mbox{.}}{2017}]%
        {DBLP:journals/pami/RenHG017}
\bibfield{author}{\bibinfo{person}{Shaoqing Ren}, \bibinfo{person}{Kaiming He},
  \bibinfo{person}{Ross~B. Girshick}, {and} \bibinfo{person}{Jian Sun}.}
  \bibinfo{year}{2017}\natexlab{}.
\newblock \showarticletitle{Faster {R-CNN:} Towards Real-Time Object Detection
  with Region Proposal Networks}.
\newblock \bibinfo{journal}{\emph{{IEEE} Transactions on Pattern Analysis and
  Machine Intelligence}} \bibinfo{volume}{39}, \bibinfo{number}{6}
  (\bibinfo{year}{2017}), \bibinfo{pages}{1137--1149}.
\newblock


\bibitem[\protect\citeauthoryear{Ruan, Huang, and Kwiatkowska}{Ruan
  et~al\mbox{.}}{2018}]%
        {ruan2018reachability}
\bibfield{author}{\bibinfo{person}{Wenjie Ruan}, \bibinfo{person}{Xiaowei
  Huang}, {and} \bibinfo{person}{Marta Kwiatkowska}.}
  \bibinfo{year}{2018}\natexlab{}.
\newblock \showarticletitle{Reachability analysis of deep neural networks with
  provable guarantees}. In \bibinfo{booktitle}{\emph{International Joint
  Conference on Artificial Intelligence(IJCAI)}}. \bibinfo{pages}{2651--2659}.
\newblock


\bibitem[\protect\citeauthoryear{Salman, Yang, Zhang, Hsieh, and Zhang}{Salman
  et~al\mbox{.}}{2019}]%
        {salman2019convex}
\bibfield{author}{\bibinfo{person}{Hadi Salman}, \bibinfo{person}{Greg Yang},
  \bibinfo{person}{Huan Zhang}, \bibinfo{person}{Cho-Jui Hsieh}, {and}
  \bibinfo{person}{Pengchuan Zhang}.} \bibinfo{year}{2019}\natexlab{}.
\newblock \showarticletitle{A Convex Relaxation Barrier to Tight Robustness
  Verification of Neural Networks}. In \bibinfo{booktitle}{\emph{Annual
  Conference on Neural Information Processing Systems (NeurIPS)}}.
  \bibinfo{pages}{9832--9842}.
\newblock


\bibitem[\protect\citeauthoryear{S{\"a}lzer and Lange}{S{\"a}lzer and
  Lange}{2021}]%
        {salzer2021reachability}
\bibfield{author}{\bibinfo{person}{Marco S{\"a}lzer} {and}
  \bibinfo{person}{Martin Lange}.} \bibinfo{year}{2021}\natexlab{}.
\newblock \showarticletitle{Reachability Is NP-Complete Even for the Simplest
  Neural Networks}.
\newblock \bibinfo{journal}{\emph{CoRR}}  \bibinfo{volume}{abs/2108.13179}
  (\bibinfo{year}{2021}).
\newblock


\bibitem[\protect\citeauthoryear{Singh, Ganvir, P{\"{u}}schel, and
  Vechev}{Singh et~al\mbox{.}}{2019a}]%
        {singh2019beyond}
\bibfield{author}{\bibinfo{person}{Gagandeep Singh}, \bibinfo{person}{Rupanshu
  Ganvir}, \bibinfo{person}{Markus P{\"{u}}schel}, {and}
  \bibinfo{person}{Martin~T. Vechev}.} \bibinfo{year}{2019}\natexlab{a}.
\newblock \showarticletitle{Beyond the Single Neuron Convex Barrier for Neural
  Network Certification}. In \bibinfo{booktitle}{\emph{Annual Conference on
  Neural Information Processing Systems (NeurIPS)}}.
  \bibinfo{pages}{15072--15083}.
\newblock


\bibitem[\protect\citeauthoryear{Singh, Gehr, Mirman, P{\"u}schel, and
  Vechev}{Singh et~al\mbox{.}}{2018}]%
        {singh2018fast}
\bibfield{author}{\bibinfo{person}{Gagandeep Singh}, \bibinfo{person}{Timon
  Gehr}, \bibinfo{person}{Matthew Mirman}, \bibinfo{person}{Markus
  P{\"u}schel}, {and} \bibinfo{person}{Martin~T Vechev}.}
  \bibinfo{year}{2018}\natexlab{}.
\newblock \showarticletitle{Fast and Effective Robustness Certification.}. In
  \bibinfo{booktitle}{\emph{Advances in Neural Information Processing Systems
  (NeurIPS)}}. \bibinfo{pages}{10825--10836}.
\newblock


\bibitem[\protect\citeauthoryear{Singh, Gehr, P{\"u}schel, and Vechev}{Singh
  et~al\mbox{.}}{2019b}]%
        {singh2019abstract}
\bibfield{author}{\bibinfo{person}{Gagandeep Singh}, \bibinfo{person}{Timon
  Gehr}, \bibinfo{person}{Markus P{\"u}schel}, {and} \bibinfo{person}{Martin
  Vechev}.} \bibinfo{year}{2019}\natexlab{b}.
\newblock \showarticletitle{An Abstract Domain for Certifying Neural Networks}.
\newblock \bibinfo{journal}{\emph{Proceedings of the ACM on Programming
  Languages (POPL)}} (\bibinfo{year}{2019}), \bibinfo{pages}{1--30}.
\newblock


\bibitem[\protect\citeauthoryear{Tjandraatmadja, Anderson, Huchette, Ma, Patel,
  and Vielma}{Tjandraatmadja et~al\mbox{.}}{2020}]%
        {DBLP:conf/nips/TjandraatmadjaA20}
\bibfield{author}{\bibinfo{person}{Christian Tjandraatmadja},
  \bibinfo{person}{Ross Anderson}, \bibinfo{person}{Joey Huchette},
  \bibinfo{person}{Will Ma}, \bibinfo{person}{Krunal Patel}, {and}
  \bibinfo{person}{Juan~Pablo Vielma}.} \bibinfo{year}{2020}\natexlab{}.
\newblock \showarticletitle{The Convex Relaxation Barrier, Revisited: Tightened
  Single-Neuron Relaxations for Neural Network Verification}. In
  \bibinfo{booktitle}{\emph{Annual Conference on Neural Information Processing
  Systems (NeurIPS)}}.
\newblock


\bibitem[\protect\citeauthoryear{Tjeng, Xiao, and Tedrake}{Tjeng
  et~al\mbox{.}}{2019}]%
        {tjeng2017evaluating}
\bibfield{author}{\bibinfo{person}{Vincent Tjeng}, \bibinfo{person}{Kai Xiao},
  {and} \bibinfo{person}{Russ Tedrake}.} \bibinfo{year}{2019}\natexlab{}.
\newblock \showarticletitle{Evaluating Robustness of Neural Networks with Mixed
  Integer Programming}. In \bibinfo{booktitle}{\emph{International Conference
  on Learning Representations (ICLR)}}.
\newblock


\bibitem[\protect\citeauthoryear{Toshev and Szegedy}{Toshev and
  Szegedy}{2014}]%
        {DBLP:conf/cvpr/ToshevS14}
\bibfield{author}{\bibinfo{person}{Alexander Toshev} {and}
  \bibinfo{person}{Christian Szegedy}.} \bibinfo{year}{2014}\natexlab{}.
\newblock \showarticletitle{DeepPose: Human Pose Estimation via Deep Neural
  Networks}. In \bibinfo{booktitle}{\emph{IEEE/CVF Conference on Computer
  Vision and Pattern Recognition (CVPR)}}. \bibinfo{publisher}{{IEEE} Computer
  Society}, \bibinfo{pages}{1653--1660}.
\newblock


\bibitem[\protect\citeauthoryear{Tran, Yang, Manzanas~Lopez, Musau, Nguyen,
  Xiang, Bak, and Johnson}{Tran et~al\mbox{.}}{2020}]%
        {tran2020nnv}
\bibfield{author}{\bibinfo{person}{Hoang-Dung Tran}, \bibinfo{person}{Xiaodong
  Yang}, \bibinfo{person}{Diego Manzanas~Lopez}, \bibinfo{person}{Patrick
  Musau}, \bibinfo{person}{Luan~Viet Nguyen}, \bibinfo{person}{Weiming Xiang},
  \bibinfo{person}{Stanley Bak}, {and} \bibinfo{person}{Taylor~T Johnson}.}
  \bibinfo{year}{2020}\natexlab{}.
\newblock \showarticletitle{{NNV}: the neural network verification tool for
  deep neural networks and learning-enabled cyber-physical systems}. In
  \bibinfo{booktitle}{\emph{International Conference on Computer Aided
  Verification(CAV)}}. Springer, \bibinfo{pages}{3--17}.
\newblock


\bibitem[\protect\citeauthoryear{Wang, Pei, Whitehouse, Yang, and Jana}{Wang
  et~al\mbox{.}}{2018a}]%
        {wang2018efficient}
\bibfield{author}{\bibinfo{person}{Shiqi Wang}, \bibinfo{person}{Kexin Pei},
  \bibinfo{person}{Justin Whitehouse}, \bibinfo{person}{Junfeng Yang}, {and}
  \bibinfo{person}{Suman Jana}.} \bibinfo{year}{2018}\natexlab{a}.
\newblock \showarticletitle{Efficient formal safety analysis of neural
  networks}. In \bibinfo{booktitle}{\emph{Annual Conference on Neural
  Information Processing Systems (NeurIPS)}}. \bibinfo{pages}{6369--6379}.
\newblock


\bibitem[\protect\citeauthoryear{Wang, Pei, Whitehouse, Yang, and Jana}{Wang
  et~al\mbox{.}}{2018b}]%
        {wang2018formal}
\bibfield{author}{\bibinfo{person}{Shiqi Wang}, \bibinfo{person}{Kexin Pei},
  \bibinfo{person}{Justin Whitehouse}, \bibinfo{person}{Junfeng Yang}, {and}
  \bibinfo{person}{Suman Jana}.} \bibinfo{year}{2018}\natexlab{b}.
\newblock \showarticletitle{Formal Security Analysis of Neural Networks using
  Symbolic Intervals}. In \bibinfo{booktitle}{\emph{USENIX Security Symposium
  (USENIX Security)}}. \bibinfo{pages}{1599--1614}.
\newblock


\bibitem[\protect\citeauthoryear{Wang, Zhang, Xu, Lin, Jana, Hsieh, and
  Kolter}{Wang et~al\mbox{.}}{2021}]%
        {wang2021beta}
\bibfield{author}{\bibinfo{person}{Shiqi Wang}, \bibinfo{person}{Huan Zhang},
  \bibinfo{person}{Kaidi Xu}, \bibinfo{person}{Xue Lin}, \bibinfo{person}{Suman
  Jana}, \bibinfo{person}{Cho-Jui Hsieh}, {and} \bibinfo{person}{J~Zico
  Kolter}.} \bibinfo{year}{2021}\natexlab{}.
\newblock \showarticletitle{Beta-crown: Efficient bound propagation with
  per-neuron split constraints for neural network robustness verification}.
\newblock \bibinfo{journal}{\emph{Annual Conference on Neural Information
  Processing Systems (NeurIPS)}}  \bibinfo{volume}{34} (\bibinfo{year}{2021}).
\newblock


\bibitem[\protect\citeauthoryear{Weng, Zhang, Chen, Song, et~al\mbox{.}}{Weng
  et~al\mbox{.}}{2018a}]%
        {weng2018towards}
\bibfield{author}{\bibinfo{person}{Lily Weng}, \bibinfo{person}{Huan Zhang},
  \bibinfo{person}{Hongge Chen}, \bibinfo{person}{Zhao Song}, {et~al\mbox{.}}}
  \bibinfo{year}{2018}\natexlab{a}.
\newblock \showarticletitle{Towards Fast Computation of Certified Robustness
  for ReLU Networks}. In \bibinfo{booktitle}{\emph{International Conference on
  Machine Learning (ICML)}}. PMLR, \bibinfo{pages}{5276--5285}.
\newblock


\bibitem[\protect\citeauthoryear{Weng, Zhang, Chen, Song, et~al\mbox{.}}{Weng
  et~al\mbox{.}}{2018b}]%
        {WengZCSHDBD18}
\bibfield{author}{\bibinfo{person}{Tsui{-}Wei Weng}, \bibinfo{person}{Huan
  Zhang}, \bibinfo{person}{Hongge Chen}, \bibinfo{person}{Zhao Song},
  {et~al\mbox{.}}} \bibinfo{year}{2018}\natexlab{b}.
\newblock \showarticletitle{Towards Fast Computation of Certified Robustness
  for ReLU Networks}. In \bibinfo{booktitle}{\emph{International Conference on
  Machine Learning{ICML}}}, Vol.~\bibinfo{volume}{80}.
  \bibinfo{publisher}{{PMLR}}, \bibinfo{pages}{5273--5282}.
\newblock


\bibitem[\protect\citeauthoryear{Wing}{Wing}{2021}]%
        {wing2021trustworthy}
\bibfield{author}{\bibinfo{person}{Jeannette~M Wing}.}
  \bibinfo{year}{2021}\natexlab{}.
\newblock \showarticletitle{Trustworthy {AI}}.
\newblock \bibinfo{journal}{\emph{Commun. ACM}} \bibinfo{volume}{64},
  \bibinfo{number}{10} (\bibinfo{year}{2021}), \bibinfo{pages}{64--71}.
\newblock


\bibitem[\protect\citeauthoryear{Wu and Zhang}{Wu and Zhang}{2021}]%
        {wu2021tightening}
\bibfield{author}{\bibinfo{person}{Yiting Wu} {and} \bibinfo{person}{Min
  Zhang}.} \bibinfo{year}{2021}\natexlab{}.
\newblock \showarticletitle{Tightening Robustness Verification of Convolutional
  Neural Networks with Fine-Grained Linear Approximation}. In
  \bibinfo{booktitle}{\emph{AAAI Conference on Artificial Intelligence
  (AAAI)}}. \bibinfo{pages}{11674--11681}.
\newblock


\bibitem[\protect\citeauthoryear{Xiang, Tran, and Johnson}{Xiang
  et~al\mbox{.}}{2018}]%
        {xiang2018output}
\bibfield{author}{\bibinfo{person}{Weiming Xiang}, \bibinfo{person}{Hoang-Dung
  Tran}, {and} \bibinfo{person}{Taylor~T Johnson}.}
  \bibinfo{year}{2018}\natexlab{}.
\newblock \showarticletitle{Output reachable set estimation and verification
  for multilayer neural networks}.
\newblock \bibinfo{journal}{\emph{IEEE transactions on neural networks and
  learning systems}} \bibinfo{volume}{29}, \bibinfo{number}{11}
  (\bibinfo{year}{2018}), \bibinfo{pages}{5777--5783}.
\newblock


\bibitem[\protect\citeauthoryear{Xiao, Rasul, and Vollgraf}{Xiao
  et~al\mbox{.}}{2017}]%
        {xiao2017/online}
\bibfield{author}{\bibinfo{person}{Han Xiao}, \bibinfo{person}{Kashif Rasul},
  {and} \bibinfo{person}{Roland Vollgraf}.} \bibinfo{year}{2017}\natexlab{}.
\newblock \showarticletitle{Fashion-MNIST: a Novel Image Dataset for
  Benchmarking Machine Learning Algorithms}.
\newblock \bibinfo{journal}{\emph{CoRR}}  \bibinfo{volume}{abs/1708.07747}
  (\bibinfo{year}{2017}).
\newblock


\bibitem[\protect\citeauthoryear{Yan, Chen, Zhang, Tan, Wang, and Wang}{Yan
  et~al\mbox{.}}{2021}]%
        {YanCZTWW21}
\bibfield{author}{\bibinfo{person}{Ming Yan}, \bibinfo{person}{Junjie Chen},
  \bibinfo{person}{Xiangyu Zhang}, \bibinfo{person}{Lin Tan},
  \bibinfo{person}{Gan Wang}, {and} \bibinfo{person}{Zan Wang}.}
  \bibinfo{year}{2021}\natexlab{}.
\newblock \showarticletitle{Exposing numerical bugs in deep learning via
  gradient back-propagation}. In \bibinfo{booktitle}{\emph{29th {ACM} Joint
  European Software Engineering Conference and Symposium on the Foundations of
  Software Engineering (ESEC/FSE)}}. \bibinfo{publisher}{{ACM}},
  \bibinfo{pages}{627--638}.
\newblock


\bibitem[\protect\citeauthoryear{Zhang, Weng, Chen, Hsieh, and Daniel}{Zhang
  et~al\mbox{.}}{2018}]%
        {zhang2018efficient}
\bibfield{author}{\bibinfo{person}{Huan Zhang}, \bibinfo{person}{Tsui-Wei
  Weng}, \bibinfo{person}{Pin-Yu Chen}, \bibinfo{person}{Cho-Jui Hsieh}, {and}
  \bibinfo{person}{Luca Daniel}.} \bibinfo{year}{2018}\natexlab{}.
\newblock \showarticletitle{Efficient Neural Network Robustness Certification
  with General Activation Functions}. In \bibinfo{booktitle}{\emph{Annual
  Conference on Neural Information Processing Systems (NeurIPS)}}.
  \bibinfo{pages}{4944--4953}.
\newblock


\end{thebibliography}
\end{multicols}

\appendix
\clearpage 
\twocolumn
\begin{figure*}	
		\begin{subfigure}{0.48\linewidth}
		\includegraphics[width=0.9\textwidth]{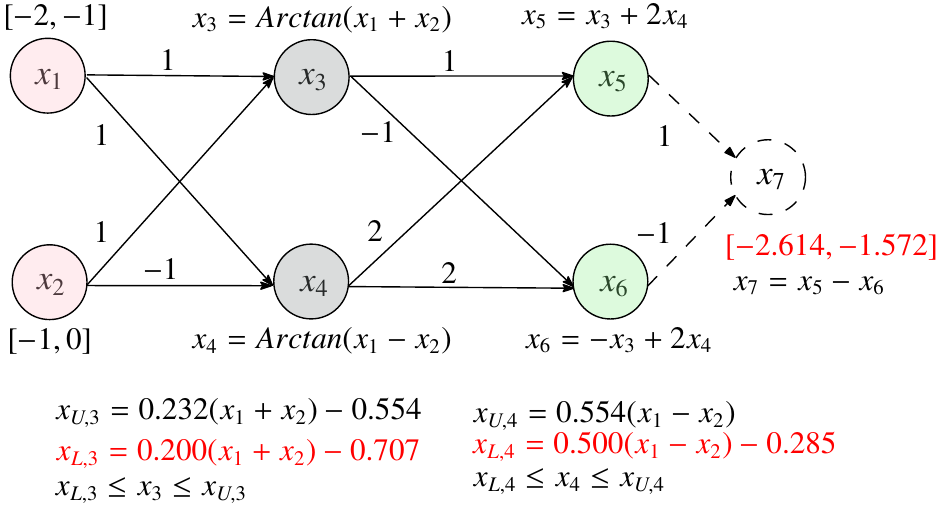}
		\caption{Two lower bounds (red) closer to the Arctan function.}
		\end{subfigure}
		\begin{subfigure}{0.48\linewidth}
		\includegraphics[width=.9\textwidth]{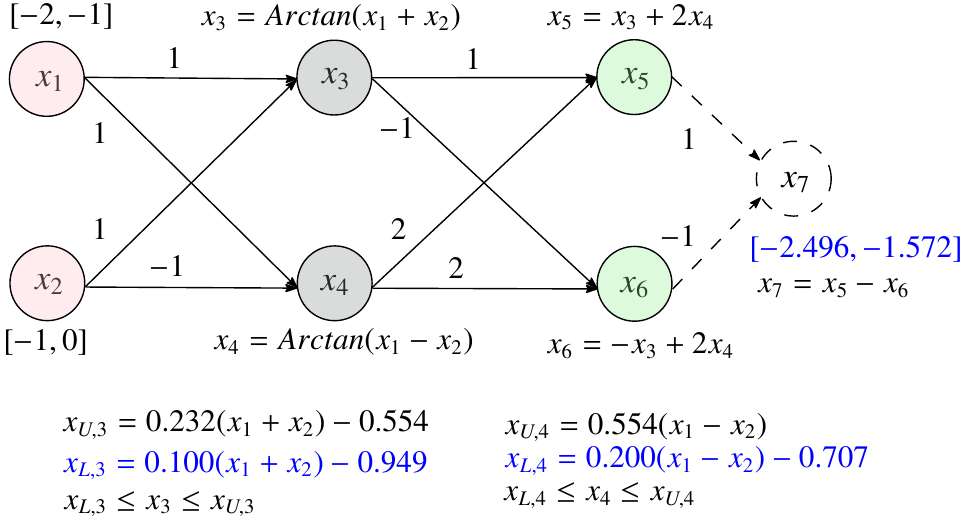}
		\caption{Two lower bounds (blue) farther from the Arctan function.}
	\end{subfigure}
	\caption{Output ranges calculated by different linear approximations on the neural networks with Arctan.} 
	\label{cal_interval_arctan}
\end{figure*}

\section{A Counterexample for Arctan}
This section shows a counterexample for Arctan, where a tighter linear bound line for the activation function actually produces a less precise verification result.

The counterexample is shown in Figure  \ref{cal_interval_arctan}. For simplicity, weights of edges are integers and biases are ignored. Let $x_{1}$ and $x_{2}$ represent the possible input values in $[-2, -1]$ and $[-1, 0]$, respectively. In the output layer, if the value of $x_{5}$ is always greater than the one of $x_{6}$, we can claim that all  $x_{1}$ in $[-2,-1]$ and $x_{2}$ in $[-1, 0]$ are classified to the label of $x_5$. We introduce an auxiliary node $x_{7}$ to represent $x_{5} - x_{6}$.

The two linear approximations in Figure \ref{cal_interval_arctan} (a) and (b) share the same upper bounds. 
The lower bounds of $x_3,x_4$ (the red in (a)) are closer to the Arctan function than the ones (the blue in (b)). However, the output range $[-2.496, -1.572]$ of $x_7$ computed by the right  approximation is more precise than  $[-2.614, -1.572]$ computed by the left approximation.  

\section{Proof of Theorem 2}
For $k$-layer neural network $f:\mathbb{R}^n \to \mathbb{R}^m$ satisfying the conditions in Theorem \ref{neuron->network}, we continue to use the symbol definition proposed. We assume that the upper and lower bounds of $\phi^{t}_{r}(x)$ are $u_r^{t}$ and $l_r^{t}$, \textit{i.e.}, $l_r^{t}\le \phi_r^{t}(x) \le u_r^{t}$. We consider the linear approximation of $\sigma(\phi_r^{t}(x))$ between  $[l_r^{t},u_r^{t}]$, where $\sigma(x)$ is a Sigmoid-like function. According to Definition \ref{def:neuronwise}, we denote the upper and lower linear bounds of $\sigma(\phi_r^{t}(x))$  obtained by our approach as $h_{U,r}^t(x)$ and $h_{L,r}^t(x)$, respectively.

First, we prove Lemma \ref{2con->mono}, i.e., the two conditions in the Theorem \ref{neuron->network} guarantee that the network is monotonous.

\begin{linenomath*}
\begin{proof}
    For any selected item of input $x_i$, items in the $i$-th column of $W^1$ are all positive or all negative. Suppose that they are all positive, which means $\forall r \in \mathbb{Z}, 1 \le r \le n^{1}, w_{r,i}^1>0$. First we consider $t=1$, for any value $x'_{i}>x''_{i}$ of $x_i$, $\phi_{r}^{1}(x'_{i}) = w_{r,i}^1 x'_{i} > w_{r,i}^1 x''_{i} = \phi_{r}^{1}(x''_{i})$.\\
    Suppose that for any $t\le T, 1\le T$, $\phi_{r}^{t}(x'_{i})> \phi_{r}^{t}(x''_{i})$ holds. Then we consider layer $T+1$, Since $\sigma$ is monotonously increasing, and all items in $W^{T+1}$ are non-negative, we have:
    \begin{align}
         \phi_{r}^{T+1}(x'_{i}) &= \textstyle\sum\limits_p w_{r,p}^{T+1} \sigma(\phi_p^T(x'_{i})) \\
         &>\textstyle\sum\limits_p w_{r,p}^{T+1} \sigma(\phi_p^T(x''_{i})) \\
         & = \phi_{r}^{T+1}(x''_{i})
    \end{align}
    In summary, by mathematical induction, $\forall r,t$, $\phi_{r}^{t}(x_i)$ is monotonously increasing. 
    
    Similarly, if the items in the $i$-th column of $W^1$ are all negative, we can deduce that $\forall r,t$, $\phi_{r}^{t}(x_i)$ is monotonously decreasing. The network satisfies Definition \ref{stren_condi}.
\end{proof}
\end{linenomath*}

\begin{linenomath*}
\begin{lemma}\label{lem:Exchangeable}
Given a $k$-layer neural network $f$ satisfying Definition \ref{stren_condi} and its input $x=[x_1,...,x_n]$, $\forall t \in \mathbb{Z}, 1 \le t \le k$, the $j$-th item in layer $t+1$ satisfies: 
    \begin{align*}
        \min \limits_{x \in \mathbb{B}_\infty(x_0, \epsilon)}\phi_j^{t+1}(x) 
        = \textstyle\sum\limits_r w_{j,r}^{t+1} \min\limits_{x \in \mathbb{B}_\infty(x_0, \epsilon)}\sigma(\phi_r^t(x)) 
    \end{align*}
\end{lemma}
\begin{proof}
    Suppose that $f$ is network-wise monotonously increasing. $\forall$ selected $i$, from Definition \ref{stren_condi}, $\forall t$, $\phi_r^t(x_i)$ and $\phi_r^{t+1}(x_i)$ are both monotonously increasing. Suppose that $x_i \in [l_i^0,u_i^0]$, $x_{min} := [l_1^0,...,l_n^0]$, $x_{max} := [u_1^0,...,u_n^0]$, then:
        
            \begin{align*}
                \min \limits_{x \in \mathbb{B}_\infty(x_0, \epsilon)}\phi_j^{t+1}(x)
                & = \phi_j^{t+1}(x_{min}) \\
                & = \textstyle\sum\limits_r w_{j,r}^{t+1} \sigma(\phi_r^t(x_{min}))\\
                & = \textstyle\sum\limits_r w_{j,r}^{t+1} \sigma(\min\limits_{x \in \mathbb{B}_\infty(x_0, \epsilon)}\phi_r^t(x))\\
                & = \textstyle\sum\limits_r w_{j,r}^{t+1} \min\limits_{x \in \mathbb{B}_\infty(x_0, \epsilon)}\sigma(\phi_r^t(x)) \qedhere
            \end{align*}
\end{proof}
\end{linenomath*}

Next, we prove the optimality of our approximation approach from the perspective of a single neuron.

\begin{linenomath*}
\begin{lemma} \label{lem:compoundfunction}
Given function $f:\mathbb{R}^n \to \mathbb{R}$, let $f$'s definition domain be $D$ and its range be $[l,u]$, if the definition domain of function $g:\mathbb{R} \to \mathbb{R}$ is $[l,u]$, we can define compound function $g\circ f:\mathbb{R}^n \to \mathbb{R}$. Suppose that $g$ is a monotonically increasing function, then we have:
\begin{equation*}
    \min\limits_{x \in D} g\circ f (x)= g(l), \quad \max\limits_{x \in D} g\circ f (x)= g(u).
\end{equation*}
\end{lemma} 
\begin{proof}
Because $g$ is monotonically increasing, there is $\forall x \in [l,u], g(l) \le g(x) \le g(u)$. Then
\begin{align*}
    \min\limits_{x \in D} g\circ f(x) & = g(\min\limits_{x \in D} f(x)) = g(l), \text{and,}\\
    \max\limits_{x \in D} g\circ f (x) & = g(\max\limits_{x \in D} f(x)) = g(u) \qedhere
\end{align*}
\end{proof}
\end{linenomath*}

\begin{linenomath*}
\begin{theorem}\label{theorem:compoundfunction}
Given a neural network $f$, let $\phi^{t}:\mathbb{R}^n \to \mathbb{R}^{n^t}$ be the compound function of the first $t$ layers of $f$. The $r$-th row of $\phi^{t}$ can be denoted as $\phi_r^{t}:\mathbb{R}^n \to \mathbb{R}$. When the input $x$ is bounded in a norm-ball $\mathbb{B}_\infty(x_0,\epsilon)$, $\phi_r^{t}$ ranges in $[l_r^{t},u_r^{t}]$, which is the definition domain of $h_{U,r}^t$ and $h_{L,r}^t$. We have: 
    \begin{align*}
        \min\limits_{x \in \mathbb{B}_\infty(x_0,\epsilon)} & h_{L,r}^t\circ \phi_r^{t} (x)= h_{L,r}^t(l_r^t), \text{and},  \\
        \max\limits_{x \in \mathbb{B}_\infty(x_0,\epsilon)} & h_{U,r}^t\circ \phi_r^{t} (x)= h_{U,r}^t(u_r^t).
    \end{align*}
\end{theorem}
\end{linenomath*}
\begin{proof}
 According to Definition \ref{def:neuronwise}, $h_{U,r}^t(x) = \alpha_{U,r}^t x +\beta_{U,r}^t$, $h_{L,r}^t(x) = \alpha_{L,r}^t x +\beta_{L,r}^t$. In our approximation, 
 $h_{L,r}^t$ and $h_{L,r}^t$ are always monotonically increasing. By replacing $f$ and $g$ in Lemma \ref{lem:compoundfunction} with $\phi_r^t$ and $h_{L,r}^t$ (\textit{resp.} $h_{U,r}^t$), we can easily get the conclusion.
\end{proof}

Theorem \ref{theorem:compoundfunction} says that our approximation method does not  overestimate  the output range of current neuron. That is because  $h_{L,r}^t(l_r^t) = \sigma(l_r^t)$ and $h_{U,r}^t(u_r^t) = \sigma(u_r^t)$. That is, the output range of $\sigma(\phi_r^{t}(x))$ after the linear approximation is still  $[\sigma(l_r^t),\sigma(u_r^t)]$. 

When dealing with the lower bound of  $\phi_j^{t+1}(x)$ - $j$-th output of layer $t+1$ under the condition that $x\in \mathbb{B}_\infty(x_0, \epsilon)$, we have:

\begin{linenomath*}
\begin{align} 
        \min \limits_{x \in \mathbb{B}_\infty(x_0, \epsilon)}\phi_j^{t+1}(x) 
        & = \textstyle\sum\limits_r w_{j,r}^{t+1} \min\limits_{x \in \mathbb{B}_\infty(x_0, \epsilon)}\sigma(\phi_r^t(x))  \label{exchanged}\\ 
        & = \textstyle\sum\limits_r w_{j,r}^{t+1} \min\limits_{x \in \mathbb{B}_\infty(x_0, \epsilon)} h_{L,r}^t \circ \phi_r^t(x) \label{formula6}\\ 
        & = \textstyle\sum\limits_r w_{j,r}^{t+1} h_{L,r}^t(l_r^t)  \label{minimum}
\end{align}
\end{linenomath*}

where (\ref{exchanged}) is the conclusion of Lemma \ref{lem:Exchangeable}, and (\ref{minimum}) is derived from  (\ref{formula6}) according to Theorem \ref{theorem:compoundfunction}. 

Likewise, we have:

\begin{linenomath*}
\begin{equation}\label{maximum}
    \max\limits_{x \in \mathbb{B}_\infty(x_0, \epsilon)}\phi_j^{t+1}(x) = \textstyle\sum w_{j,r}^t h_{U,r}^t(u_r^t). 
\end{equation}
\end{linenomath*}

\begin{theorem}\label{independent-optimal}
    For networks satisfying conditions in Theorem \ref{neuron->network}, our linear approximation guarantees that (\ref{minimum}) is the largest and (\ref{maximum}) is the smallest among all possible approximations.
\end{theorem}

\begin{proof}

    Given an arbitrary other lower and upper linear bounds $\hat{h}_{L,r}^t$ and $\hat{h}_{U,r}^t$, there are $\hat{h}_{L,r}^t(x) \le \sigma(x)$ and $\hat{h}_{U,r}^t(x) \ge \sigma(x)$ for all $x \in [l_r^t,u_r^t]$, which implies that  $\hat{h}_{L,r}^t(l_r^t) \le \sigma(l_r^t)$ and $\hat{h}_{U,r}^t(u_r^t) \ge \sigma(u_r^t)$. According to our linear approximation:
    \begin{linenomath*}	
    \begin{align*}
    h_{L,r}^t(l_r^t) &= \sigma(l_r^t) \ge \hat{h}_{L,r}^t(l_r^t), \text{and,}\\
    h_{U,r}^t(u_r^t) &= \sigma(u_r^t) \le \hat{h}_{U,r}^t(u_r^t).
    \end{align*}
\end{linenomath*}   
    As a result, for $w_{j,r}^t\ge 0$, we have the following formula hold:
    \begin{linenomath*}	
    \begin{align*}
        \textstyle\sum w_{j,r}^t h_{L,r}^t(l_r^t) \ge & \textstyle\sum w_{j,r}^t \hat{h}_{L,r}^t(l_r^t), \text{and,} \\
        \textstyle\sum w_{j,r}^t h_{U,r}^t(u_r^t) \le & \textstyle\sum w_{j,r}^t \hat{h}_{U,r}^t(u_r^t).
    \end{align*}
\end{linenomath*}   
    So our linear approximation obtains the supremum of (\ref{minimum}) and the infimum of (\ref{maximum}). 
\end{proof}

According to Theorem \ref{independent-optimal}, our approximation obtains the upper bound of $\min \limits_{x \in \mathbb{B}_\infty(x_0, \epsilon)}\phi_j^{t+1}(x)$. Next, we prove the network-wise tightness of our approach through Definition \ref{global_opt}. Let $(f_L,f_U)$ be our linear approximation of $f$ with $f_U$ and $f_L$  the upper and lower bounds, respectively. Suppose that there exists another linear approximation $(\hat{f}_L,\hat{f}_U)$ and label $s$ s.t. 
\begin{equation}
\min\limits_{x\in \mathbb{B}_\infty(x_0, \epsilon)}f_{L,s}(x) <  \min\limits_{x\in \mathbb{B}_\infty(x_0, \epsilon)}\hat{f}_{L,s}(x).\label{neq1}
\end{equation}
If the two approaches share the same bounds of nodes of ($k$-1)-th layer, then:
\begin{linenomath*}
\begin{align*}
     \min\limits_{x\in \mathbb{B}_\infty(x_0, \epsilon)}f_{L,s}(x) 
     &=  \min\limits_{x\in \mathbb{B}_\infty(x_0, \epsilon)}\phi_{L,s}^k(x)\\
     &= \textstyle\sum\limits_r w_{s,r}^{k} h_{L,r}^{k-1}(l_r^{k-1}) \\
     & \ge  \textstyle\sum\limits_r w_{s,r}^{k} \hat{h}_{L,r}^{k-1}(l_r^{k-1})
\end{align*}
\end{linenomath*}
which contradicts with (\ref{neq1}). So the two approaches cannot share the same bounds, and there must be $\hat{l}_r^{k-1}>l_r^{k-1}$. Otherwise,  $\hat{h}_{L,r}^{k-1}(\hat{l}_r^{k-1})\le \sigma(\hat{l}_r^{k-1}) < \sigma(l_r^{k-1}) \le h_{L,r}^{k-1}(l_r^{k-1})$,  meaning that (\ref{neq1}) cannot be achieved. We have:
\begin{linenomath*}
\begin{align*}
    l_r^{k-1} &= \min\limits_{x\in \mathbb{B}_\infty(x_0, \epsilon)}\phi_{L,r}^{k-1}(x) \\
    &= \textstyle\sum\limits_j w_{r,j}^{k-1} h_{L,j}^{k-2}(l_j^{k-2}) \\
    & < \hat{l}_r^{k-1} \\
    & =\min\limits_{x\in \mathbb{B}_\infty(x_0, \epsilon)}\hat{\phi}_{L,r}^{k-1}(x) \\
    & = \textstyle\sum\limits_j w_{r,j}^{k-1} \hat{h}_{L,j}^{k-2}(\hat{l}_j^{k-2})
\end{align*}
\end{linenomath*}
Similarly, there must be $\hat{l}_j^{k-2}>l_j^{k-2}$ and so on. In the end, we can deduce that $\forall i\in [1,n], \hat{l}_i^0>l_i^0$, which are the lower bounds of $x_i$. However, the input range of $x_i$ is the same for the two approaches, which is a contradiction.

To conclude, there does not exist a $(\hat{f}_L,\hat{f}_U)$ satisfying (\ref{neq1}). Theorem \ref{neuron->network} is proven.

\section{Additional Experimental Results}
This section presents the precision and efficiency comparison of \textsc{NeWise} with \textsc{DeepCert}, \textsc{VeriNet} and \textsc{RobustVerifier} on 27 Sigmoid, 30 Tanh, and 32 Arctan models with non-negative weights. In addition, we also demonstrated the effectiveness of Algorithm 1 on 40 1-hidden-layer models with mixed weights.

\begin{table*}[ht]
    \scriptsize 
    \def\arraystretch{1}
    \setlength{\tabcolsep}{4pt}
    \caption{Performance comparison on non-negative Sigmoid networks between \textsc{NeWise} (NW) and existing tools, \textsc{DeepCert} (DC), \textsc{VeriNet} (VN), and \textsc{RobustVerifier} (RV).
		$t\times n$ refers to an FNN with $t$ layers and $n$ neurons per layer. $\mbox{CNN}_{t-c}$ denotes a CNN with $t$ layers and $c$ filters of size 3$\times$3.}
    \vspace{-2mm}
    \centering
    \resizebox{\textwidth}{!}{
    \begin{tabular}{|c|c|r|r|r|r|r|r|r|r|r|r|r|r|r|r|r|r|}
        \hline 
        \multirow{3}{*}{\textbf{Dataset}} & \multirow{3}{*}{\textbf{Model}} & \multirow{3}{*}{\textbf{\#Neur.}} &   \multicolumn{14}{c|}{\textbf{Certified Lower Bound}} & \multirow{3}{*}{\textbf{Time (s)}}  \\
        \hhline{~~~--------------~}
        & & & \multicolumn{7}{c|}{\textbf{Average}} & \multicolumn{7}{c|}{\textbf{ Standard Deviation}} & \\
        \hhline{~~~--------------~}
        & & & \textsc{NW} & \textsc{DC} &  Impr. (\%) & \textsc{VN} &  Impr. (\%) & \textsc{RV} &  Impr. (\%) & \textsc{NW} & \textsc{DC} &  Impr. (\%) & \textsc{VN} &  Impr. (\%) & \textsc{RV} &  Impr. (\%) &  \\
        \hline
        \hline
        \multirow{11}{*}{MNIST} 
        & 5x100              & 510    & 0.0091 & 0.0071 & \cellcolor{atomictangerine} 28.15  $\uparrow$ & 0.0071 & \cellcolor{atomictangerine} 27.25 $\uparrow$ & 0.0064 & \cellcolor{atomictangerine} 40.90 $\uparrow$ & 0.0057 & 0.0042 & \cellcolor{atomictangerine} 37.11 $\uparrow$ & 0.0042 & \cellcolor{atomictangerine} 35.48 $\uparrow$ & 0.0034 & \cellcolor{atomictangerine} 69.35 $\uparrow$ &   4.39 $\pm$0.03    \\ 
        & 3x700              & 2,110  & 0.0037 & 0.0030 & \cellcolor{atomictangerine} 24.92  $\uparrow$ & 0.0030 & \cellcolor{atomictangerine} 22.85 $\uparrow$ & 0.0029 & \cellcolor{atomictangerine} 27.05 $\uparrow$ & 0.0018 & 0.0013 & \cellcolor{atomictangerine} 41.86 $\uparrow$ & 0.0014 & \cellcolor{atomictangerine} 34.56 $\uparrow$ & 0.0013 & \cellcolor{atomictangerine} 41.86 $\uparrow$ & 117.50 $\pm$0.22  \\ 
        & $\mbox{CNN}_{6-5}$ & 12,300 & 0.0968 & 0.0788 & \cellcolor{atomictangerine} 22.82  $\uparrow$ & 0.0778 & \cellcolor{atomictangerine} 24.37 $\uparrow$ & 0.0699 & \cellcolor{atomictangerine} 38.50 $\uparrow$ & 0.0372 & 0.0280 & \cellcolor{atomictangerine} 32.92 $\uparrow$ & 0.0276 & \cellcolor{atomictangerine} 35.09 $\uparrow$ & 0.0212 & \cellcolor{atomictangerine} 75.86 $\uparrow$ &   5.70 $\pm$0.42  \\ 
        & 3x50               & 160    & 0.0105 & 0.0088 & \cellcolor{apricot}         19.23  $\uparrow$ & 0.0088 & \cellcolor{apricot}         19.50 $\uparrow$ & 0.0080 & \cellcolor{atomictangerine} 31.42 $\uparrow$ & 0.0051 & 0.0038 & \cellcolor{atomictangerine} 32.72 $\uparrow$ & 0.0038 & \cellcolor{atomictangerine} 32.72 $\uparrow$ & 0.0029 & \cellcolor{atomictangerine} 71.86 $\uparrow$ &   0.14 $\pm$0.00  \\ 
        & 3x100              & 310    & 0.0139 & 0.0120 & \cellcolor{apricot}         15.46  $\uparrow$ & 0.0120 & \cellcolor{apricot}         15.56 $\uparrow$ & 0.0111 & \cellcolor{atomictangerine} 25.47 $\uparrow$ & 0.0071 & 0.0057 & \cellcolor{apricot}         24.82 $\uparrow$ & 0.0057 & \cellcolor{apricot}         23.30 $\uparrow$ & 0.0046 & \cellcolor{atomictangerine} 53.46 $\uparrow$ &   2.21 $\pm$0.02  \\ 
        & $\mbox{CNN}_{5-5}$ & 10,680 & 0.0801 & 0.0708 & \cellcolor{apricot}         13.14  $\uparrow$ & 0.0704 & \cellcolor{apricot}         13.75 $\uparrow$ & 0.0683 & \cellcolor{apricot}         17.30 $\uparrow$ & 0.0238 & 0.0200 & \cellcolor{apricot}         18.87 $\uparrow$ & 0.0198 & \cellcolor{apricot}         20.50 $\uparrow$ & 0.0180 & \cellcolor{apricot}         32.20 $\uparrow$ &   2.89 $\pm$0.31  \\ 
        & 3x200              & 610    & 0.0080 & 0.0071 & \cellcolor{apricot}         12.54  $\uparrow$ & 0.0071 & \cellcolor{apricot}         12.85 $\uparrow$ & 0.0068 & \cellcolor{apricot}         17.16 $\uparrow$ & 0.0046 & 0.0037 & \cellcolor{apricot}         26.43 $\uparrow$ & 0.0037 & \cellcolor{apricot}         25.41 $\uparrow$ & 0.0034 & \cellcolor{apricot}         37.28 $\uparrow$ &  10.81 $\pm$0.07 \\
        & 3x400              & 1,210  & 0.0061 & 0.0056 & \cellcolor{antiquewhite}     9.66  $\uparrow$ & 0.0056 & \cellcolor{antiquewhite}     9.86 $\uparrow$ & 0.0054 & \cellcolor{apricot}         12.89 $\uparrow$ & 0.0035 & 0.0030 & \cellcolor{apricot}         16.78 $\uparrow$ & 0.0030 & \cellcolor{apricot}         16.39 $\uparrow$ & 0.0028 & \cellcolor{apricot}         26.09 $\uparrow$ &  39.52 $\pm$0.09 \\
        & $\mbox{CNN}_{3-2}$ & 2,514  & 0.0521 & 0.0483 & \cellcolor{antiquewhite}     7.82  $\uparrow$ & 0.0483 & \cellcolor{antiquewhite}     7.94 $\uparrow$ & 0.0478 & \cellcolor{antiquewhite}     8.88 $\uparrow$ & 0.0180 & 0.0161 & \cellcolor{apricot}         12.13 $\uparrow$ & 0.0160 & \cellcolor{antiquewhite}    12.41 $\uparrow$ & 0.0156 & \cellcolor{apricot}         15.44 $\uparrow$ &   0.17 $\pm$0.04  \\ 
        & $\mbox{CNN}_{4-5}$ & 8,680  & 0.0505 & 0.0473 & \cellcolor{antiquewhite}     6.68  $\uparrow$ & 0.0471 & \cellcolor{antiquewhite}     7.26 $\uparrow$ & 0.0464 & \cellcolor{antiquewhite}     8.81 $\uparrow$ & 0.0207 & 0.0186 & \cellcolor{apricot}         11.26 $\uparrow$ & 0.0183 & \cellcolor{antiquewhite}    12.84 $\uparrow$ & 0.0175 & \cellcolor{apricot}         17.87 $\uparrow$ &   1.17 $\pm$0.20  \\ 
        & $\mbox{CNN}_{3-4}$ & 5,018  & 0.0448 & 0.0422 & \cellcolor{antiquewhite}     6.09  $\uparrow$ & 0.0421 & \cellcolor{antiquewhite}     6.24 $\uparrow$ & 0.0418 & \cellcolor{antiquewhite}     6.98 $\uparrow$ & 0.0156 & 0.0142 & \cellcolor{antiquewhite}     9.71 $\uparrow$ & 0.0141 & \cellcolor{antiquewhite}    10.18 $\uparrow$ & 0.0138 & \cellcolor{antiquewhite}    12.56 $\uparrow$ &   0.30 $\pm$0.08  \\ 
        \hline
        \multirow{9}{*}{\makecell{Fashion \\ MNIST} }
        & 4x100              & 410    & 0.0312 & 0.0188 & \cellcolor{atomictangerine} 65.48  $\uparrow$ & 0.0194 & \cellcolor{atomictangerine} 60.62 $\uparrow$ & 0.0159 & \cellcolor{atomictangerine} 96.22 $\uparrow$ & 0.0403 & 0.0210 & \cellcolor{atomictangerine} 92.28 $\uparrow$ & 0.0220 & \cellcolor{atomictangerine} 83.20 $\uparrow$ & 0.0176 & \cellcolor{atomictangerine} 129.47 $\uparrow$ & 3.34 $\pm$0.02   \\ 
        & 3x100              & 310    & 0.0326 & 0.0263 & \cellcolor{apricot}         24.02  $\uparrow$ & 0.0270 & \cellcolor{atomictangerine} 21.03 $\uparrow$ & 0.0238 & \cellcolor{atomictangerine} 36.81 $\uparrow$ & 0.0335 & 0.0262 & \cellcolor{atomictangerine} 27.67 $\uparrow$ & 0.0282 & \cellcolor{apricot}         18.92 $\uparrow$ & 0.0234 & \cellcolor{atomictangerine}  43.22 $\uparrow$ & 2.24 $\pm$0.02 \\ 
        & 2x100              & 210    & 0.0306 & 0.0250 & \cellcolor{apricot}         22.49  $\uparrow$ & 0.0254 & \cellcolor{atomictangerine} 20.51 $\uparrow$ & 0.0230 & \cellcolor{atomictangerine} 33.09 $\uparrow$ & 0.0286 & 0.0211 & \cellcolor{atomictangerine} 36.04 $\uparrow$ & 0.0228 & \cellcolor{apricot}         25.88 $\uparrow$ & 0.0194 & \cellcolor{atomictangerine}  47.76 $\uparrow$ & 1.13 $\pm$0.01 \\
        & 3x200              & 610    & 0.0223 & 0.0184 & \cellcolor{apricot}         21.80  $\uparrow$ & 0.0187 & \cellcolor{apricot}         19.45 $\uparrow$ & 0.0170 & \cellcolor{atomictangerine} 31.70 $\uparrow$ & 0.0220 & 0.0159 & \cellcolor{atomictangerine} 38.66 $\uparrow$ & 0.0170 & \cellcolor{apricot}         29.38 $\uparrow$ & 0.0143 & \cellcolor{atomictangerine}  54.42 $\uparrow$ & 9.20 $\pm$2.00 \\
        & 2x200              & 410    & 0.0263 & 0.0220 & \cellcolor{apricot}         19.52  $\uparrow$ & 0.0223 & \cellcolor{apricot}         17.86 $\uparrow$ & 0.0204 & \cellcolor{atomictangerine} 28.77 $\uparrow$ & 0.0279 & 0.0200 & \cellcolor{atomictangerine} 39.55 $\uparrow$ & 0.0211 & \cellcolor{apricot}         31.96 $\uparrow$ & 0.0176 & \cellcolor{atomictangerine}  58.76 $\uparrow$ & 3.60 $\pm$0.02 \\
        & $\mbox{CNN}_{5-5}$ & 10,680 & 0.1303 & 0.1155 & \cellcolor{apricot}         12.81  $\uparrow$ & 0.1151 & \cellcolor{apricot}         13.22 $\uparrow$ & 0.1088 & \cellcolor{apricot}         19.72 $\uparrow$ & 0.0830 & 0.0714 & \cellcolor{apricot}         16.23 $\uparrow$ & 0.0721 & \cellcolor{apricot}         15.10 $\uparrow$ & 0.0636 & \cellcolor{apricot}          30.51 $\uparrow$ & 2.90 $\pm$0.33 \\ 
        & $\mbox{CNN}_{3-2}$ & 2,514  & 0.0790 & 0.0713 & \cellcolor{apricot}         10.79  $\uparrow$ & 0.0713 & \cellcolor{apricot}         10.74 $\uparrow$ & 0.0695 & \cellcolor{apricot}         13.68 $\uparrow$ & 0.0497 & 0.0416 & \cellcolor{apricot}         19.55 $\uparrow$ & 0.0418 & \cellcolor{apricot}         19.06 $\uparrow$ & 0.0386 & \cellcolor{apricot}          28.98 $\uparrow$ & 0.17 $\pm$0.04 \\ 
        & $\mbox{CNN}_{4-5}$ & 8,680  & 0.0959 & 0.0868 & \cellcolor{apricot}         10.40  $\uparrow$ & 0.0864 & \cellcolor{apricot}         10.90 $\uparrow$ & 0.0839 & \cellcolor{apricot}         14.19 $\uparrow$ & 0.0561 & 0.0486 & \cellcolor{apricot}         15.51 $\uparrow$ & 0.0482 & \cellcolor{apricot}         16.52 $\uparrow$ & 0.0453 & \cellcolor{apricot}          24.03 $\uparrow$ & 1.18 $\pm$0.21 \\ 
        & $\mbox{CNN}_{3-4}$ & 5,018  & 0.0747 & 0.0694 & \cellcolor{antiquewhite}     7.52  $\uparrow$ & 0.0693 & \cellcolor{antiquewhite}     7.72 $\uparrow$ & 0.0681 & \cellcolor{antiquewhite}     9.70 $\uparrow$ & 0.0465 & 0.0410 & \cellcolor{antiquewhite}    13.32 $\uparrow$ & 0.0409 & \cellcolor{antiquewhite}    13.59 $\uparrow$ & 0.0391 & \cellcolor{antiquewhite}     18.85 $\uparrow$ & 0.30 $\pm$0.09 \\ 
        \hline
        \multirow{7}{*}{CIFAR10} 
        & 9x100               & 910   & 0.0315 & 0.0211 & \cellcolor{atomictangerine} 49.03  $\uparrow$ & 0.0214 & \cellcolor{atomictangerine} 46.94 $\uparrow$ & 0.0192 & \cellcolor{atomictangerine} 63.58 $\uparrow$ & 0.0280 & 0.0183 & \cellcolor{atomictangerine} 52.70 $\uparrow$ & 0.0186 & \cellcolor{atomictangerine} 50.32 $\uparrow$ & 0.0133 & \cellcolor{atomictangerine} 110.07 $\uparrow$ & 4.93 $\pm$0.02  \\ 
        & 6x100               & 610   & 0.0221 & 0.0174 & \cellcolor{atomictangerine} 27.08  $\uparrow$ & 0.0176 & \cellcolor{atomictangerine} 26.14 $\uparrow$ & 0.0170 & \cellcolor{atomictangerine} 30.22 $\uparrow$ & 0.0165 & 0.0118 & \cellcolor{atomictangerine} 40.05 $\uparrow$ & 0.0120 & \cellcolor{atomictangerine} 37.82 $\uparrow$ & 0.0111 & \cellcolor{atomictangerine}  48.24 $\uparrow$ & 3.04 $\pm$0.02  \\ 
        & 5x100               & 510   & 0.0200 & 0.0167 & \cellcolor{apricot}         19.76  $\uparrow$ & 0.0167 & \cellcolor{apricot}         19.47 $\uparrow$ & 0.0163 & \cellcolor{atomictangerine} 22.40 $\uparrow$ & 0.0137 & 0.0104 & \cellcolor{apricot}         31.80 $\uparrow$ & 0.0104 & \cellcolor{apricot}         31.42 $\uparrow$ & 0.0099 & \cellcolor{apricot}          38.45 $\uparrow$ & 2.44 $\pm$0.01  \\ 
        & 3x50                & 160   & 0.0206 & 0.0178 & \cellcolor{apricot}         15.43  $\uparrow$ & 0.0179 & \cellcolor{apricot}         14.66 $\uparrow$ & 0.0176 & \cellcolor{apricot}         16.88 $\uparrow$ & 0.0144 & 0.0113 & \cellcolor{apricot}         27.57 $\uparrow$ & 0.0115 & \cellcolor{apricot}         25.24 $\uparrow$ & 0.0110 & \cellcolor{apricot}          31.30 $\uparrow$ & 0.43 $\pm$0.00  \\ 
        & 4x100               & 410   & 0.0161 & 0.0140 & \cellcolor{apricot}         15.23  $\uparrow$ & 0.0140 & \cellcolor{apricot}         14.81 $\uparrow$ & 0.0138 & \cellcolor{apricot}         16.56 $\uparrow$ & 0.0111 & 0.0089 & \cellcolor{apricot}         24.61 $\uparrow$ & 0.0090 & \cellcolor{apricot}         23.63 $\uparrow$ & 0.0087 & \cellcolor{apricot}          27.62 $\uparrow$ & 1.87 $\pm$0.03  \\ 
        & $\mbox{CNN}_{3-4}$  & 6,746 & 0.0187 & 0.0181 & \cellcolor{antiquewhite}     3.38  $\uparrow$ & 0.0181 & \cellcolor{antiquewhite}     3.32 $\uparrow$ & 0.0181 & \cellcolor{antiquewhite}     3.43 $\uparrow$ & 0.0109 & 0.0103 & \cellcolor{antiquewhite}     5.93 $\uparrow$ & 0.0103 & \cellcolor{antiquewhite}     5.83 $\uparrow$ & 0.0103 & \cellcolor{antiquewhite}      6.13 $\uparrow$ & 0.56 $\pm$0.08 \\ 
        & $\mbox{CNN}_{3-2}$  & 3,378 & 0.0185 & 0.0180 & \cellcolor{antiquewhite}     2.49  $\uparrow$ & 0.0180 & \cellcolor{antiquewhite}     2.55 $\uparrow$ & 0.0180 & \cellcolor{antiquewhite}     2.67 $\uparrow$ & 0.0125 & 0.0120 & \cellcolor{antiquewhite}     4.34 $\uparrow$ & 0.0120 & \cellcolor{antiquewhite}     4.34 $\uparrow$ & 0.0120 & \cellcolor{antiquewhite}      4.60 $\uparrow$ & 0.30 $\pm$0.05 \\ 
        \hline
    \end{tabular}}
    \label{all resutls on Sigmoid models}
\end{table*}

\begin{table*}[ht]
    \scriptsize 
    \def\arraystretch{1}
    \setlength{\tabcolsep}{4pt}
    \caption{Performance comparison of \textsc{NeWise} (NW) with \textsc{DeepCert} (DC), \textsc{VeriNet} (VN), and \textsc{RobustVerifier} (RV) on non-negative Tanh networks. 
		$t\times n$ refers to an FNN with $t$ layers and $n$ neurons per layer. $\mbox{CNN}_{t-c}$ denotes a CNN with $t$ layers and $c$ filters of size 3$\times$3.}
    \vspace{-2mm}
    \centering
    \begin{tabular}{|c|c|r|r|r|r|r|r|r|r|r|r|r|r|r|r|r|r|}
        \hline 
        \multirow{3}{*}{\textbf{Dataset}} & \multirow{3}{*}{\textbf{Model}} & \multirow{3}{*}{\textbf{\#Neur.}} &   \multicolumn{14}{c|}{\textbf{Certified Lower Bound}} & \multirow{3}{*}{\textbf{Time (s)}}  \\
        \hhline{~~~--------------~}
        & & & \multicolumn{7}{c|}{\textbf{Average}} & \multicolumn{7}{c|}{\textbf{ Standard Deviation}} & \\
        \hhline{~~~--------------~}
        & & & \textsc{NW} & \textsc{DC} &  Impr. (\%) & \textsc{VN} &  Impr. (\%) & \textsc{RV} &  Impr. (\%) & \textsc{NW} & \textsc{DC} &  Impr. (\%) & \textsc{VN} &  Impr. (\%) & \textsc{RV} &  Impr. (\%) &  \\
        \hline
        \hline
        \multirow{11}{*}{MNIST}
        & 5x100               & 510    & 0.0032 & 0.0012 & \cellcolor{atomictangerine} 173.28 $\uparrow$ & 0.0013 & \cellcolor{atomictangerine} 145.74 $\uparrow$ & 0.0010 & \cellcolor{atomictangerine} 207.77 $\uparrow$ & 0.0018 & 0.0005 & \cellcolor{atomictangerine} 233.96 $\uparrow$ & 0.0006 & \cellcolor{atomictangerine} 195.00 $\uparrow$ & 0.0006 & \cellcolor{atomictangerine} 216.07 $\uparrow$ &   4.36 $\pm$ 0.06  \\ 
        & 3x700               & 2,110  & 0.0044 & 0.0016 & \cellcolor{atomictangerine} 170.99 $\uparrow$ & 0.0018 & \cellcolor{atomictangerine} 141.21 $\uparrow$ & 0.0016 & \cellcolor{atomictangerine} 174.38 $\uparrow$ & 0.0027 & 0.0008 & \cellcolor{atomictangerine} 251.28 $\uparrow$ & 0.0009 & \cellcolor{atomictangerine} 191.49 $\uparrow$ & 0.0009 & \cellcolor{atomictangerine} 191.49 $\uparrow$ & 117.43 $\pm$ 0.14  \\ 
        & 3x400               & 1,210  & 0.0038 & 0.0017 & \cellcolor{atomictangerine} 119.65 $\uparrow$ & 0.0019 & \cellcolor{apricot}          95.88 $\uparrow$ & 0.0016 & \cellcolor{atomictangerine} 130.30 $\uparrow$ & 0.0018 & 0.0007 & \cellcolor{atomictangerine} 166.67 $\uparrow$ & 0.0008 & \cellcolor{atomictangerine} 128.57 $\uparrow$ & 0.0007 & \cellcolor{atomictangerine} 137.84 $\uparrow$ &  39.45 $\pm$ 0.14  \\ 
        & $\mbox{CNN}_{6-5}$  & 12,300 & 0.0601 & 0.0358 & \cellcolor{apricot}          67.98 $\uparrow$ & 0.0381 & \cellcolor{apricot}          57.67 $\uparrow$ & 0.0303 & \cellcolor{apricot}          98.61 $\uparrow$ & 0.0243 & 0.0115 & \cellcolor{atomictangerine} 111.84 $\uparrow$ & 0.0131 & \cellcolor{apricot}          85.94 $\uparrow$ & 0.0087 & \cellcolor{atomictangerine} 179.13 $\uparrow$ &   5.74 $\pm$ 0.46  \\ 
        & 3x50                & 160    & 0.0027 & 0.0023 & \cellcolor{antiquewhite}     18.78 $\uparrow$ & 0.0023 & \cellcolor{antiquewhite}     17.24 $\uparrow$ & 0.0022 & \cellcolor{apricot}          25.35 $\uparrow$ & 0.0012 & 0.0009 & \cellcolor{apricot}          29.79 $\uparrow$ & 0.0010 & \cellcolor{apricot}          27.08 $\uparrow$ & 0.0008 & \cellcolor{apricot}          45.24 $\uparrow$ &   0.14 $\pm$ 0.00  \\ 
        & 3x100               & 310    & 0.0033 & 0.0029 & \cellcolor{antiquewhite}     15.97 $\uparrow$ & 0.0029 & \cellcolor{antiquewhite}     14.78 $\uparrow$ & 0.0027 & \cellcolor{apricot}          21.45 $\uparrow$ & 0.0014 & 0.0011 & \cellcolor{apricot}          23.85 $\uparrow$ & 0.0011 & \cellcolor{antiquewhite}     19.47 $\uparrow$ & 0.0010 & \cellcolor{apricot}          33.66 $\uparrow$ &   2.17 $\pm$ 0.03  \\ 
        & 3x200               & 610    & 0.0031 & 0.0027 & \cellcolor{antiquewhite}     15.41 $\uparrow$ & 0.0027 & \cellcolor{antiquewhite}     14.13 $\uparrow$ & 0.0026 & \cellcolor{antiquewhite}     18.99 $\uparrow$ & 0.0013 & 0.0011 & \cellcolor{apricot}          26.42 $\uparrow$ & 0.0011 & \cellcolor{antiquewhite}     19.64 $\uparrow$ & 0.0010 & \cellcolor{apricot}          31.37 $\uparrow$ &   7.54 $\pm$ 1.00  \\ 
        & $\mbox{CNN}_{3-4}$  & 5,018  & 0.0234 & 0.0208 & \cellcolor{antiquewhite}     12.48 $\uparrow$ & 0.0213 & \cellcolor{antiquewhite}      9.94 $\uparrow$ & 0.0209 & \cellcolor{antiquewhite}     12.21 $\uparrow$ & 0.0067 & 0.0055 & \cellcolor{apricot}          21.05 $\uparrow$ & 0.0058 & \cellcolor{antiquewhite}     14.80 $\uparrow$ & 0.0055 & \cellcolor{antiquewhite}     20.83 $\uparrow$ &   0.33 $\pm$ 0.09  \\ 
        & $\mbox{CNN}_{5-5}$  & 10,680 & 0.0329 & 0.0294 & \cellcolor{antiquewhite}     11.92 $\uparrow$ & 0.0296 & \cellcolor{antiquewhite}     10.93 $\uparrow$ & 0.0290 & \cellcolor{antiquewhite}     13.42 $\uparrow$ & 0.0108 & 0.0090 & \cellcolor{apricot}          20.24 $\uparrow$ & 0.0092 & \cellcolor{antiquewhite}     17.50 $\uparrow$ & 0.0087 & \cellcolor{antiquewhite}     24.54 $\uparrow$ &   2.91 $\pm$ 0.31  \\ 
        & $\mbox{CNN}_{4-5}$  & 8,680  & 0.0272 & 0.0241 & \cellcolor{antiquewhite}     13.08 $\uparrow$ & 0.0244 & \cellcolor{antiquewhite}     11.51 $\uparrow$ & 0.0239 & \cellcolor{antiquewhite}     14.03 $\uparrow$ & 0.0074 & 0.0061 & \cellcolor{apricot}          21.21 $\uparrow$ & 0.0063 & \cellcolor{antiquewhite}     17.19 $\uparrow$ & 0.0060 & \cellcolor{antiquewhite}     23.63 $\uparrow$ &   1.19 $\pm$ 0.19  \\ 
        & $\mbox{CNN}_{3-2}$  & 2,514  & 0.0327 & 0.0289 & \cellcolor{antiquewhite}     13.08 $\uparrow$ & 0.0294 & \cellcolor{antiquewhite}     11.35 $\uparrow$ & 0.0288 & \cellcolor{antiquewhite}     13.63 $\uparrow$ & 0.0100 & 0.0082 & \cellcolor{apricot}          21.90 $\uparrow$ & 0.0085 & \cellcolor{antiquewhite}     17.33 $\uparrow$ & 0.0081 & \cellcolor{antiquewhite}     23.10 $\uparrow$ &   0.18 $\pm$ 0.04  \\ 
        \hline
        \multirow{9}{*}{\makecell{Fashion\\MNIST}}
        & 3x100              & 310    & 0.0153 & 0.0114 & \cellcolor{atomictangerine} 34.39 $\uparrow$ & 0.0117 & \cellcolor{atomictangerine} 30.72 $\uparrow$ & 0.0103 & \cellcolor{atomictangerine} 48.02 $\uparrow$ & 0.0156 & 0.0105 & \cellcolor{atomictangerine} 48.81 $\uparrow$ & 0.0110 & \cellcolor{atomictangerine} 42.04 $\uparrow$ & 0.0091 & \cellcolor{atomictangerine} 70.79 $\uparrow$ &  2.20 $\pm$ 0.01 \\
        & 2x200              & 410    & 0.0099 & 0.0082 & \cellcolor{atomictangerine} 21.42 $\uparrow$ & 0.0083 & \cellcolor{atomictangerine} 19.23 $\uparrow$ & 0.0079 & \cellcolor{apricot}         25.57 $\uparrow$ & 0.0102 & 0.0076 & \cellcolor{atomictangerine} 33.38 $\uparrow$ & 0.0080 & \cellcolor{atomictangerine} 27.06 $\uparrow$ & 0.0076 & \cellcolor{antiquewhite}    33.55 $\uparrow$ &  5.42 $\pm$ 0.03 \\
        & $\mbox{CNN}_{6-5}$ & 12,300 & 0.0408 & 0.0334 & \cellcolor{atomictangerine} 22.36 $\uparrow$ & 0.0343 & \cellcolor{atomictangerine} 19.18 $\uparrow$ & 0.0296 & \cellcolor{atomictangerine} 37.67 $\uparrow$ & 0.0329 & 0.0237 & \cellcolor{atomictangerine} 38.83 $\uparrow$ & 0.0242 & \cellcolor{atomictangerine} 35.67 $\uparrow$ & 0.0180 & \cellcolor{atomictangerine} 82.66 $\uparrow$ &  5.79 $\pm$ 0.50 \\
        & 2x100              & 210    & 0.0108 & 0.0090 & \cellcolor{apricot}         19.69 $\uparrow$ & 0.0092 & \cellcolor{atomictangerine} 17.60 $\uparrow$ & 0.0086 & \cellcolor{apricot}         25.85 $\uparrow$ & 0.0109 & 0.0081 & \cellcolor{atomictangerine} 35.27 $\uparrow$ & 0.0085 & \cellcolor{atomictangerine} 28.59 $\uparrow$ & 0.0077 & \cellcolor{apricot}         41.95 $\uparrow$ &  1.12 $\pm$ 0.01 \\
        & $\mbox{CNN}_{5-5}$ & 10,680 & 0.0305 & 0.0256 & \cellcolor{apricot}         18.88 $\uparrow$ & 0.0264 & \cellcolor{apricot}         15.55 $\uparrow$ & 0.0239 & \cellcolor{atomictangerine} 27.28 $\uparrow$ & 0.0201 & 0.0153 & \cellcolor{atomictangerine} 31.52 $\uparrow$ & 0.0158 & \cellcolor{atomictangerine} 27.03 $\uparrow$ & 0.0125 & \cellcolor{atomictangerine} 60.69 $\uparrow$ &  2.92 $\pm$ 0.34 \\
        & $\mbox{CNN}_{4-5}$ & 8,680  & 0.0276 & 0.0230 & \cellcolor{atomictangerine} 20.08 $\uparrow$ & 0.0240 & \cellcolor{apricot}         15.22 $\uparrow$ & 0.0223 & \cellcolor{apricot}         24.07 $\uparrow$ & 0.0188 & 0.0134 & \cellcolor{atomictangerine} 41.12 $\uparrow$ & 0.0139 & \cellcolor{atomictangerine} 35.83 $\uparrow$ & 0.0129 & \cellcolor{apricot}         45.82 $\uparrow$ &  1.22 $\pm$ 0.22 \\
        & $\mbox{CNN}_{3-2}$ & 2,514  & 0.0271 & 0.0233 & \cellcolor{antiquewhite}    15.85 $\uparrow$ & 0.0242 & \cellcolor{antiquewhite}    11.68 $\uparrow$ & 0.0234 & \cellcolor{antiquewhite}    15.45 $\uparrow$ & 0.0111 & 0.0087 & \cellcolor{apricot}         28.49 $\uparrow$ & 0.0091 & \cellcolor{apricot}         22.42 $\uparrow$ & 0.0085 & \cellcolor{antiquewhite}    31.37 $\uparrow$ &  0.18 $\pm$ 0.04 \\
        & $\mbox{CNN}_{3-4}$ & 5,018  & 0.0302 & 0.0264 & \cellcolor{antiquewhite}    14.19 $\uparrow$ & 0.0271 & \cellcolor{antiquewhite}    11.32 $\uparrow$ & 0.0263 & \cellcolor{antiquewhite}    14.88 $\uparrow$ & 0.0160 & 0.0126 & \cellcolor{apricot}         27.29 $\uparrow$ & 0.0130 & \cellcolor{apricot}         23.36 $\uparrow$ & 0.0121 & \cellcolor{antiquewhite}    32.45 $\uparrow$ &  0.31 $\pm$ 0.09 \\
        & 3x200              & 610    & 0.0529 & 0.0430 & \cellcolor{atomictangerine} 23.17 $\uparrow$ & 0.0440 & \cellcolor{atomictangerine} 20.23 $\uparrow$ & 0.0354 & \cellcolor{atomictangerine} 49.70 $\uparrow$ & 0.0662 & 0.0550 & \cellcolor{antiquewhite}    20.29 $\uparrow$ & 0.0563 & \cellcolor{antiquewhite}    17.43 $\uparrow$ & 0.0454 & \cellcolor{apricot}         45.67 $\uparrow$ & 10.77 $\pm$ 0.05 \\
        \hline
        \multirow{10}{*}{CIFAR10}
        & 3x200               & 610   & 0.0226 & 0.0127 & \cellcolor{atomictangerine}  78.69 $\uparrow$ & 0.0137 & \cellcolor{atomictangerine} 65.86 $\uparrow$ & 0.0120 & \cellcolor{atomictangerine}  88.82 $\uparrow$ & 0.0176 & 0.0083 & \cellcolor{atomictangerine} 113.30 $\uparrow$ & 0.0092 & \cellcolor{atomictangerine} 90.91 $\uparrow$ & 0.0080 & \cellcolor{atomictangerine} 119.40 $\uparrow$ &   6.04 $\pm$ 0.85 \\
        & 4x100               & 410   & 0.0198 & 0.0094 & \cellcolor{atomictangerine} 111.29 $\uparrow$ & 0.0106 & \cellcolor{atomictangerine} 88.06 $\uparrow$ & 0.0091 & \cellcolor{atomictangerine} 118.98 $\uparrow$ & 0.0146 & 0.0081 & \cellcolor{atomictangerine}  79.75 $\uparrow$ & 0.0101 & \cellcolor{atomictangerine} 45.63 $\uparrow$ & 0.0097 & \cellcolor{atomictangerine}  51.19 $\uparrow$ &   1.84 $\pm$ 0.00 \\
        & 3x700               & 2,110 & 0.0688 & 0.0364 & \cellcolor{atomictangerine}  89.11 $\uparrow$ & 0.0391 & \cellcolor{atomictangerine} 76.09 $\uparrow$ & 0.0384 & \cellcolor{atomictangerine}  78.93 $\uparrow$ & 0.0512 & 0.0377 & \cellcolor{atomictangerine}  35.99 $\uparrow$ & 0.0386 & \cellcolor{atomictangerine} 32.64 $\uparrow$ & 0.0406 & \cellcolor{atomictangerine}  26.33 $\uparrow$ & 103.08 $\pm$ 0.06 \\
        & 3x400               & 1,210 & 0.0484 & 0.0256 & \cellcolor{atomictangerine}  89.00 $\uparrow$ & 0.0295 & \cellcolor{atomictangerine} 64.15 $\uparrow$ & 0.0256 & \cellcolor{atomictangerine}  88.85 $\uparrow$ & 0.0441 & 0.0245 & \cellcolor{atomictangerine}  79.67 $\uparrow$ & 0.0291 & \cellcolor{atomictangerine} 51.58 $\uparrow$ & 0.0253 & \cellcolor{atomictangerine}  74.69 $\uparrow$ &  35.79 $\pm$ 0.05 \\
        & 3x100               & 310   & 0.0429 & 0.0273 & \cellcolor{apricot}          57.37 $\uparrow$ & 0.0295 & \cellcolor{apricot}         45.33 $\uparrow$ & 0.0269 & \cellcolor{apricot}          59.42 $\uparrow$ & 0.0435 & 0.0243 & \cellcolor{atomictangerine}  78.79 $\uparrow$ & 0.0270 & \cellcolor{atomictangerine} 61.35 $\uparrow$ & 0.0256 & \cellcolor{atomictangerine}  69.66 $\uparrow$ &   1.26 $\pm$ 0.00 \\
        & 5x100               & 510   & 0.0462 & 0.0299 & \cellcolor{apricot}          54.64 $\uparrow$ & 0.0315 & \cellcolor{apricot}         46.45 $\uparrow$ & 0.0300 & \cellcolor{apricot}          54.12 $\uparrow$ & 0.0400 & 0.0372 & \cellcolor{apricot}           7.58 $\uparrow$ & 0.0378 & \cellcolor{apricot}          5.96 $\uparrow$ & 0.0378 & \cellcolor{apricot}           5.87 $\uparrow$ &   2.43 $\pm$ 0.01 \\
        & 3x50                & 160   & 0.0170 & 0.0135 & \cellcolor{apricot}          25.94 $\uparrow$ & 0.0138 & \cellcolor{apricot}         23.30 $\uparrow$ & 0.0134 & \cellcolor{apricot}          27.26 $\uparrow$ & 0.0111 & 0.0073 & \cellcolor{atomictangerine}  53.52 $\uparrow$ & 0.0077 & \cellcolor{atomictangerine} 44.73 $\uparrow$ & 0.0071 & \cellcolor{atomictangerine}  56.32 $\uparrow$ &   0.43 $\pm$ 0.00 \\
        & $\mbox{CNN}_{3-2}$  & 3,378 & 0.0136 & 0.0134 & \cellcolor{antiquewhite}      1.49 $\uparrow$ & 0.0134 & \cellcolor{antiquewhite}     1.49 $\uparrow$ & 0.0134 & \cellcolor{antiquewhite}      1.56 $\uparrow$ & 0.0106 & 0.0102 & \cellcolor{apricot}           4.01 $\uparrow$ & 0.0102 & \cellcolor{apricot}          4.01 $\uparrow$ & 0.0102 & \cellcolor{apricot}           4.21 $\uparrow$ &   0.31 $\pm$ 0.06 \\
        & $\mbox{CNN}_{3-4}$  & 6,746 & 0.0102 & 0.0101 & \cellcolor{antiquewhite}      0.89 $\uparrow$ & 0.0101 & \cellcolor{antiquewhite}     0.99 $\uparrow$ & 0.0101 & \cellcolor{antiquewhite}      0.99 $\uparrow$ & 0.0062 & 0.0061 & \cellcolor{antiquewhite}      1.80 $\uparrow$ & 0.0061 & \cellcolor{antiquewhite}     1.80 $\uparrow$ & 0.0061 & \cellcolor{antiquewhite}      1.96 $\uparrow$ &   0.59 $\pm$ 0.12 \\
        & $\mbox{CNN}_{3-5}$  & 8,430 & 0.0094 & 0.0093 & \cellcolor{antiquewhite}      0.97 $\uparrow$ & 0.0093 & \cellcolor{antiquewhite}     0.97 $\uparrow$ & 0.0093 & \cellcolor{antiquewhite}      0.97 $\uparrow$ & 0.0066 & 0.0065 & \cellcolor{antiquewhite}      1.86 $\uparrow$ & 0.0065 & \cellcolor{antiquewhite}     1.86 $\uparrow$ & 0.0065 & \cellcolor{antiquewhite}      2.02 $\uparrow$ &   0.75 $\pm$ 0.15 \\
        \hline
    \end{tabular}
    \label{all resutls on Tanh models}
\end{table*}

\begin{table*}[ht]
    \scriptsize 
    \def\arraystretch{1}
    \setlength{\tabcolsep}{4pt}
    \caption{Comparison on non-negative Arctan networks.}
    \vspace{-2mm}
    \centering
    \begin{tabular}{|c|c|r|r|r|r|r|r|r|r|r|r|r|r|r|r|r|r|}
        \hline 
        \multirow{3}{*}{\textbf{Dataset}} & \multirow{3}{*}{\textbf{Model}} & \multirow{3}{*}{\textbf{\#Neur.}} &   \multicolumn{14}{c|}{\textbf{Certified Lower Bound}} & \multirow{3}{*}{\textbf{Time (s)}}  \\
        \hhline{~~~--------------~}
        & & & \multicolumn{7}{c|}{\textbf{Average}} & \multicolumn{7}{c|}{\textbf{ Standard Deviation}} & \\
        \hhline{~~~--------------~}
        & & & \textsc{NW} & \textsc{DC} &  Impr. (\%) & \textsc{VN} &  Impr. (\%) & \textsc{RV} &  Impr. (\%) & \textsc{NW} & \textsc{DC} &  Impr. (\%) & \textsc{VN} &  Impr. (\%) & \textsc{RV} &  Impr. (\%) &  \\
        \hline
        \hline
        \multirow{11}{*}{MNIST} 
        & $\mbox{CNN}_{6-5}$ & 12,300 & 0.0484 & 0.0262 & \cellcolor{atomictangerine} 84.78 $\uparrow$ & 0.0367 & \cellcolor{atomictangerine} 32.11 $\uparrow$ & 0.0325 & \cellcolor{atomictangerine} 49.15 $\uparrow$ & 0.0189 & 0.0072 & \cellcolor{atomictangerine} 162.78 $\uparrow$ & 0.0124 & \cellcolor{atomictangerine} 52.21 $\uparrow$ & 0.0098 & \cellcolor{atomictangerine} 92.67 $\uparrow$ &   5.57 $\pm$0.05    \\ 
        & 5x100              & 510    & 0.0020 & 0.0012 & \cellcolor{atomictangerine} 73.28 $\uparrow$ & 0.0015 & \cellcolor{atomictangerine} 29.68 $\uparrow$ & 0.0014 & \cellcolor{atomictangerine} 41.55 $\uparrow$ & 0.0012 & 0.0005 & \cellcolor{atomictangerine} 144.00 $\uparrow$ & 0.0008 & \cellcolor{atomictangerine} 48.78 $\uparrow$ & 0.0007 & \cellcolor{atomictangerine} 67.12 $\uparrow$ &   4.36 $\pm$0.02  \\ 
        & 3x700              & 2,110  & 0.0022 & 0.0015 & \cellcolor{atomictangerine} 47.06 $\uparrow$ & 0.0019 & \cellcolor{atomictangerine} 20.32 $\uparrow$ & 0.0017 & \cellcolor{atomictangerine} 29.31 $\uparrow$ & 0.0012 & 0.0007 & \cellcolor{apricot}          58.11 $\uparrow$ & 0.0009 & \cellcolor{atomictangerine} 28.57 $\uparrow$ & 0.0008 & \cellcolor{atomictangerine} 42.68 $\uparrow$ & 117.69 $\pm$0.24  \\ 
        & 3x50               & 160    & 0.0051 & 0.0037 & \cellcolor{apricot}         36.83 $\uparrow$ & 0.0043 & \cellcolor{apricot}         18.10 $\uparrow$ & 0.0040 & \cellcolor{atomictangerine} 26.62 $\uparrow$ & 0.0026 & 0.0017 & \cellcolor{apricot}          51.15 $\uparrow$ & 0.0020 & \cellcolor{atomictangerine} 32.83 $\uparrow$ & 0.0017 & \cellcolor{atomictangerine} 52.91 $\uparrow$ &   0.14 $\pm$0.00  \\ 
        & 3x400              & 1,210  & 0.0024 & 0.0017 & \cellcolor{apricot}         39.64 $\uparrow$ & 0.0020 & \cellcolor{apricot}         17.41 $\uparrow$ & 0.0019 & \cellcolor{atomictangerine} 24.21 $\uparrow$ & 0.0011 & 0.0007 & \cellcolor{apricot}          61.76 $\uparrow$ & 0.0009 & \cellcolor{atomictangerine} 26.44 $\uparrow$ & 0.0008 & \cellcolor{atomictangerine} 37.50 $\uparrow$ &  39.59 $\pm$0.13  \\ 
        & 3x200              & 610    & 0.0028 & 0.0021 & \cellcolor{apricot}         36.23 $\uparrow$ & 0.0024 & \cellcolor{apricot}         16.53 $\uparrow$ & 0.0023 & \cellcolor{atomictangerine} 22.08 $\uparrow$ & 0.0012 & 0.0008 & \cellcolor{apricot}          50.00 $\uparrow$ & 0.0010 & \cellcolor{atomictangerine} 28.12 $\uparrow$ & 0.0009 & \cellcolor{atomictangerine} 41.38 $\uparrow$ &  10.68 $\pm$0.04  \\ 
        & 3x100              & 310    & 0.0031 & 0.0023 & \cellcolor{apricot}         32.76 $\uparrow$ & 0.0027 & \cellcolor{apricot}         14.93 $\uparrow$ & 0.0026 & \cellcolor{apricot}         19.84 $\uparrow$ & 0.0013 & 0.0009 & \cellcolor{antiquewhite}     42.05 $\uparrow$ & 0.0010 & \cellcolor{atomictangerine} 22.55 $\uparrow$ & 0.0009 & \cellcolor{apricot}         31.58 $\uparrow$ &   2.21 $\pm$0.03  \\ 
        & $\mbox{CNN}_{5-5}$ & 10,680 & 0.0218 & 0.0162 & \cellcolor{apricot}         34.24 $\uparrow$ & 0.0196 & \cellcolor{apricot}         11.22 $\uparrow$ & 0.0191 & \cellcolor{apricot}         14.02 $\uparrow$ & 0.0069 & 0.0043 & \cellcolor{apricot}          62.38 $\uparrow$ & 0.0059 & \cellcolor{apricot}         18.80 $\uparrow$ & 0.0055 & \cellcolor{apricot}         26.59 $\uparrow$ &   2.79 $\pm$0.05 \\
        & $\mbox{CNN}_{3-2}$ & 2,514  & 0.0138 & 0.0103 & \cellcolor{apricot}         34.66 $\uparrow$ & 0.0127 & \cellcolor{antiquewhite}     8.90 $\uparrow$ & 0.0124 & \cellcolor{apricot}         11.08 $\uparrow$ & 0.0042 & 0.0027 & \cellcolor{apricot}          56.88 $\uparrow$ & 0.0037 & \cellcolor{apricot}         12.83 $\uparrow$ & 0.0036 & \cellcolor{antiquewhite}    17.88 $\uparrow$ &   0.16 $\pm$0.01 \\
        & $\mbox{CNN}_{4-5}$ & 8,680  & 0.0203 & 0.0158 & \cellcolor{antiquewhite}    28.08 $\uparrow$ & 0.0186 & \cellcolor{antiquewhite}     9.05 $\uparrow$ & 0.0182 & \cellcolor{apricot}         11.20 $\uparrow$ & 0.0068 & 0.0045 & \cellcolor{apricot}          52.00 $\uparrow$ & 0.0059 & \cellcolor{apricot}         15.15 $\uparrow$ & 0.0057 & \cellcolor{antiquewhite}    20.85 $\uparrow$ &   1.12 $\pm$0.03  \\ 
        & $\mbox{CNN}_{3-4}$ & 5,018  & 0.0162 & 0.0130 & \cellcolor{antiquewhite}    24.71 $\uparrow$ & 0.0152 & \cellcolor{antiquewhite}     6.51 $\uparrow$ & 0.0151 & \cellcolor{antiquewhite}     7.57 $\uparrow$ & 0.0053 & 0.0037 & \cellcolor{antiquewhite}     42.28 $\uparrow$ & 0.0048 & \cellcolor{antiquewhite}     9.83 $\uparrow$ & 0.0047 & \cellcolor{antiquewhite}    12.66 $\uparrow$ &   0.28 $\pm$0.02  \\ 
        \hline
        \multirow{10}{*}{\makecell{Fashion \\ MNIST} }
        & 4x100              & 410    & 0.0179 & 0.0084 & \cellcolor{atomictangerine} 112.98 $\uparrow$ & 0.0108 & \cellcolor{atomictangerine} 66.26 $\uparrow$ & 0.0092 & \cellcolor{atomictangerine} 94.88 $\uparrow$ & 0.0279 & 0.0118 & \cellcolor{atomictangerine} 136.19 $\uparrow$ & 0.0150 & \cellcolor{atomictangerine} 85.31 $\uparrow$ & 0.0134 & \cellcolor{atomictangerine} 107.68 $\uparrow$ & 3.20 $\pm$0.02   \\ 
        & 3x100              & 310    & 0.0120 & 0.0078 & \cellcolor{atomictangerine}  53.98 $\uparrow$ & 0.0091 & \cellcolor{atomictangerine} 31.07 $\uparrow$ & 0.0081 & \cellcolor{apricot}         47.17 $\uparrow$ & 0.0127 & 0.0076 & \cellcolor{apricot}          66.80 $\uparrow$ & 0.0089 & \cellcolor{apricot}         42.41 $\uparrow$ & 0.0078 & \cellcolor{apricot}          63.14 $\uparrow$ & 2.13 $\pm$0.01 \\ 
        & 3x200              & 610    & 0.0097 & 0.0067 & \cellcolor{atomictangerine}  44.68 $\uparrow$ & 0.0077 & \cellcolor{atomictangerine} 25.49 $\uparrow$ & 0.0071 & \cellcolor{apricot}         36.49 $\uparrow$ & 0.0107 & 0.0073 & \cellcolor{apricot}          45.37 $\uparrow$ & 0.0081 & \cellcolor{apricot}         31.89 $\uparrow$ & 0.0070 & \cellcolor{apricot}          52.43 $\uparrow$ & 8.90 $\pm$1.88 \\
        & 2x100              & 210    & 0.0105 & 0.0082 & \cellcolor{antiquewhite}     27.95 $\uparrow$ & 0.0091 & \cellcolor{apricot}         16.35 $\uparrow$ & 0.0086 & \cellcolor{apricot}         23.01 $\uparrow$ & 0.0095 & 0.0069 & \cellcolor{antiquewhite}     37.19 $\uparrow$ & 0.0074 & \cellcolor{apricot}         28.11 $\uparrow$ & 0.0069 & \cellcolor{antiquewhite}     37.99 $\uparrow$ & 1.08 $\pm$0.01 \\ 
        & $\mbox{CNN}_{6-5}$ & 12,300 & 0.0324 & 0.0228 & \cellcolor{atomictangerine}  42.16 $\uparrow$ & 0.0282 & \cellcolor{apricot}         14.95 $\uparrow$ & 0.0258 & \cellcolor{apricot}         25.88 $\uparrow$ & 0.0228 & 0.0157 & \cellcolor{apricot}          45.14 $\uparrow$ & 0.0180 & \cellcolor{apricot}         27.19 $\uparrow$ & 0.0134 & \cellcolor{apricot}          69.87 $\uparrow$ & 5.63 $\pm$0.16 \\
        & 2x200              & 410    & 0.0086 & 0.0069 & \cellcolor{antiquewhite}     24.82 $\uparrow$ & 0.0076 & \cellcolor{apricot}         13.97 $\uparrow$ & 0.0073 & \cellcolor{antiquewhite}    18.82 $\uparrow$ & 0.0079 & 0.0062 & \cellcolor{antiquewhite}     28.04 $\uparrow$ & 0.0065 & \cellcolor{apricot}         20.98 $\uparrow$ & 0.0061 & \cellcolor{antiquewhite}     29.08 $\uparrow$ & 4.50 $\pm$0.92 \\ 
        & $\mbox{CNN}_{3-2}$ & 2,514  & 0.0266 & 0.0190 & \cellcolor{atomictangerine}  40.32 $\uparrow$ & 0.0235 & \cellcolor{apricot}         13.34 $\uparrow$ & 0.0225 & \cellcolor{antiquewhite}    17.97 $\uparrow$ & 0.0179 & 0.0105 & \cellcolor{apricot}          71.41 $\uparrow$ & 0.0127 & \cellcolor{apricot}         41.29 $\uparrow$ & 0.0115 & \cellcolor{apricot}          55.37 $\uparrow$ & 0.16 $\pm$0.01 \\
        & $\mbox{CNN}_{4-5}$ & 8,680  & 0.0241 & 0.0176 & \cellcolor{apricot}          37.07 $\uparrow$ & 0.0217 & \cellcolor{apricot}         10.82 $\uparrow$ & 0.0207 & \cellcolor{antiquewhite}    16.06 $\uparrow$ & 0.0117 & 0.0086 & \cellcolor{antiquewhite}     36.96 $\uparrow$ & 0.0099 & \cellcolor{antiquewhite}    17.93 $\uparrow$ & 0.0090 & \cellcolor{antiquewhite}     30.40 $\uparrow$ & 1.14 $\pm$0.06 \\ 
        & $\mbox{CNN}_{5-5}$ & 10,680 & 0.0266 & 0.0201 & \cellcolor{apricot}          32.22 $\uparrow$ & 0.0241 & \cellcolor{apricot}         10.50 $\uparrow$ & 0.0228 & \cellcolor{antiquewhite}    16.95 $\uparrow$ & 0.0121 & 0.0097 & \cellcolor{antiquewhite}     24.77 $\uparrow$ & 0.0107 & \cellcolor{antiquewhite}    13.88 $\uparrow$ & 0.0099 & \cellcolor{antiquewhite}     22.75 $\uparrow$ & 2.81 $\pm$0.08 \\ 
        & $\mbox{CNN}_{3-4}$ & 5,018  & 0.0258 & 0.0205 & \cellcolor{antiquewhite}     26.05 $\uparrow$ & 0.0237 & \cellcolor{antiquewhite}     8.64 $\uparrow$ & 0.0231 & \cellcolor{antiquewhite}    11.60 $\uparrow$ & 0.0158 & 0.0114 & \cellcolor{antiquewhite}     38.50 $\uparrow$ & 0.0135 & \cellcolor{antiquewhite}    17.61 $\uparrow$ & 0.0124 & \cellcolor{antiquewhite}     27.97 $\uparrow$ & 0.29 $\pm$0.02 \\ 
        \hline
        \multirow{11}{*}{CIFAR10} 
        & 4x100              & 410   & 0.0173 & 0.0145 & \cellcolor{atomictangerine} 19.30 $\uparrow$ & 0.0158 & \cellcolor{atomictangerine} 9.70 $\uparrow$ & 0.0156 & \cellcolor{atomictangerine} 11.25 $\uparrow$ & 0.0128 & 0.0100 & \cellcolor{atomictangerine} 27.86 $\uparrow$ & 0.0110 & \cellcolor{atomictangerine} 16.11 $\uparrow$ & 0.0107 & \cellcolor{atomictangerine} 19.59 $\uparrow$ &   1.83 $\pm$0.01   \\ 
        & 5x100              & 510   & 0.0248 & 0.0213 & \cellcolor{atomictangerine} 16.67 $\uparrow$ & 0.0229 & \cellcolor{atomictangerine} 8.28 $\uparrow$ & 0.0227 & \cellcolor{atomictangerine}  9.38 $\uparrow$ & 0.0195 & 0.0167 & \cellcolor{apricot}         17.16 $\uparrow$ & 0.0176 & \cellcolor{atomictangerine} 11.16 $\uparrow$ & 0.0173 & \cellcolor{atomictangerine} 13.09 $\uparrow$ &   2.42 $\pm$0.00   \\ 
        & 3x50               & 160   & 0.0186 & 0.0158 & \cellcolor{atomictangerine} 17.17 $\uparrow$ & 0.0171 & \cellcolor{atomictangerine} 8.41 $\uparrow$ & 0.0169 & \cellcolor{atomictangerine}  9.82 $\uparrow$ & 0.0112 & 0.0092 & \cellcolor{atomictangerine} 22.04 $\uparrow$ & 0.0099 & \cellcolor{atomictangerine} 13.42 $\uparrow$ & 0.0096 & \cellcolor{atomictangerine} 17.45 $\uparrow$ &   0.43 $\pm$0.00   \\ 
        & 6x100              & 610   & 0.0171 & 0.0148 & \cellcolor{atomictangerine} 15.64 $\uparrow$ & 0.0159 & \cellcolor{atomictangerine} 7.69 $\uparrow$ & 0.0158 & \cellcolor{atomictangerine}  8.44 $\uparrow$ & 0.0157 & 0.0139 & \cellcolor{apricot}         12.72 $\uparrow$ & 0.0144 & \cellcolor{apricot}          9.19 $\uparrow$ & 0.0142 & \cellcolor{apricot}         10.42 $\uparrow$ &   3.03 $\pm$0.01   \\ 
        & 3x100              & 310   & 0.0138 & 0.0119 & \cellcolor{atomictangerine} 15.71 $\uparrow$ & 0.0127 & \cellcolor{atomictangerine} 8.17 $\uparrow$ & 0.0126 & \cellcolor{atomictangerine}  9.29 $\uparrow$ & 0.0113 & 0.0094 & \cellcolor{atomictangerine} 20.53 $\uparrow$ & 0.0100 & \cellcolor{atomictangerine} 13.30 $\uparrow$ & 0.0098 & \cellcolor{atomictangerine} 15.85 $\uparrow$ &   1.26 $\pm$0.01   \\ 
        & 3x200              & 610   & 0.0107 & 0.0095 & \cellcolor{apricot}         12.92 $\uparrow$ & 0.0101 & \cellcolor{apricot}         6.23 $\uparrow$ & 0.0101 & \cellcolor{apricot}          6.86 $\uparrow$ & 0.0073 & 0.0061 & \cellcolor{apricot}         19.12 $\uparrow$ & 0.0066 & \cellcolor{atomictangerine} 10.12 $\uparrow$ & 0.0065 & \cellcolor{apricot}         11.47 $\uparrow$ &   6.33 $\pm$0.01   \\ 
        & 3x700              & 2,110 & 0.0079 & 0.0070 & \cellcolor{apricot}         13.12 $\uparrow$ & 0.0075 & \cellcolor{apricot}         5.73 $\uparrow$ & 0.0074 & \cellcolor{apricot}          6.59 $\uparrow$ & 0.0053 & 0.0044 & \cellcolor{apricot}         19.33 $\uparrow$ & 0.0049 & \cellcolor{apricot}          9.03 $\uparrow$ & 0.0048 & \cellcolor{apricot}         10.63 $\uparrow$ & 102.47 $\pm$0.14   \\ 
        & 3x400              & 1,210 & 0.0105 & 0.0094 & \cellcolor{apricot}         12.06 $\uparrow$ & 0.0099 & \cellcolor{apricot}         5.74 $\uparrow$ & 0.0099 & \cellcolor{apricot}          6.38 $\uparrow$ & 0.0068 & 0.0058 & \cellcolor{apricot}         17.76 $\uparrow$ & 0.0062 & \cellcolor{apricot}          9.46 $\uparrow$ & 0.0062 & \cellcolor{apricot}         10.88 $\uparrow$ &  37.85 $\pm$0.02   \\ 
        & $\mbox{CNN}_{3-2}$ & 3,378 & 0.0115 & 0.0113 & \cellcolor{antiquewhite}     1.77 $\uparrow$ & 0.0114 & \cellcolor{antiquewhite}    0.88 $\uparrow$ & 0.0114 & \cellcolor{antiquewhite}     0.88 $\uparrow$ & 0.0078 & 0.0076 & \cellcolor{antiquewhite}     3.02 $\uparrow$ & 0.0077 & \cellcolor{antiquewhite}     1.55 $\uparrow$ & 0.0077 & \cellcolor{antiquewhite}     1.69 $\uparrow$ &   0.29 $\pm$0.02  \\ 
        & $\mbox{CNN}_{3-4}$ & 6,746 & 0.0080 & 0.0079 & \cellcolor{antiquewhite}     1.53 $\uparrow$ & 0.0079 & \cellcolor{antiquewhite}    0.88 $\uparrow$ & 0.0079 & \cellcolor{antiquewhite}     0.88 $\uparrow$ & 0.0051 & 0.0049 & \cellcolor{antiquewhite}     2.84 $\uparrow$ & 0.0050 & \cellcolor{antiquewhite}     1.40 $\uparrow$ & 0.0050 & \cellcolor{antiquewhite}     1.60 $\uparrow$ &   0.55 $\pm$0.03  \\ 
        & $\mbox{CNN}_{3-5}$ & 8,430 & 0.0095 & 0.0094 & \cellcolor{antiquewhite}     1.39 $\uparrow$ & 0.0094 & \cellcolor{antiquewhite}    0.64 $\uparrow$ & 0.0094 & \cellcolor{antiquewhite}     0.64 $\uparrow$ & 0.0060 & 0.0059 & \cellcolor{antiquewhite}     2.72 $\uparrow$ & 0.0060 & \cellcolor{antiquewhite}     1.17 $\uparrow$ & 0.0060 & \cellcolor{antiquewhite}     1.34 $\uparrow$ &   0.70 $\pm$0.03  \\ 
        \hline
    \end{tabular}
    \label{all resutls on Arctan models}
\end{table*}

\begin{table*}[ht]
    \footnotesize  
    \def\arraystretch{1}
    \setlength{\tabcolsep}{4pt}
    \caption{Performance comparison of Alg. 1 with \textsc{DeepCert} (DC), \textsc{VeriNet} (VN), and \textsc{RobustVerifier} (RV) on 1-hidden-layer Sigmoid and Tanh networks with mixed weights.}
    \centering
    \resizebox{\textwidth}{!}{
    \begin{tabular}{|c|c|r|r|r|r|r|r|r|r|r|r|r|r|r|r|r|r|r|}
        \hline 
        \multirow{3}{*}{\textbf{Arch.}} & \multirow{3}{*}{\textbf{Model}} & \multirow{3}{*}{\textbf{$\sigma$}} &   \multicolumn{14}{c|}{\textbf{Certified Lower Bound}} & \multicolumn{2}{c|}{\textbf{Time (s)}}  \\
        \hhline{~~~----------------}
        & & &  \multicolumn{7}{c|}{\textbf{Average}} & \multicolumn{7}{c|}{\textbf{Standard Deviation}} &     \multirow{2}{*}{\textsc{Alg.1}} & \multirow{2}{*}{\textsc{Others}} \\
        \hhline{~~~--------------~~}
        & & & \textsc{Alg.1} & \textsc{DC} &  Impr. (\%) & \textsc{VN} &  Impr. (\%) & \textsc{RV} & Impr. (\%) &  \textsc{Alg.1} & \textsc{DC} &  Impr. (\%) & \textsc{VN} &  Impr. (\%) & \textsc{RV} & Impr. (\%) &  & \\
        \hline
       \hline
        \multirow{20}{*}{CNN} 
        & \multirow{2}{*}{$\mbox{CNN}_{2-1-5}$$^\ast$} & Sig. & 0.0706 & 0.0437 & 61.39 $\uparrow$ & 0.0435 &  62.32 $\uparrow$ & 0.0428 & 64.90 $\uparrow$ & 0.0306 & 0.0145 & 110.32 $\uparrow$  & 0.0143 & 114.30 $\uparrow$  & 0.0137 & 122.57 $\uparrow$  & 0.56  & 0.06 $\pm$0.01    \\ 
        &                                              & Tanh & 0.0489 & 0.0286 & 71.06 $\uparrow$ & 0.0285 &  71.54 $\uparrow$ & 0.0277 & 76.74 $\uparrow$ & 0.0227 & 0.0132 &  71.32 $\uparrow$  & 0.0132 &  72.10 $\uparrow$  & 0.0121 &  87.00 $\uparrow$  & 0.59  & 0.06 $\pm$0.01    \\ 
        & \multirow{2}{*}{$\mbox{CNN}_{2-2-5}$$^\ast$} & Sig. & 0.0749 & 0.0541 & 38.29 $\uparrow$ & 0.0539 &  38.98 $\uparrow$ & 0.0521 & 43.79 $\uparrow$ & 0.0312 & 0.0208 &  49.97 $\uparrow$  & 0.0207 &  51.14 $\uparrow$  & 0.0187 &  66.63 $\uparrow$  & 1.19  & 0.08 $\pm$0.01    \\ 
        &                                              & Tanh & 0.0449 & 0.0332 & 35.03 $\uparrow$ & 0.0331 &  35.48 $\uparrow$ & 0.0321 & 39.91 $\uparrow$ & 0.0204 & 0.0123 &  66.19 $\uparrow$  & 0.0122 &  67.15 $\uparrow$  & 0.0113 &  80.64 $\uparrow$  & 1.22  & 0.08 $\pm$0.01    \\ 
        & \multirow{2}{*}{$\mbox{CNN}_{2-3-5}$$^\ast$} & Sig. & 0.0689 & 0.0526 & 31.03 $\uparrow$ & 0.0524 &  31.48 $\uparrow$ & 0.0503 & 37.03 $\uparrow$ & 0.0279 & 0.0197 &  41.45 $\uparrow$  & 0.0196 &  41.88 $\uparrow$  & 0.0176 &  58.60 $\uparrow$  & 1.75  & 0.09 $\pm$0.02    \\ 
        &                                              & Tanh & 0.0444 & 0.0327 & 35.84 $\uparrow$ & 0.0326 &  36.05 $\uparrow$ & 0.0317 & 40.26 $\uparrow$ & 0.0182 & 0.0121 &  50.90 $\uparrow$  & 0.0120 &  51.15 $\uparrow$  & 0.0111 &  63.94 $\uparrow$  & 1.99  & 0.10 $\pm$0.02    \\ 
        & \multirow{2}{*}{$\mbox{CNN}_{2-4-5}$$^\ast$} & Sig. & 0.0915 & 0.0647 & 41.39 $\uparrow$ & 0.0645 &  41.81 $\uparrow$ & 0.0621 & 47.43 $\uparrow$ & 0.0405 & 0.0233 &  73.73 $\uparrow$  & 0.0232 &  74.55 $\uparrow$  & 0.0210 &  92.76 $\uparrow$  & 2.54  & 0.11 $\pm$0.02    \\ 
        &                                              & Tanh & 0.0446 & 0.0272 & 63.87 $\uparrow$ & 0.0272 &  64.06 $\uparrow$ & 0.0266 & 67.95 $\uparrow$ & 0.0184 & 0.0103 &  78.91 $\uparrow$  & 0.0103 &  78.91 $\uparrow$  & 0.0096 &  91.81 $\uparrow$  & 2.38  & 0.11 $\pm$0.03    \\ 
        & \multirow{2}{*}{$\mbox{CNN}_{2-5-3}$$^\ast$} & Sig. & 0.0682 & 0.0477 & 42.81 $\uparrow$ & 0.0476 &  43.11 $\uparrow$ & 0.0471 & 44.81 $\uparrow$ & 0.0349 & 0.0182 &  91.53 $\uparrow$  & 0.0182 &  92.27 $\uparrow$  & 0.0176 &  98.50 $\uparrow$  & 2.57  & 0.09 $\pm$0.03    \\ 
        &                                              & Tanh & 0.0289 & 0.0192 & 50.28 $\uparrow$ & 0.0192 &  50.52 $\uparrow$ & 0.0191 & 51.62 $\uparrow$ & 0.0134 & 0.0073 &  85.10 $\uparrow$  & 0.0072 &  85.86 $\uparrow$  & 0.0071 &  89.81 $\uparrow$  & 2.93  & 0.10 $\pm$0.03    \\ 
        & \multirow{2}{*}{$\mbox{CNN}_{2-1-5}$$^+$}    & Sig. & 0.1013 & 0.0877 & 15.47 $\uparrow$ & 0.0876 &  15.56 $\uparrow$ & 0.0835 & 21.28 $\uparrow$ & 0.0506 & 0.0385 &  31.58 $\uparrow$  & 0.0386 &  31.24 $\uparrow$  & 0.0348 &  45.29 $\uparrow$  & 0.43  & 0.06 $\pm$0.01    \\ 
        &                                              & Tanh & 0.0658 & 0.0555 & 18.58 $\uparrow$ & 0.0554 &  18.82 $\uparrow$ & 0.0525 & 25.36 $\uparrow$ & 0.0349 & 0.0253 &  37.89 $\uparrow$  & 0.0254 &  37.62 $\uparrow$  & 0.0227 &  54.06 $\uparrow$  & 0.57  & 0.06 $\pm$0.01    \\ 
        & \multirow{2}{*}{$\mbox{CNN}_{2-2-5}$$^+$}    & Sig. & 0.1045 & 0.0823 & 26.99 $\uparrow$ & 0.0822 &  27.05 $\uparrow$ & 0.0777 & 34.47 $\uparrow$ & 0.0574 & 0.0353 &  62.46 $\uparrow$  & 0.0355 &  61.59 $\uparrow$  & 0.0311 &  84.39 $\uparrow$  & 0.80  & 0.08 $\pm$0.01    \\ 
        &                                              & Tanh & 0.0629 & 0.0536 & 17.39 $\uparrow$ & 0.0535 &  17.52 $\uparrow$ & 0.0509 & 23.62 $\uparrow$ & 0.0354 & 0.0260 &  36.39 $\uparrow$  & 0.0260 &  36.13 $\uparrow$  & 0.0233 &  51.69 $\uparrow$  & 0.71  & 0.08 $\pm$0.02    \\ 
        & \multirow{2}{*}{$\mbox{CNN}_{2-3-5}$$^+$}    & Sig. & 0.1036 & 0.0843 & 22.77 $\uparrow$ & 0.0843 &  22.79 $\uparrow$ & 0.0811 & 27.72 $\uparrow$ & 0.0509 & 0.0371 &  37.06 $\uparrow$  & 0.0374 &  36.07 $\uparrow$  & 0.0344 &  47.89 $\uparrow$  & 1.56  & 0.09 $\pm$0.02    \\ 
        &                                              & Tanh & 0.0699 & 0.0507 & 38.00 $\uparrow$ & 0.0506 &  38.22 $\uparrow$ & 0.0478 & 46.07 $\uparrow$ & 0.0404 & 0.0245 &  65.28 $\uparrow$  & 0.0245 &  65.35 $\uparrow$  & 0.0215 &  88.02 $\uparrow$  & 1.38  & 0.10 $\pm$0.02    \\ 
        & \multirow{2}{*}{$\mbox{CNN}_{2-4-5}$$^+$}    & Sig. & 0.1040 & 0.0814 & 27.76 $\uparrow$ & 0.0814 &  27.81 $\uparrow$ & 0.0784 & 32.73 $\uparrow$ & 0.0587 & 0.0366 &  60.43 $\uparrow$  & 0.0367 &  59.91 $\uparrow$  & 0.0336 &  74.50 $\uparrow$  & 1.80  & 0.11 $\pm$0.02    \\ 
        &                                              & Tanh & 0.0713 & 0.0519 & 37.43 $\uparrow$ & 0.0518 &  37.59 $\uparrow$ & 0.0494 & 44.19 $\uparrow$ & 0.0447 & 0.0246 &  81.51 $\uparrow$  & 0.0246 &  81.81 $\uparrow$  & 0.0219 & 104.05 $\uparrow$  & 1.84  & 0.11 $\pm$0.03    \\ 
        & \multirow{2}{*}{$\mbox{CNN}_{2-5-3}$$^+$}    & Sig. & 0.0922 & 0.0700 & 31.65 $\uparrow$ & 0.0699 &  31.84 $\uparrow$ & 0.0678 & 36.00 $\uparrow$ & 0.0485 & 0.0340 &  42.87 $\uparrow$  & 0.0340 &  42.66 $\uparrow$  & 0.0312 &  55.36 $\uparrow$  & 2.33  & 0.09 $\pm$0.04    \\ 
        &                                              & Tanh & 0.0520 & 0.0336 & 54.99 $\uparrow$ & 0.0335 &  55.18 $\uparrow$ & 0.0325 & 60.25 $\uparrow$ & 0.0389 & 0.0183 & 112.18 $\uparrow$  & 0.0184 & 112.06 $\uparrow$  & 0.0167 & 133.72 $\uparrow$  & 2.78  & 0.10 $\pm$0.04    \\ 
        \hline
        \multirow{20}{*}{FNN} 
        & 1x50$^\ast$  & Sig. & 0.0357 & 0.0275 & 29.65 $\uparrow$ & 0.0279 & 27.98 $\uparrow$ & 0.0235 & 51.64 $\uparrow$ & 0.0121 & 0.0082 &  47.17 $\uparrow$ & 0.0085 &  41.97 $\uparrow$ & 0.0062 &  93.70 $\uparrow$ & 0.33 & 0.03 $\pm$0.00    \\ 
        &              & Tanh & 0.0241 & 0.0197 & 22.49 $\uparrow$ & 0.0199 & 21.25 $\uparrow$ & 0.0167 & 44.60 $\uparrow$ & 0.0077 & 0.0056 &  36.69 $\uparrow$ & 0.0058 &  32.44 $\uparrow$ & 0.0042 &  82.26 $\uparrow$ & 0.29 & 0.03 $\pm$0.00    \\ 
        & 1x100$^\ast$ & Sig. & 0.0324 & 0.0247 & 31.15 $\uparrow$ & 0.0250 & 29.73 $\uparrow$ & 0.0215 & 50.72 $\uparrow$ & 0.0104 & 0.0064 &  63.11 $\uparrow$ & 0.0066 &  56.72 $\uparrow$ & 0.0050 & 108.97 $\uparrow$ & 0.40 & 0.04 $\pm$0.00    \\ 
        &              & Tanh & 0.0219 & 0.0177 & 23.47 $\uparrow$ & 0.0179 & 22.23 $\uparrow$ & 0.0152 & 43.94 $\uparrow$ & 0.0077 & 0.0049 &  57.37 $\uparrow$ & 0.0050 &  52.69 $\uparrow$ & 0.0037 & 107.43 $\uparrow$ & 0.33 & 0.05 $\pm$0.01    \\ 
        & 1x150$^\ast$ & Sig. & 0.0373 & 0.0290 & 28.59 $\uparrow$ & 0.0292 & 27.71 $\uparrow$ & 0.0254 & 46.55 $\uparrow$ & 0.0133 & 0.0083 &  60.57 $\uparrow$ & 0.0085 &  55.68 $\uparrow$ & 0.0064 & 107.16 $\uparrow$ & 0.36 & 0.06 $\pm$0.00    \\ 
        &              & Tanh & 0.0238 & 0.0197 & 20.81 $\uparrow$ & 0.0198 & 20.02 $\uparrow$ & 0.0171 & 39.50 $\uparrow$ & 0.0076 & 0.0055 &  39.20 $\uparrow$ & 0.0056 &  35.96 $\uparrow$ & 0.0042 &  80.43 $\uparrow$ & 0.37 & 0.06 $\pm$0.00    \\ 
        & 1x200$^\ast$ & Sig. & 0.0343 & 0.0260 & 31.83 $\uparrow$ & 0.0261 & 31.17 $\uparrow$ & 0.0230 & 48.94 $\uparrow$ & 0.0117 & 0.0074 &  57.83 $\uparrow$ & 0.0076 &  54.49 $\uparrow$ & 0.0058 & 102.41 $\uparrow$ & 0.38 & 0.07 $\pm$0.00    \\ 
        &              & Tanh & 0.0228 & 0.0183 & 24.64 $\uparrow$ & 0.0184 & 24.03 $\uparrow$ & 0.0162 & 41.08 $\uparrow$ & 0.0077 & 0.0044 &  74.74 $\uparrow$ & 0.0045 &  70.51 $\uparrow$ & 0.0035 & 121.17 $\uparrow$ & 0.38 & 0.07 $\pm$0.01    \\ 
        & 1x250$^\ast$ & Sig. & 0.0364 & 0.0256 & 42.34 $\uparrow$ & 0.0257 & 41.84 $\uparrow$ & 0.0229 & 58.86 $\uparrow$ & 0.0117 & 0.0075 &  56.87 $\uparrow$ & 0.0076 &  54.38 $\uparrow$ & 0.0060 &  95.76 $\uparrow$ & 0.44 & 0.09 $\pm$0.00    \\ 
        &              & Tanh & 0.0240 & 0.0166 & 45.03 $\uparrow$ & 0.0166 & 44.51 $\uparrow$ & 0.0147 & 63.37 $\uparrow$ & 0.0083 & 0.0040 & 107.56 $\uparrow$ & 0.0040 & 103.98 $\uparrow$ & 0.0031 & 163.09 $\uparrow$ & 0.44 & 0.09 $\pm$0.00    \\ 
        & 1x50$^+$     & Sig. & 0.0451 & 0.0351 & 28.50 $\uparrow$ & 0.0356 & 26.73 $\uparrow$ & 0.0292 & 54.40 $\uparrow$ & 0.0193 & 0.0117 &  64.36 $\uparrow$ & 0.0122 &  58.29 $\uparrow$ & 0.0089 & 117.29 $\uparrow$ & 0.27 & 0.03 $\pm$0.00    \\ 
        &              & Tanh & 0.0300 & 0.0235 & 27.89 $\uparrow$ & 0.0237 & 26.54 $\uparrow$ & 0.0193 & 55.70 $\uparrow$ & 0.0109 & 0.0078 &  39.12 $\uparrow$ & 0.0081 &  33.97 $\uparrow$ & 0.0057 &  91.05 $\uparrow$ & 0.27 & 0.03 $\pm$0.00    \\ 
        & 1x100$^+$    & Sig. & 0.0421 & 0.0331 & 27.37 $\uparrow$ & 0.0334 & 26.07 $\uparrow$ & 0.0276 & 52.89 $\uparrow$ & 0.0169 & 0.0119 &  41.35 $\uparrow$ & 0.0123 &  37.77 $\uparrow$ & 0.0088 &  92.66 $\uparrow$ & 0.31 & 0.04 $\pm$0.00    \\ 
        &              & Tanh & 0.0260 & 0.0212 & 22.61 $\uparrow$ & 0.0213 & 21.75 $\uparrow$ & 0.0177 & 46.97 $\uparrow$ & 0.0109 & 0.0066 &  64.37 $\uparrow$ & 0.0068 &  60.72 $\uparrow$ & 0.0048 & 125.28 $\uparrow$ & 0.36 & 0.05 $\pm$0.01    \\ 
        & 1x150$^+$    & Sig. & 0.0460 & 0.0375 & 22.63 $\uparrow$ & 0.0377 & 21.81 $\uparrow$ & 0.0317 & 44.99 $\uparrow$ & 0.0188 & 0.0120 &  56.68 $\uparrow$ & 0.0123 &  53.11 $\uparrow$ & 0.0090 & 108.74 $\uparrow$ & 0.35 & 0.06 $\pm$0.00    \\ 
        &              & Tanh & 0.0303 & 0.0238 & 27.49 $\uparrow$ & 0.0239 & 26.80 $\uparrow$ & 0.0198 & 53.29 $\uparrow$ & 0.0107 & 0.0077 &  38.42 $\uparrow$ & 0.0078 &  35.95 $\uparrow$ & 0.0056 &  90.58 $\uparrow$ & 0.41 & 0.06 $\pm$0.00    \\ 
        & 1x200$^+$    & Sig. & 0.0484 & 0.0378 & 28.11 $\uparrow$ & 0.0380 & 27.24 $\uparrow$ & 0.0318 & 52.02 $\uparrow$ & 0.0181 & 0.0129 &  40.47 $\uparrow$ & 0.0132 &  37.38 $\uparrow$ & 0.0096 &  89.04 $\uparrow$ & 0.37 & 0.07 $\pm$0.00    \\ 
        &              & Tanh & 0.0295 & 0.0228 & 29.45 $\uparrow$ & 0.0229 & 28.77 $\uparrow$ & 0.0192 & 53.64 $\uparrow$ & 0.0107 & 0.0074 &  44.50 $\uparrow$ & 0.0075 &  42.20 $\uparrow$ & 0.0056 &  93.18 $\uparrow$ & 0.40 & 0.07 $\pm$0.00    \\ 
        & 1x250$^+$    & Sig. & 0.0474 & 0.0385 & 23.06 $\uparrow$ & 0.0387 & 22.52 $\uparrow$ & 0.0326 & 45.16 $\uparrow$ & 0.0185 & 0.0126 &  46.20 $\uparrow$ & 0.0128 &  44.26 $\uparrow$ & 0.0095 &  93.20 $\uparrow$ & 0.41 & 0.09 $\pm$0.00    \\ 
        &              & Tanh & 0.0263 & 0.0225 & 16.84 $\uparrow$ & 0.0226 & 16.37 $\uparrow$ & 0.0191 & 37.67 $\uparrow$ & 0.0103 & 0.0066 &  55.62 $\uparrow$ & 0.0067 &  53.30 $\uparrow$ & 0.0049 & 111.66 $\uparrow$ & 0.44 & 0.09 $\pm$0.00    \\ 
        \hline
        
    \end{tabular}}
    \label{alg on models}
\end{table*}
	
\end{document}